\documentclass[sigconf]{acmart}

\usepackage{hyperref}       
\usepackage{url}            
\usepackage{booktabs}       
\usepackage{amsfonts}       
\usepackage{nicefrac}       
\usepackage{microtype}      
\usepackage{amsthm}
\usepackage{amsmath}
\usepackage{mathtools}
\usepackage{multirow} 
\usepackage{algorithm}
\usepackage{algorithmic}
\usepackage{subfig}

\usepackage[capitalize,noabbrev]{cleveref}
\theoremstyle{plain}
\newtheorem{theorem}{Theorem}[section]
\newtheorem{proposition}[theorem]{Proposition}

\theoremstyle{definition}
\newtheorem{definition}[theorem]{Definition}

\theoremstyle{remark}
\newtheorem{remark}[theorem]{Remark}
\usepackage[table]{xcolor}
\crefname{equation}{Eq.}{}
\crefname{table}{Table}{}
\crefname{figure}{Figure}{}
\crefname{algorithm}{Algorithm}{}
\crefname{theorem}{Theorem}{}
\crefname{definition}{Definition}{}
\crefname{remark}{Remark}{}
\crefname{assu}{Assumption}{}

\hypersetup{hidelinks}
\usepackage[most]{tcolorbox} 

\newtcolorbox{mybox}[1]{
    colback=gray!10, 
    colframe=gray!50, 
    sharp corners, 
}

\AtBeginDocument{%
  }

\setcopyright{none}
\renewcommand\footnotetextcopyrightpermission[1]{}

\setcopyright{none}
\begin{document}

\title{Breaking the Periodicity Assumption: Robust Tensorial Multi-View Clustering via Graph-Spectral Low-Rank Learning}
\author{Jintian Ji}
\authornote{Both authors contributed equally to this research.}
\email{jintian.ji@griffithuni.edu.au}
\orcid{1234-5678-9012}
\affiliation{%
  \institution{Griffith University}
  \city{Brisbane}
  \state{QLD}
  \country{Australia}
}

\author{Xingsu Li}
\authornotemark[1]
\email{xlijp@connect.ust.hk}
\affiliation{%
\department{Computer Science and Engineering}
  \institution{ The Hong Kong University of Science and Technology}
  \city{Hong Kong}
  \country{China}
}

\author{Songhe Feng}
\authornote{Corresponding author}
\email{shfeng@bjtu.edu.cn}
\orcid{0000-0002-5922-9358}
\affiliation{%
   \department{School of Computer and Information Technology}
  \institution{ Beijing Jiaotong University,}
  \city{Beijing}
  \country{China}
}

\begin{abstract}
Tensorial multi-view clustering (TMC) has achieved strong performance due to its ability to capture high-order correlations across multiple views. Most existing t-SVD-based TMC frameworks apply the Fast Fourier Transform (FFT) along the sample mode to impose frequency-domain low-rank constraints. However, we reveal that this widely adopted design critically relies on an implicit ``periodicity assumption'' induced by the sample arrangement. When samples are ordered by class, neighboring indices tend to be semantically similar, creating artificial local continuity along the sample mode and a favorable spectral structure for FFT-based low-rank regularization. Once this ordering is removed by random permutation, existing t-SVD-based TMC methods suffer severe performance degradation. This strong sensitivity to class ordering conflicts with the permutation-invariant nature of clustering and indicates that part of the reported performance may be attributed to a privileged sample arrangement rather than genuine high-order structure modeling. In this paper, we systematically investigate this phenomenon and its underlying algebraic and spectral mechanisms. To address this fundamental flaw, we further propose a graph-spectral low-rank tensor learning framework based on the Graph Fourier Transform (GFT), which replaces the fixed Fourier basis along the sample mode with a data-driven graph spectral basis, thereby capturing the intrinsic manifold structure without relying on a particular sample ordering. Moreover, we develop an anchor-based variant to address large-scale datasets efficiently. Extensive experiments on various benchmarks validate our findings and demonstrate the competitive or superior performance of the proposed methods compared with state-of-the-art TMC approaches.
\end{abstract}




\maketitle
\pagestyle{empty}

\section{Introduction}

Multi-view clustering (MVC) \cite{liu2025two,shen2022multi} is a fundamental task in unsupervised learning that aims to integrate complementary information from multiple sources or feature spaces to obtain reliable data partitions \cite{chao2021survey}. By jointly exploiting the consensus and complementary information across views, MVC methods often achieve better clustering performance than single-view approaches, such as $k$-means \cite{ahmed2020k}. According to how multi-view information is represented and organized, existing MVC methods can be broadly divided into matrix-based and tensor-based approaches. Matrix-based methods mainly include: (i) co-regularization methods \cite{renmulti,gan2025multi,2024YuTo}, which enforce consistency among view-specific latent representations; (ii) subspace learning methods \cite{zhang2017latent,xu2025asymptotics,chen2022efficient}, which project or factorize multiple views into a shared low-dimensional consensus space; and (iii) graph-based methods \cite{li2020multiview,xin2025Mu,2024DAIMU,zhao2025multi}, which learn a unified affinity matrix or manifold structure from heterogeneous views. Despite their effectiveness, matrix-based methods primarily characterize pairwise or view-wise relationships and may not explicitly capture high-order dependencies shared across multiple views.

\begin{figure}[!h]
\centering
\includegraphics[width=0.45\textwidth]{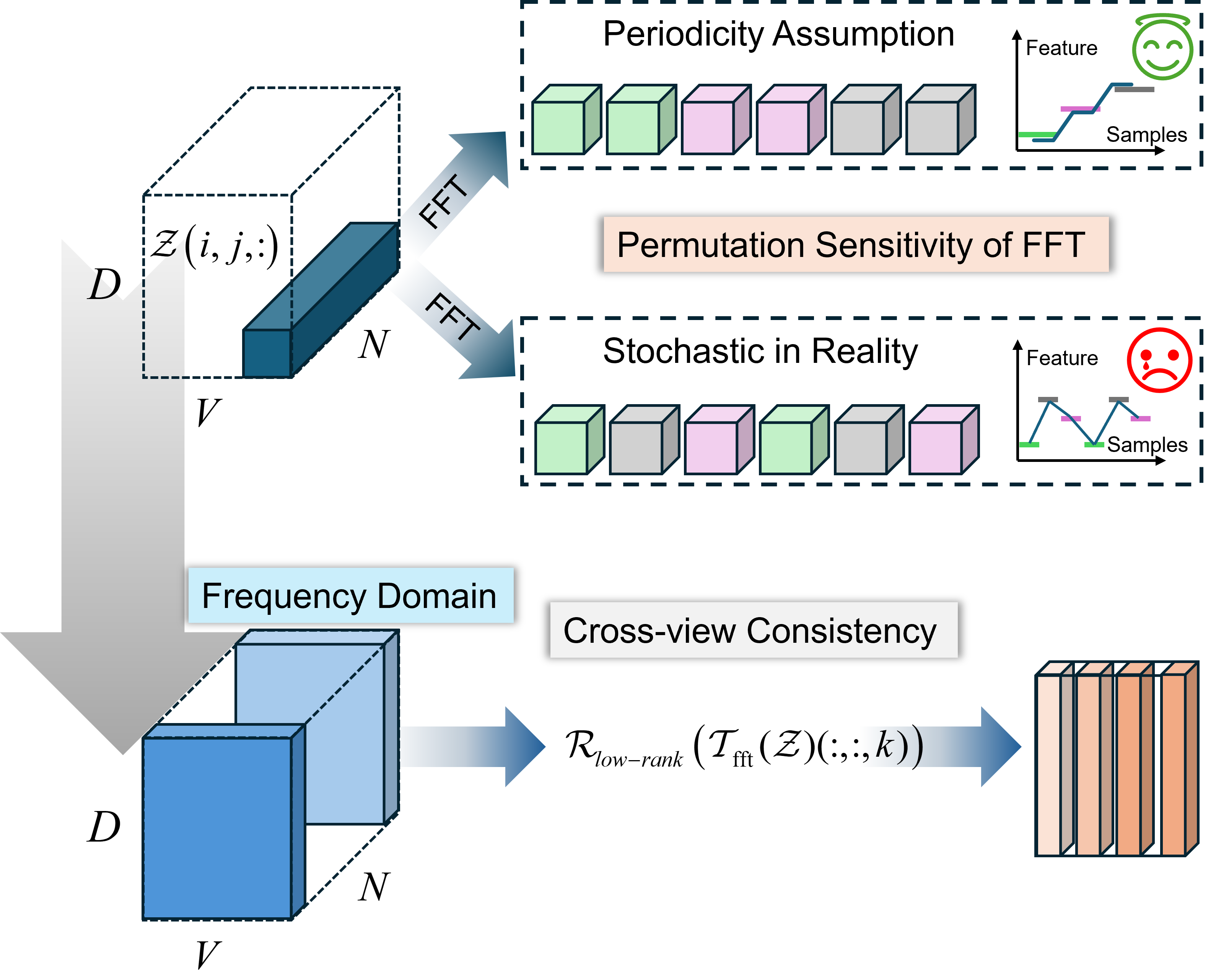}
\setlength{\abovecaptionskip}{0.1cm} 
\caption{The general workflow of t-SVD-based low-rank tensor learning for multi-view clustering.}
\label{fig:frame1}
\end{figure}

Tensor-based multi-view clustering (TMC) methods \cite{xie2018unifying,zhang2015low, GU2024ed,wang2025tpch,gu2024dictionary,li2024label} address this problem by combining the representation matrices of different views into a third-order tensor and then leveraging the low-rank tensor learning framework for the discovery of consistent clustering structure. The success of these tensor-based methods is largely built upon Tensor Singular Value Decomposition (t-SVD) based on the Fast Fourier Transform (FFT) \cite{lu2019tensor,xie2018unifying, luo2026estim}. As shown in \cref{fig:frame1}, these frameworks first apply FFT along the sample mode to convert the representation tensor into the frequency domain, and then impose a low-rank constraint to exploit cross-view consistency. However, applying FFT along the sample mode implicitly treats independent, unordered samples as if they formed a circular sequence, imposing a hidden \textit{Periodicity Assumption}. Class-contiguous ordering approximately induces the local smoothness favored by the Fourier basis. However, such an ordering is merely an artifact of how benchmark datasets are commonly organized rather than an intrinsic property of the clustering problem. The resulting dependence of model performance on this privileged arrangement fundamentally conflicts with permutation invariance \cite{ackerman2010towards}, an essential property of clustering. Consequently, current TMC methods risk reporting an inflated, index-dependent notion of effectiveness: performance can appear strong on class-ordered datasets yet deteriorate sharply once samples are randomly shuffled, even though the ground-truth clustering remains unchanged. This observation motivates the following research questions:
\begin{mybox}

\textbf{RQ1.} \textit{Why are FFT-based tensor regularizers applied along the sample mode sensitive to sample permutations?}

\textbf{RQ2.} \textit{How can transform-based low-rank tensor learning be redesigned to remain robust to arbitrary sample indexing?}
\end{mybox}

Answering these questions requires revisiting how transform-domain tensor regularization interacts with sample permutations. Although several recent studies \cite{liu2025large,ji2025anchorsbring} have empirically observed the sensitivity of tensor-based MVC methods to sample ordering, they mainly provide experimental evidence and do not fully characterize its algebraic origin or establish a principled permutation-robust alternative. A seemingly straightforward workaround is to apply FFT along the feature or view mode instead of the sample mode. However, such a change merely avoids applying the fixed transform to the sample dimension and does not explain or resolve the underlying order dependence that arises when fixed spectral bases are used to model arbitrarily indexed samples.

To address this gap, we systematically investigate the algebraic and spectral foundations of FFT-based low-rank tensor learning. For \textbf{RQ1}, we show that the Fourier transform imposes a circulant structure along the sample mode and is generally non-equivariant to arbitrary sample permutations. From a spectral perspective, random reindexing disrupts the artificial local continuity induced by category ordering and redistributes structural energy across Fourier components. This spectral leakage weakens the energy concentration and transform-domain low-rank structure exploited by existing tensor regularizers, causing discriminative information to be excessively suppressed during low-rank shrinkage.

For \textbf{RQ2}, we propose a graph-spectral low-rank tensor learning framework that replaces the fixed Fourier basis with a data-adaptive graph Fourier transform (GFT) basis. The graph Laplacian is constructed from the intrinsic relationships among samples and transforms consistently when the samples are permuted. We establish spectral-subspace equivariance under a general graph spectrum
and exact GTR stability under a simple spectrum. When repeated eigenvalues
are present, the equivariance error of the resulting low-rank mapping is
bounded by the spectral energy contained in the repeated-eigenvalue
eigenspaces. Consequently, the recovered representation is no longer tied to an arbitrary sample order. To improve scalability, we further develop an anchor-based graph Fourier transform (AGFT), which uses a compact sample-to-anchor graph to approximate the graph-spectral subspace. This construction avoids the full graph eigendecomposition and reduces both computational and memory costs for large-scale datasets.

Our main contributions are summarized as follows:

\textbf{1. Diagnosis of Sample-Order Bias.} We identify an underexamined sample-order dependence in FFT-based tensorial multi-view clustering, tracing its algebraic origin and connecting random sample permutations to spectral energy redistribution and the collapse of transform-domain low-rank structure.

\textbf{2. Permutation-Robust Graph-Spectral Regularization.} We develop a graph-spectral low-rank tensor learning framework that replaces the fixed Fourier basis with the graph Fourier basis, and establish spectral subspace equivariance under general graph spectra and exact GTR stability under the simple spectrum condition.

\textbf{3. Scalable and Plug-and-Play Design.} We introduce an anchor-based graph Fourier transform (AGFT) that approximates the graph-spectral subspace via a compact sample-to-anchor graph, improving scalability over full-graph GFT. The resulting framework can be integrated into existing FFT-based TMC pipelines with minimal modification to their low-rank tensor learning modules.

\textbf{4. Extensive Empirical Validation.} We integrate the proposed transforms into multiple representative TMC methods and evaluate them under controlled random sample permutations, demonstrating substantially improved robustness to sample ordering and consistently competitive or superior clustering performance across datasets and baselines.

\section{Preliminaries and Related Work}

We first summarize the notation used throughout this paper. Scalars are denoted by 
lowercase letters, e.g., $x$; vectors by bold lowercase letters, e.g., $\boldsymbol{x}$; 
and matrices by bold uppercase letters, e.g., $\mathbf{X}$. For a matrix, the Frobenius 
norm, $\ell_{2,1}$ norm, and nuclear norm are denoted by $\left\| \mathbf{X}\right\|_F$, 
$\left\| \mathbf{X}\right\|_{2,1}$, and $\left\| \mathbf{X}\right\|_{\circledast}$, 
respectively. $\boldsymbol{1}$ and $\mathbf{I}$ denote an all-ones vector and the 
identity matrix, respectively. Third-order tensors are represented by calligraphic 
letters, e.g., $\mathcal{X}\in\mathbb{R}^{n_1\times n_2\times n_3}$, and 
$\mathcal{X}^{(k)}$ denotes its $k$-th frontal slice.
Given tensors $\mathcal{X}\in\mathbb{R}^{n_1\times n_2\times n_3}$ and 
$\mathcal{Y}\in\mathbb{R}^{n_2\times n_4\times n_3}$, we next introduce several 
tensor operations and definitions \cite{KERNFELD2015545} used in our framework:

\textbf{Face-wise product:} Given tensors $\mathcal{X}\in\mathbb{R}^{n_1\times n_2\times n_3}$ and $\mathcal{Y}\in\mathbb{R}^{n_2\times n_4\times n_3}$, the face-wise product is defined as:
\begin{equation}
    (\mathcal{X}\Delta \mathcal{Y})^{(i)} = \mathcal{X}^{(i)} \mathcal{Y}^{(i)},\ \ i=1,\ldots,n_3. 
\end{equation}
\textbf{Mode-$3$ product}: Given an $n_3 \times n_3$ matrix $\mathbf{M}$, the mode-$3$ product of tensor $\mathcal{X}\in\mathbb{R}^{n_1\times n_2\times n_3}$ is:
\begin{equation}
    (\mathcal{X}\times_3\mathbf{M})_{(i,j,:)} = \mathbf{M}\mathcal{X}_{(i,j,:)}.
\end{equation}
\textbf{Transform of Tensor:} Given an invertible linear transform $\mathcal{T}$, the Transform of Tensor $\mathcal{X}$ along mode-$3$ is
\begin{equation}
    \bar{\mathcal{X}}=\mathcal{T}(\mathcal{X},\mathbf{T},3)=\mathcal{X} \times_3\mathbf{T}\in\mathbb{R}^{n_1\times n_2\times n_3}
\end{equation}
where the linear transform $\mathcal{T}$ is defined by an invertible matrix $\mathbf{T}\in\mathbb{R}^{n_3\times n_3}$. $\mathcal{X}=\mathcal{T}^{-1}(\bar{\mathcal{X}},\mathbf{T},3)$ denotes the corresponding inverse mapping.

\textbf{Generalized Tensor-Tensor product ($\mathcal{T}$-product):} Given an invertible linear transform $\mathcal{T}$, tensors $\mathcal{X}\in\mathbb{R}^{n_1\times n_2\times n_3}$ and $\mathcal{Y}\in\mathbb{R}^{n_2\times n_4\times n_3}$, the $\mathcal{T}$-product of $\mathcal{X}$ and $\mathcal{Y}$ under $\mathcal{T}$ is 
\begin{equation}
    \mathcal{C}=\mathcal{X} {*_{\mathcal{T}}}\mathcal{Y}\in\mathbb{R}^{n_1\times n_4 \times n_3}
\end{equation}
where $\mathcal{T}(\mathcal{C},\mathbf{T},3)=\mathcal{T}(\mathcal{X},\mathbf{T},3)\Delta\mathcal{T}(\mathcal{Y},\mathbf{T},3)$.

\textbf{Generalized Tensor Transpose:} For any invertible linear transform $\mathcal{T}$, the generalized tensor transpose of $\mathcal{X}$ under $\mathcal{T}$, denoted by $\mathcal{X}^{\top}$, satisfies $\mathcal{T}(\mathcal{X}^\top,\mathbf{T},3)^{(i)}=(\mathcal{T}(\mathcal{X},\mathbf{T},3)^{(i)})^\top, i=1,\ldots, n_3$.

\textbf{Generalized Identity Tensor:} For any invertible linear transform $\mathcal{T}$, the generalized identity tensor
$\mathcal{I}\in\mathbb{R}^{n\times n\times n_3}$
is defined such that
$\mathcal{T}(\mathcal{I},\mathbf{T},3)^{(i)}$
is an identity matrix for each $i$.

\textbf{Generalized Orthogonal Tensor:} For any invertible linear transform $\mathcal{T}$, a tensor $\mathcal{Q}\in\mathbb{R}^{n\times n\times n_3}$ is orthogonal under $\mathcal{T}$ if it satisfies $\mathcal{Q}^\top*_{\mathcal{T}}\mathcal{Q}=\mathcal{Q}*_{\mathcal{T}}\mathcal{Q}^\top=\mathcal{I}$.

\textbf{f-diagonal Tensor:} A tensor is called f-diagonal if each of its frontal slices is a diagonal matrix.

\begin{definition} 
\label{GT-SVD}
\textbf{Generalized Tensor Singular Value Decomposition (GT-SVD).}
Given any invertible linear transform $\mathcal{T}$, any tensor $\mathcal{X}\in\mathbb{R}^{n_1\times n_2\times n_3}$ can be
factorized as 
\begin{equation}
    \mathcal{X}=\mathcal{U}*_{\mathcal{T}}\mathcal{S}*_{\mathcal{T}}\mathcal{V}^\top,
\end{equation}
where $\mathcal{U}\in\mathbb{R}^{n_1\times n_1\times n_3}$, $\mathcal{V}\in\mathbb{R}^{n_2\times n_2\times n_3}$ are orthogonal,
and $\mathcal{S}\in\mathbb{R}^{n_1\times n_2\times n_3}$ is an f-diagonal tensor.
\end{definition}

\begin{definition} 
\label{GT-rank}
\textbf{Generalized Tensor Rank \cite{lu2019tensor}.}
Given any invertible linear transform $\mathcal{T}$, tensor $\mathcal{X}\in\mathbb{R}^{n_1\times n_2\times n_3}$ and its GT-SVD $\mathcal{X}=\mathcal{U}*_{\mathcal{T}}\mathcal{S}*_{\mathcal{T}}\mathcal{V}^\top$, the generalized tensor rank of $\mathcal{X}$ under $\mathcal{T}$ is defined as 
\begin{equation}
\begin{aligned}
        \left\| \mathcal{X} \right\|_{\rm \mathcal{T}-GTR}
        =\frac{1}{n_3}\sum_{k=1}^{n_3}\sum_{i=1}^{h}rank\left(\mathcal{T}(\mathcal{S},\mathbf{T},3)^{(k)}(i,i)\right),\\
\end{aligned}
\label{GNTR}
\end{equation}
\noindent where $h={\rm min}(n_1,n_2)$, $rank(\cdot)$ means the rank function: $rank(x)=1,x>0; rank(x)=0, otherwise$.  Here, various surrogate functions $\psi(\cdot)$ (e.g., $\psi(x)=x$) are commonly used to approximate the true rank function \cite{chen2021generalized,ji2025anchors}. Accordingly, we define $\|\mathcal{X}\|_{\rm \mathcal{T}\text{-}GTR}^{\psi}={1}/{n_3}\sum_{k=1}^{n_3}\sum_{i=1}^{h}\psi\left(\mathcal{T}(\mathcal{S},\mathbf{T},3)^{(k)}(i,i)\right)$.
\end{definition}

\subsection{Tensor-based Multi-view Clustering (TMC)}

The success of tensor-based multi-view clustering (TMC) heavily relies on exploring global consensus through tensor low-rankness. Shifting from early spatial-domain methods \cite{zhang2015low}, modern TMC predominantly operates in the transform domain via the Fast Fourier Transform (FFT) \cite{lu2016tensor, xie2018unifying}. For a multi-view dataset with $n$ samples and $m$ views, $\{\mathbf{X}^{(v)}\}_{v=1}^{m}$, where
$\mathbf{X}^{(v)}\in\mathbb{R}^{d_v\times n}$, the generalized t-SVD-based TMC framework can be formulated as
\begin{equation}
\label{t-SVD}
\begin{aligned} 
&\min_{\{\mathbf{Z}^{(v)}, \mathbf{E}^{(v)}\}_{v=1}^m} \|\mathcal{Z}\|_{\rm fft-GTR}^{\psi}+\alpha\mathcal{R}_1(\{\mathbf{E}^{(v)}\}_{v=1}^m) + \beta\mathcal{R}_2(\{\mathbf{Z}^{(v)}\}_{v=1}^m) \\ &s.t.\ \forall v, \{\mathbf{Z}^{(v)},\mathbf{E}^{(v)}\}=\Omega(\mathbf{X}^{(v)}),\mathcal{Z}=\Phi(\mathbf{Z}^{(1)},\mathbf{Z}^{(2)},\ldots,\mathbf{Z}^{(m)}), 
\end{aligned}
\end{equation}
where the operator $\Omega(\cdot)$ represents generalized feature construction paradigms (e.g., subspace \cite{fu2024subspace} or graph learning \cite{wang2019gmc}) that extract view-specific representations $\mathbf{Z}^{(v)}\in\mathbb{R}^{d\times n}$ and their corresponding reconstruction residuals $\mathbf{E}^{(v)}\in\mathbb{R}^{d_v\times n}$, and $\mathcal{R}_1, \mathcal{R}_2$ denote regularization terms (e.g., sparsity or low-rankness). $\Phi(\cdot)$ \cite{xie2018unifying} stacks and rotates these view-specific representations into a representation tensor $\mathcal{Z} \in \mathbb{R}^{d \times m \times n}$ (features $\times$ views $\times$ samples). The core constraint, $\|\mathcal{Z}\|_{\rm fft-GTR}^{\psi}$, enforces cross-view consistency via spectral energy compaction. Subsequent research has largely evolved within this framework, primarily focusing on anchor-based scalability \cite{tang2025multi, ji2023high, chang2024tensorized} to reduce tensor scale, and non-convex rank surrogates \cite{sun2022sliced, xia2021multiview, ding2025nonconvex} to better approximate true low-rankness. To optimize this framework, methods typically employ the Alternating Direction Method of Multipliers (ADMM), which isolates the tensor rank constraint into a proximal sub-problem:
\begin{equation}
\label{fft_subproblem}
\min_{\mathcal{Z}}\lambda\|\mathcal{Z}\|_{\rm fft\text{-}GTR}^{\psi}+\frac{1}{2}\|\mathcal{Z}-\mathcal{Y}\|_{F}^2
\end{equation}
where $\mathcal{Y}$ is the auxiliary tensor. This problem can be solved by applying the Singular Value Thresholding (SVT) framework to each frequency slice after performing the Fast Fourier Transform (FFT) along the sample mode, as summarized in \cref{alg:fft_low_rank}.

Although \cref{alg:fft_low_rank} is computationally efficient and widely adopted in modern TMC methods, this FFT-based sub-problem exposes a critical vulnerability: pathological sensitivity to sample permutation. The FFT implicitly imposes an artificial periodic (circulant) structure along the sample mode, which fundamentally mismatches the irregular, manifold-structured relationships of real-world multi-view data. Consequently, models often yield inflated performance on "pre-sorted" benchmarks but collapse under random shuffling. Addressing this "order-induced bias" necessitates transitioning from rigid trigonometric bases to data-adaptive transforms that respect the underlying data topology.

\begin{algorithm}[H]
\caption{t-SVD-based Low-Rank Tensor Learning}
\label{alg:fft_low_rank}
\begin{algorithmic}[1] 
\STATE \textbf{Input:} Representation tensor $\mathcal{Y} \in \mathbb{R}^{d \times m \times n}$, parameter $\lambda$, proximal penalty function $\psi(\cdot)$.
\STATE \textbf{Output:} Recovered low-rank tensor $\mathcal{Z}$.

\STATE $\bar{\mathcal{Y}} \leftarrow \mathcal{T}_\text{\rm fft}(\mathcal{Y},\mathbf{F},3)$;
\FOR{$k = 1,\ldots,n$}
    \STATE $[\mathbf{U}_{svd}, {\Sigma}, \mathbf{V}] \leftarrow \text{SVD}(\bar{\mathcal{Y}}^{(k)})$;
  \STATE${\Sigma}_{\text{new}} = \text{Prox}_{\psi, \lambda}({\Sigma})$ \cite{cai2010singular};
    \STATE $\bar{\mathcal{Z}}^{(k)} \leftarrow \mathbf{U}_{svd} {\Sigma}_{\text{new}} \mathbf{V}^H$;
\ENDFOR
\STATE $\mathcal{Z} \leftarrow \mathcal{T}^{-1}_{\text{\rm fft}}(\bar{\mathcal{Z}},\mathbf{F},3)$;
\end{algorithmic}
\end{algorithm}

\begin{figure*}[!h]

	\centering
	\subfloat[Clustering Performance]{\includegraphics[width=0.23\textwidth]{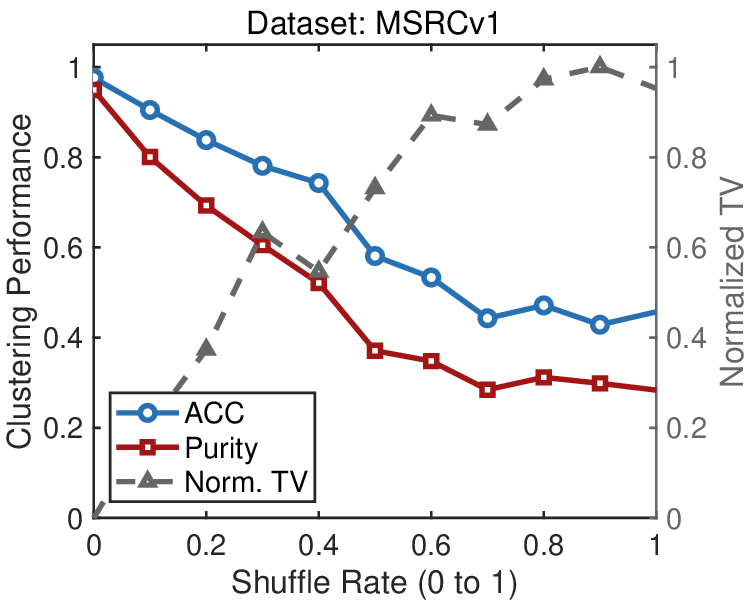}
    \includegraphics[width=0.23\textwidth]{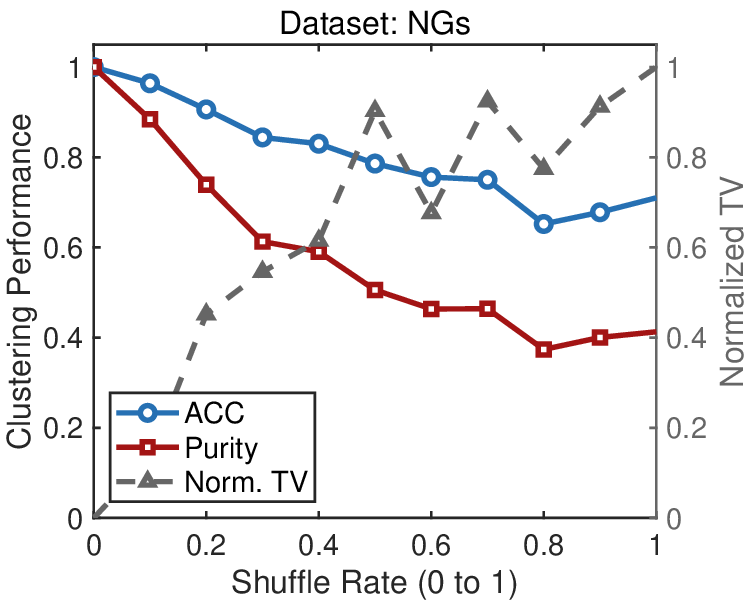}}\hspace{2pt}
	\subfloat[AC Spectral Energy]{\includegraphics[width=0.23\textwidth]{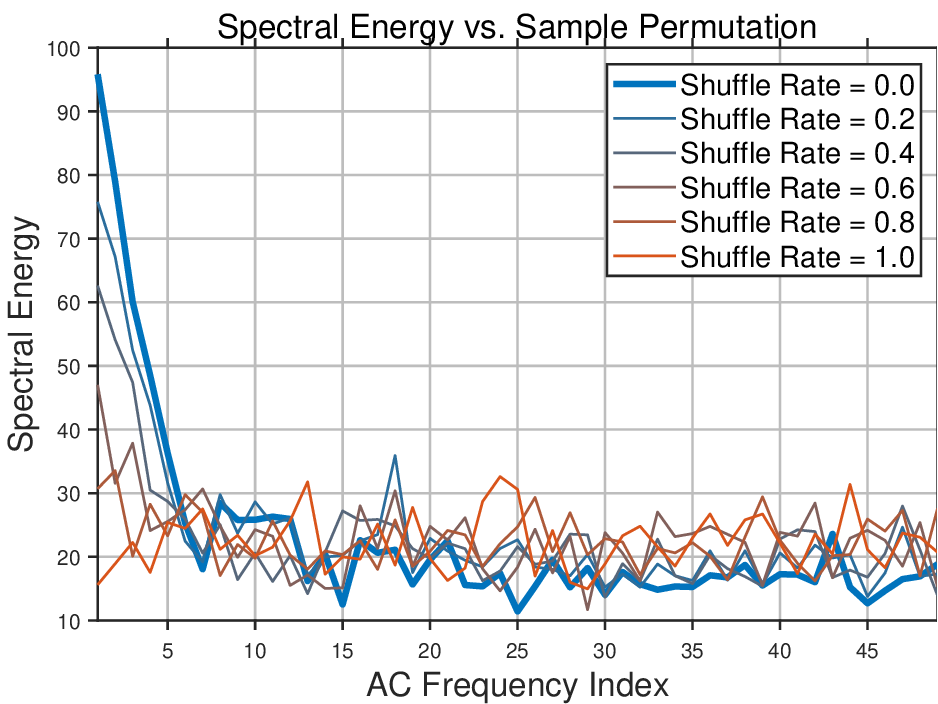}
    \includegraphics[width=0.23\textwidth]{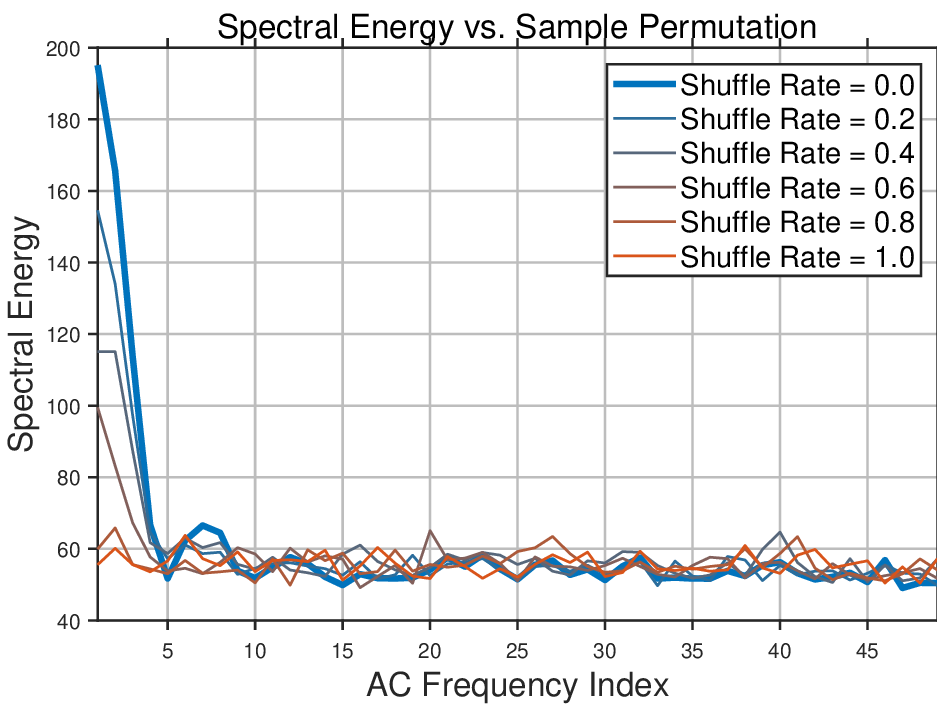}}\hspace{2pt}
    \setlength{\abovecaptionskip}{0.1cm} 
\caption{Impact of sample permutation on TLSpNM \cite{guo2022logarithmic} on MSRCv1 and NGs. (a) Increasing the shuffle rate (the proportion of randomized samples) increases normalized CTV and degrades clustering accuracy. (b) Shuffling weakens low-frequency energy concentration and redistributes energy toward higher frequencies.}
	\label{fig:ratio1}
    \vspace{-0.4cm}
\end{figure*}

\section{Permutation Sensitivity of FFT-Based Low-Rank Tensor Learning}

The effectiveness of \cref{alg:fft_low_rank} depends critically on the energy-compaction property, under which informative multi-view structures are concentrated in a limited number of spectral components. However, to mathematically answer why this framework fails under sample shuffling, we must investigate the relationship between sample permutation and spectral energy distribution.

For a multi-view representation tensor $\mathcal{X} \in \mathbb{R}^{d \times m \times n}$, the FFT along the sample mode is defined as $\bar{\mathcal{X}} = \mathcal{T}_{\rm fft}(\mathcal{X},\mathbf{F},3) = \mathcal{X} \times_3\mathbf{F}$, where $\mathbf{F} \in \mathbb{C}^{n \times n}$ is the Discrete Fourier Transform matrix. We propose the following theorem to characterize the pathological sensitivity of FFT in the tensor domain.

\begin{theorem} 
\label{fft_ps}
\textbf{(Permutation Sensitivity of FFT-based Tensor Transform)}
Let $\mathcal{T}_{{\rm fft}}(\mathcal{X},\mathbf{F},3)$ denote the Fast Fourier Transform of a representation tensor $\mathcal{X} \in \mathbb{R}^{d \times m \times n}$ along its third dimension. For a generic non-identity permutation matrix $\mathbf{P} \in \{0, 1\}^{n \times n}$, let $\mathcal{X}' = \mathcal{X} \times_3 \mathbf{P}$ be the permuted version of $\mathcal{X}$. For generic $\mathcal{X}$, the following properties hold:

\textbf{1. Spectral Non-Equivariance:} 
\begin{equation}
\begin{aligned}
        &\mathcal{T}_{{\rm fft}}(\mathcal{X}',\mathbf{F},3) \neq\mathcal{T}_{{\rm fft}}(\mathcal{X},\mathbf{F},3),\ \ \\ &\mathcal{T}_{{\rm fft}}(\mathcal{X}',\mathbf{F},3) \neq\mathcal{T}_{{\rm fft}}(\mathcal{X},\mathbf{F},3)\times_3\mathbf{P}, \\
\end{aligned}
\end{equation}

\textbf{2. Energy Redistribution Sensitivity:} Define the non-direct-current (non-DC) energy vector as $\boldsymbol{e}(\mathcal{X}) = [\|\bar{\mathcal{X}}^{(2)}\|^2_F, \dots, \|\bar{\mathcal{X}}^{(n)}\|^2_F]^\top$. The energy distribution is non-invariant under a generic permutation:
\begin{equation}
\boldsymbol{e}(\mathcal{X} \times_3 \mathbf{P}) \neq \boldsymbol{e}(\mathcal{X}).
\end{equation}
This implies that there exists at least one frequency $k \in \{2, \dots, n\}$ such that $\|\bar{\mathcal{X}}^{\prime(k)}\|^2_F \neq \|\bar{\mathcal{X}}^{(k)}\|^2_F$.
\end{theorem}

The formal proof is provided in Appendix~\ref{proof1}. \cref{fft_ps} establishes that the FFT-domain low-rank structure exploited by existing TMC methods is sensitive to the particular arrangement of samples. To formalize exactly how permutation degrades performance, we link this spectral redistribution to the Time-Domain Circular Total Variation (CTV). For $\mathcal{X} \in \mathbb{R}^{d \times m \times n}$, we quantify its smoothness along the sample mode as:
\begin{equation}
\label{TV}
CTV(\mathcal{X}) = \sum_{t=1}^{n-1} \|\mathcal{X}^{(t+1)} - \mathcal{X}^{(t)}\|_F +\|\mathcal{X}^{(1)}-\mathcal{X}^{(n)}\|_F,
\end{equation}
where $\mathcal{X}^{(t)}$ denotes the $t$-th sample slice. When samples are sorted by 
category, $\mathcal{X}^{(t+1)} \approx \mathcal{X}^{(t)}$ holds for most $t$ except at 
cluster boundaries, so class-contiguous ordering typically yields a relatively small CTV. 
In contrast, a uniformly random permutation $\mathbf{P}$ is unlikely to preserve this local smoothness and therefore typically increases $CTV(\mathcal{X}\times_3\mathbf{P})$ relative to $CTV(\mathcal{X})$. The following proposition explains why a small CTV provides a favorable condition for FFT-based low-rank tensor regularization.

\begin{proposition}
\label{FEB}
\textbf{ (Frequency Energy Bound)} 
Let $\bar{\mathcal{X}} = \mathcal{T}_{\text{\rm fft}}(\mathcal{X},\mathbf{F},3)$ be the representation of the tensor $\mathcal{X}\in\mathbb{R}^{d \times m\times n}$ in the Fourier domain. The Frobenius norm of the $k$-th alternating-current (AC) frequency component is bounded by the Circular Total Variation (CTV) defined in \cref{TV}:
\begin{equation}\|\bar{\mathcal{X}}^{(k)}\|_F \leq \frac{CTV(\mathcal{X})}{2 \left| \sin(\pi (k-1) / n) \right|}, \quad \forall k=2, \dots, n.
\end{equation}
\end{proposition}
Proposition~\ref{FEB} characterizes how sample-mode variation controls the magnitude of the non-DC Fourier components. Its proof is provided in Appendix~\ref{proof2}. Note that the DC component ($k=1$) is invariant to sample permutations and primarily captures the global mean, carrying limited information for distinguishing cluster structures. The sample-dependent structural variations relevant to clustering are primarily encoded in the non-DC components ($k>1$), whose magnitudes are bounded by Proposition~\ref{FEB}. 

When samples are arranged contiguously by category, $\mathcal{X}$ exhibits an approximately piecewise-smooth structure with a relatively small $CTV(\mathcal{X})$. This tightens the upper envelope of the non-DC components and typically promotes stronger spectral energy concentration (\cref{fig:ratio1}(b), shuffle rate $=0$), providing a favorable spectral structure for low-rank regularization to recover the multi-view consensus. In contrast, random permutations destroy this local continuity and substantially increase $CTV(\mathcal{X})$. The resulting spectral redistribution spreads discriminative energy across a broader range of frequency slices. Consequently, low-rank shrinkage may excessively suppress these dispersed structural components, leading to degradation in clustering accuracy observed in \cref{fig:ratio1}(a).

\section{Graph-Spectral Low-Rank Tensor Learning}

The preceding analysis underscores that to overcome the "periodic curse" of fixed trigonometric bases, a robust low-rank tensor framework necessitates a data-adaptive transformation characterized by permutation equivariance. To this end, we introduce the Graph Fourier Transform (GFT) \cite{shuman2013emerging,sandryhaila2013discrete,pan2025wind,deng2021graph} in the t-SVD-based low-rank tensor learning framework.

\begin{definition}
\textbf{(Graph Fourier Transform of Tensor)} Given a graph Laplacian matrix $\mathbf{L} = \mathbf{U} \mathbf{\Lambda} \mathbf{U}^\top \in \mathbb{R}^{n \times n}$, where $\mathbf{U} = [\boldsymbol{u}_1, \dots, \boldsymbol{u}_n]$ is the orthonormal matrix of eigenvectors, the Graph Fourier Transform (GFT) of a signal $\boldsymbol{x} \in \mathbb{R}^{n}$ is defined as $\bar{\boldsymbol{x}} = \mathbf{U}^\top \boldsymbol{x}$. For a third-order tensor $\mathcal{X} \in \mathbb{R}^{d \times m \times n}$, the GFT along the third dimension is defined as:
\begin{equation}
    \bar{\mathcal{X}} = \mathcal{T}_{\rm gft}(\mathcal{X},\mathbf{U}^\top,3) = \mathcal{X} \times_3 \mathbf{U}^\top
\end{equation}
where each tube $\mathcal{X}(i, j,:)$ is transformed by $\mathbf{U}^\top$.
\label{gft_de}
\end{definition}

This GFT-based formulation fundamentally departs from the classical FFT approach by respecting the underlying topological structure of the data rather than imposing an artificial sequential order. The significance of this geometric awareness is formalized in \cref{theorem_gft}, which guarantees robustness against sample shuffling.

\begin{theorem} \textbf{(Permutation Equivariance of GFT-based Tensor Transform)}
\label{theorem_gft}
Let $\mathcal{X} \in \mathbb{R}^{d \times m \times n}$ be a third-order tensor and $\mathbf{L} = \mathbf{U}\mathbf{\Lambda} \mathbf{U}^\top\in\mathbb{R}^{n\times n}$ be the graph Laplacian characterizing the relationships along the third mode. For any permutation matrix $\mathbf{P} \in \{0, 1\}^{n \times n}$, let $\mathcal{X}' = \mathcal{X} \times_3\mathbf{P}$ and $\mathbf{L}' = \mathbf{P}\mathbf{L}\mathbf{P}^\top = \mathbf{U}'\mathbf{\Lambda} \mathbf{U}'^\top$ be the permuted tensor and Laplacian, respectively. The following properties hold:

\textbf{1. Spectral Subspace Equivariance:} There exists a block-diagonal orthogonal matrix $\mathbf{Q}$, whose blocks correspond to the eigenspaces associated with repeated eigenvalues, such that $\mathbf{U}'=\mathbf{P}\mathbf{U}\mathbf{Q}$ and
\begin{equation}
\mathcal{T}_{\rm gft}(\mathcal{X}', \mathbf{U}'^\top,3)
=
\mathcal{T}_{\rm gft}(\mathcal{X}, \mathbf{U}^\top,3)
\times_3\mathbf{Q}^\top.
\end{equation}

\textbf{2. GTR Rank Stability:} Under the simple spectrum condition, the following equality holds:
\begin{equation}
  \left\|\mathcal{X}'\right\|_{\rm gft\text{-}GTR} =  \left\|\mathcal{X}\right\|_{\rm gft\text{-}GTR}.
\end{equation}
\end{theorem}

The formal proof is provided in Appendix~\ref{proof3}. By leveraging this permutation-equivariance property, we can mitigate the sample-order sensitivity of current TMC frameworks. We reformulate the low-rank tensor learning sub-problem by substituting the FFT constraint with the GFT-enhanced norm:
\begin{equation}
\begin{aligned}
     &\min_{\mathcal{Z}}\lambda\|\mathcal{Z}\|_{\rm gft \text{-}GTR}^{\psi}+ \frac{1}{2}\left\|\mathcal{Z}-\mathcal{Y}\right\|_{F}^2,\\
     \end{aligned}
\end{equation}
where $\mathcal{Y}$ is an auxiliary tensor. Similar to the classical framework, this sub-problem is efficiently solved by applying Singular Value Thresholding (SVT) to each graph-spectral slice $\bar{\mathcal{Y}}^{(k)}$ in the graph-spectral domain, followed by the inverse GFT (\cref{alg:gft_low_rank}). Because the GFT basis adapts to the sample manifold, the resulting low-rank regularizer is better aligned with the intrinsic data geometry and is less likely to suppress structural components solely because of arbitrary sample indexing. Although exact slice-wise GTR invariance is guaranteed only under the simple spectrum condition, the GFT-based low-rank mapping has a bounded permutation-equivariance error when repeated eigenvalues are present (see Appendix~\ref{app:bounded_equivariance_error}).

\begin{algorithm}[H]
\caption{t-SVD-based Low-Rank Tensor Learning with GFT}
\label{alg:gft_low_rank}
\begin{algorithmic}[1] 
\STATE \textbf{Input:} Representation tensor $\mathcal{Y} \in \mathbb{R}^{d \times m \times n}$, parameter $\lambda$, proximal penalty function $\psi(\cdot)$, graph Fourier basis $\mathbf{U}$.
\STATE \textbf{Output:} Low-rank tensor $\mathcal{Z}$.

\STATE $\bar{\mathcal{Y}}\in\mathbb{R}^{d\times m\times n} \leftarrow \mathcal{T}_\text{\rm gft}(\mathcal{Y}, \mathbf{U}^\top,3)$;
\FOR{$k = 1,\ldots,n$}
    \STATE $[\mathbf{U}_{svd}, {\Sigma}, \mathbf{V}] \leftarrow \text{SVD}(\bar{\mathcal{Y}}^{(k)})$;
  \STATE${\Sigma}_{\text{new}} = \text{Prox}_{\psi, \lambda}({\Sigma})$;
    \STATE $\bar{\mathcal{Z}}^{(k)} \leftarrow \mathbf{U}_{svd} {\Sigma}_{\text{new}} \mathbf{V}^\top$;
\ENDFOR
\STATE $\mathcal{Z} \leftarrow \mathcal{T}^{-1}_{\text{\rm gft}}(\bar{\mathcal{Z}}, \mathbf{U}^\top,3)$;
\end{algorithmic}
\end{algorithm}

\subsection{Simple Graph Construction Strategy}
Although the GFT removes the dependence on a fixed periodic basis, its effectiveness depends critically on the quality of the underlying graph Laplacian $\mathbf{L}$. To capture the shared geometric structure across views with moderate computational overhead, we adopt a simple graph construction strategy, termed \textbf{SimG}.

Given a normalized multi-view dataset $\{\mathbf{X}^{(v)}\}_{v=1}^m$ with $m$ views, where $\mathbf{X}^{(v)} \in \mathbb{R}^{d_v \times n}$ contains $n$ samples with $d_v$ features, we first form a global representation by vertically concatenating all views:
\begin{equation}
\mathbf{X}_{unified}=[\mathbf{X}^{(1)};\ldots;\mathbf{X}^{(m)}]=[\hat{\boldsymbol{x}}_1,\ldots,\hat{\boldsymbol{x}}_n]\in \mathbb{R}^{d_{all}\times n},
\label{x_u}
\end{equation}
where $d_{all} = \sum_{v=1}^m d_v$ is the total feature dimensionality. To characterize the local connectivity of the sample manifold, we construct a ${k_n}$-nearest-neighbor (${k_n}$-NN) graph. The entries of the unified adjacency matrix $\mathbf{W}_{unified} \in \mathbb{R}^{n \times n}$ are defined as:
\begin{equation}
\mathbf{W}_{unified}(i,j) =
\begin{cases} 1 & \text{if } \hat{\boldsymbol{x}}_{j} \in \mathcal{N}_{k_n}(\hat{\boldsymbol{x}}_i) \text{ or } \hat{\boldsymbol{x}}_i \in \mathcal{N}_{k_n}(\hat{\boldsymbol{x}}_{j}), \\ 0 & \text{otherwise,} \end{cases}
\end{equation}
where $\mathcal{N}_{k_n}(\hat{\boldsymbol{x}}_i)$ denotes the set of ${k_n}$ nearest neighbors of sample $\hat{\boldsymbol{x}}_i$ based on the Euclidean distance in the joint feature space. By employing the normalized Laplacian $\mathbf{L} = \mathbf{I} - \mathbf{D}^{-1/2} \mathbf{W}_{unified} \mathbf{D}^{-1/2}$, we obtain the graph Fourier basis $\mathbf{U}$ through eigendecomposition.

\subsection{Anchor-based Graph Fourier Transform}
Although the GFT-based framework provides theoretical guarantees, the time and space complexities of GFT are both $\mathcal{O}(n^2)$. Such computational overhead limits the scalability of the algorithm for large-scale datasets. To alleviate this burden, we propose a computationally efficient framework termed \textbf{Anchor-based Graph Fourier Transform (AGFT)}.

\begin{definition} \textbf{(Anchor-based Graph Fourier Transform, AGFT)}
\label{def:A-GFT}
Given a concatenated multi-view feature matrix $\mathbf{X}_{unified} \in \mathbb{R}^{d_{all} \times n}$ in \cref{x_u}, let $\mathbf{A} = [\boldsymbol{a}_1, \dots, \boldsymbol{a}_{k_a}] \in \mathbb{R}^{d_{all} \times k_a}$ be ${k_a}$ anchors (${k_a} \ll n$) obtained via anchor-selection methods. Let $\mathbf{S} \in \mathbb{R}^{n \times {k_a}}$ be the anchor graph constructed following the method in \cite{long2025tlrlf4mvc}, where $\mathbf{S}_{ij}$ measures the similarity between sample $\hat{\boldsymbol{x}}_i$ and anchor $\boldsymbol{a}_j$. The AGFT of a tensor $\mathcal{X} \in \mathbb{R}^{d \times m \times n}$ along the third dimension is defined as:
\begin{equation}
\bar{\mathcal{X}} = \mathcal{T}_{\rm agft}(\mathcal{X},\mathbf{U}_{k_a}^\top, 3) = \mathcal{X} \times_3 \mathbf{U}_{k_a}^\top \in \mathbb{R}^{d \times m \times {k_a}},
\end{equation}
where $\mathbf{U}_{k_a} \in \mathbb{R}^{n \times {k_a}}$ is the anchor-based graph Fourier basis. Specifically, let $\mathbf{\Delta} = \text{diag}(\mathbf{S}^\top \mathbf{1}) \in \mathbb{R}^{{k_a} \times {k_a}}$ and $\mathbf{D} = \text{diag}(\mathbf{S}\mathbf{1}) \in \mathbb{R}^{n \times n}$ be the degree matrices. We first construct the normalized anchor affinity matrix:
$\hat{\mathbf{S}} = \mathbf{D}^{-1/2} \mathbf{S} \mathbf{\Delta}^{-1/2}.$
Then, we compute the eigendecomposition of the reduced matrix:
$\hat{\mathbf{S}}^\top \hat{\mathbf{S}} = \mathbf{V} {\Sigma}^2 \mathbf{V}^\top,$
where $\mathbf{V}$ and ${\Sigma}$ contain the corresponding right singular vectors and positive singular values, respectively. The anchor-based graph Fourier basis is then constructed as:
$\mathbf{U}_{k_a} = \hat{\mathbf{S}} \mathbf{V} {\Sigma}^{-1} = \mathbf{D}^{-1/2} \mathbf{S} \mathbf{\Delta}^{-1/2} \mathbf{V} {\Sigma}^{-1}.$
\end{definition}

\begin{remark} \textbf{(Permutation Invariance of AGFT)}
The truncated spectral basis $\mathbf{U}_{k_a}$ exhibits \textit{permutation equivariance}. Specifically, if the samples are permuted by $\mathbf{P}$ while the anchor set is kept fixed or selected by a permutation-invariant procedure, the
sample-to-anchor matrix satisfies $\mathbf{S}'=\mathbf{P}\mathbf{S}$. Assuming distinct nonzero singular values and a deterministic sign convention, the left singular basis obtained from the SVD of the normalized affinity $\hat{\mathbf{S}}$ satisfies $\mathbf{U}_{k_a}'=\mathbf{P}\mathbf{U}_{k_a}$. Consequently, the AGFT spectral coefficients are invariant to sample ordering:
$\mathcal{T}_{\rm agft}(\mathcal{X} \times_3\mathbf{P}, \mathbf{U}_{k_a}'^\top,3) = \mathcal{T}_{\rm agft}(\mathcal{X}, \mathbf{U}_{k_a}^\top,3).$
This ensures that the resulting truncated graph-spectral coefficients and the corresponding low-rank regularization remain consistent regardless of the sample sequence.
\end{remark}

\begin{table*}[!h]
\setlength{\abovecaptionskip}{0.1cm} 
\centering
\caption{Clustering performance (mean(std)) of five baseline architectures on the MSRCv1 and NGs datasets.
}
\label{tab:main_results}

  \resizebox{\textwidth}{!}
  {
\begin{tabular}{l l cccccccccc}
\toprule
\multirow{2}{*}{\textbf{Baselines}} & \multirow{2}{*}{\textbf{Variants}} & \multicolumn{5}{c}{\textbf{Shuffled MSRCv1}} & \multicolumn{5}{c}{\textbf{Shuffled NGs}} \\
\cmidrule(lr){3-7} \cmidrule(lr){8-12}
& & ACC& NMI & Purity& F-score &ARI & ACC& NMI & Purity& F-score & ARI \\
\midrule
  & +FFT   & 81.43(2.38) & 73.35(1.53) & 81.81(2.38) & 70.82(1.94) & 65.99(2.30) & 88.76(0.86) & 74.82(1.24) & 88.76(0.86) & 80.22(1.29) & 75.30(1.61) \\
          & +FFT(view mode) & 86.10(0.21) & 77.90(0.50) & 86.10(0.21) & 75.90(0.28) & 71.95(0.33) & \textbf{99.00(0)} & \textbf{96.52(0)} & \textbf{99.00(0)} & \textbf{98.00(0)} & \textbf{97.51(0)} \\
          \rowcolor{gray!15} \cellcolor{white} & +GFT(SimG) & 88.57(0) & 79.33(0) & 88.57(0) & 78.69(0) & 75.52(0) & \underline{98.60(0)} & \underline{95.13(0)} & \underline{98.60(0)} & \underline{97.21(0)} & \underline{96.52(0)} \\
          \rowcolor{gray!15} \cellcolor{white} & +GFT(GMC) & \textbf{91.90(0)} & \textbf{84.34(0.37)} & \textbf{91.90(0)} & \textbf{84.44(0.04)} & \textbf{81.92(0.04)} & 98.28(0.11) & 94.19(0.38) & 98.28(0.11) & 96.59(0.22) & 95.74(0.27) \\
          \rowcolor{gray!15} \cellcolor{white} & +AGFT & \underline{88.86(0.43)} & \underline{80.45(0.44)} & \underline{88.86(0.43)} & \underline{79.20(0.76)} & \underline{75.80(0.90)} & 92.60(0) & 79.37(0) & 92.60(0) & 85.75(0) & 82.30(0) \\
          \rowcolor{gray!30} \multirow{-6}{*}{\cellcolor{white} t-SVD-MVC} & +GFT(Ideal) & 100(0) & 100(0) & 100(0) & 100(0) & 100(0) & 100(0) & 100(0) & 100(0) & 100(0) & 100(0) \\
\midrule
\addlinespace[0.5em] 

  & +FFT   & 47.05(3.17) & 30.04(2.06) & 47.71(2.84) & 30.82(1.77) & 19.49(2.08) & 68.68(1.67) & 39.99(1.81) & 68.68(1.67) & 50.81(1.80) & 38.50(2.28) \\
          & +FFT(view mode) & 54.95(2.62) & 37.79(1.35) & 55.14(2.46) & 37.10(1.58) & 26.75(1.78) & 82.24(0.09) & 61.22(0.31) & 82.24(0.09) & 68.50(0.16) & 60.63(0.20) \\
          \rowcolor{gray!15} \cellcolor{white} & +GFT(SimG) & 69.71(0.99) & 56.56(1.19) & 69.71(0.99) & \underline{62.11(0.84)} & \underline{47.21(1.17)} & \underline{85.00(0)} & \underline{67.09(0)} & \underline{85.00(0)} & \underline{73.10(0)} & \underline{66.39(0)} \\
          \rowcolor{gray!15} \cellcolor{white} & +GFT(GMC) & \underline{74.48(4.59)} & \underline{62.76(6.17)} & \underline{74.86(4.50)} & 61.02(5.76) & \textbf{54.57(6.73)} & \textbf{97.56(0.65)} & \textbf{92.23(1.63)} & \textbf{97.56(0.65)} & \textbf{95.20(1.24)} & \textbf{94.01(1.55)} \\
          \rowcolor{gray!15} \cellcolor{white} & +AGFT & \textbf{80.10(1.23)} & \textbf{66.56(1.03)} & \textbf{80.10(1.23)} & \textbf{66.36(1.15)} & 41.45(2.20) & 75.00(0) & 58.53(0) & 75.00(0) & 61.84(0) & 52.14(0) \\
          \rowcolor{gray!30} \multirow{-6}{*}{\cellcolor{white} TLSpNM} & +GFT(Ideal) & 100(0) & 100(0) & 100(0) & 100(0) & 100(0) & 99.84(0.09) & 99.44(0.31) & 99.84(0.09) & 99.68(0.18) & 99.60(0.22) \\
\midrule
\addlinespace[0.5em] 

  & +FFT   & 80.33(0.55) & 69.82(1.38) & 80.33(0.55) & 67.70(2.11) & 62.30(1.62) & 95.32(0.19) & 85.89(0.54) & 95.32(0.19) & 90.85(0.65) & 88.58(0.46) \\
          & +FFT(view mode) & 80.29(0.43) & 70.04(2.12) & 80.29(0.43) & 67.88(2.02) & 62.59(2.37) & 95.80(0) & 87.84(0) & 95.80(0) & 91.78(0) & 89.74(0) \\
          \rowcolor{gray!15} \cellcolor{white} & +GFT(SimG) & 87.81(0.26) & 80.03(0.48) & 87.81(0.26) & 77.58(0.33) & 73.85(0.38) & \underline{97.16(0.09)} & \underline{90.79(0.30)} & \underline{97.16(0.09)} & \underline{94.39(0.17)} & \underline{93.00(0)} \\
          \rowcolor{gray!15} \cellcolor{white} & +GFT(GMC) & \textbf{91.62(0.26)} & \textbf{83.99(0.50)} & \textbf{91.62(0.26)} & \textbf{83.90(0.41)} & \textbf{81.28(0.48)} & \textbf{98.24(0.09)} & \textbf{94.02(0.23)} & \textbf{98.24(0.09)} & \textbf{96.51(0.17)} & \textbf{95.64(0)} \\
          \rowcolor{gray!15} \cellcolor{white} & +AGFT & \underline{89.05(0)} & \underline{80.05(0)} & \underline{89.05(0)} & \underline{79.07(0)} & \underline{75.62(0)} & 90.84(0) & 75.56(0) & 90.84(0) & 82.56(0) & 78.21(0) \\
          \rowcolor{gray!30} \multirow{-6}{*}{\cellcolor{white} ASR-ETR} & +GFT(Ideal) & 100(0) & 100(0) & 100(0) & 100(0) & 100(0) & 100(0) & 100(0) & 100(0) & 100(0) & 100(0) \\
\midrule
\addlinespace[0.5em] 
    & +FFT   & 77.81(1.74) & 67.56(2.84) & 77.81(1.74) & 65.97(1.95) & 60.41(2.29) & 85.84(0.52) & 67.72(0.77) & 85.84(0.52) & 73.99(1.08) & 67.48(1.14) \\
          & +FFT(view mode) & 78.29(0.80) & 70.23(0.52) & 78.29(0.80) & 67.28(0.50) & 61.97(0.58) & 86.44(0.09) & 68.59(0.12) & 86.44(0.09) & 75.07(0.30) & 68.64(0.17) \\
          \rowcolor{gray!15} \cellcolor{white} & +GFT(SimG) & 85.05(1.53) & 74.99(1.78) & 85.05(1.53) & 74.32(1.64) & 70.06(1.92) & \underline{96.60(0)} & \underline{89.53(0)} & \underline{96.60(0)} & \underline{93.33(0)} & \underline{91.68(0)} \\
          \rowcolor{gray!15} \cellcolor{white} & +GFT(GMC) & \textbf{88.67(2.72)} & \textbf{79.58(4.08)} & \textbf{88.67(2.72)} & \textbf{78.75(4.75)} & \textbf{75.26(5.55)} & \textbf{98.20(0)} & \textbf{93.92(0)} & \textbf{98.20(0)} & \textbf{96.43(0)} & \textbf{95.54(0)} \\
          \rowcolor{gray!15} \cellcolor{white} & +AGFT & \underline{87.81(0.26)} & \underline{78.87(0.55)} & \underline{87.81(0.26)} & \underline{77.38(0.44)} & \underline{73.61(0.51)} & 87.80(0) & 71.35(0) & 87.80(0) & 77.53(0) & 71.92(0) \\
          \rowcolor{gray!30} \multirow{-6}{*}{\cellcolor{white} ESTMC} & +GFT(Ideal) & 100(0) & 100(0) & 100(0) & 100(0) & 100(0) & 100(0) & 100(0) & 100(0) & 100(0) & 100(0) \\
\midrule
\addlinespace[0.5em] 
  & +FFT   & \underline{85.24(0)} & 75.33(0) & \underline{85.24(0)} & 74.60(0) & 70.34(0) & 88.26(0.19) & 70.02(0.72) & 88.26(0.19) & 78.32(0.34) & 72.93(0.42) \\
          & +FFT(view mode) & 85.05(0.43) & 76.21(0.63) & 85.05(0.43) & \underline{75.83(1.21)} & \underline{71.81(1.36)} & 91.40(0) & 77.13(0) & 91.40(0) & 83.91(0) & 79.93(0) \\
          \rowcolor{gray!15} \cellcolor{white} & +GFT(SimG) & 83.81(2.13) & 75.05(0.63) & 84.38(0.85) & 73.90(0.22) & 69.57(0.25) & \underline{97.20(0)} & \underline{91.26(0)} & \underline{97.20(0)} & \underline{94.46(0)} & \underline{93.08(0)} \\
          \rowcolor{gray!15} \cellcolor{white} & +GFT(GMC) & \textbf{91.90(0)} & \textbf{84.52(0)} & \textbf{91.90(0)} & \textbf{84.48(0)} & \textbf{81.97(0)} & \textbf{98.20(0)} & \textbf{93.92(0)} & \textbf{98.20(0)} & \textbf{96.43(0)} & \textbf{95.54(0)} \\
          \rowcolor{gray!15} \cellcolor{white} & +AGFT & 84.76(0) & \underline{76.45(0)} & 84.76(0) & 74.23(0) & 69.93(0) & 90.80(0) & 76.95(0) & 90.80(0) & 82.86(0) & 78.61(0) \\
          \rowcolor{gray!30} \multirow{-6}{*}{\cellcolor{white} TLRLF4MVC} & +GFT(Ideal) & 95.24(0) & 90.53(0) & 95.24(0) & 91.07(0) & 89.62(0) & 100(0) & 100(0) & 100(0) & 100(0) & 100(0) \\
\bottomrule
\end{tabular}}

\end{table*}
\begin{remark} \textbf{(Intuitive Interpretation of AGFT)}
Since $\mathbf{U}_{k_a}\in\mathbb{R}^{n\times k_a}$ is rectangular, AGFT is not an invertible transform over the original sample space, but a truncated projection onto a $k_a$-dimensional graph-spectral subspace. Intuitively, $\mathbf{U}_{k_a}$ retains the dominant smooth components characterized by the anchor graph, while the omitted orthogonal complement $\mathbf{U}_{\perp}$ is assumed to contain relatively little structural energy. Because the singular-value proximal operator maps zero matrices to zero, the discarded components remain inactive during low-rank shrinkage. AGFT can therefore be viewed as performing slice-wise low-rank learning within a compact graph-spectral subspace. Moreover, it reduces the per-tube cost of full GFT from $\mathcal{O}(n^2)$ to $\mathcal{O}(nk_a)$, where $k_a\ll n$, thereby improving scalability on large datasets.
\end{remark}


\section{Experiments}

\subsection{Experimental Setup}

 \textbf{Datasets:} We conduct experiments on eight benchmark datasets, including MSRCv1, NGs, 
BBCSport, Caltech101-20, CCV, Caltech101-all, Aloi-100, and Animal. For each trial, we generate a random permutation of the sample indices and apply the same permutation to all views and the ground-truth labels. This procedure removes class-contiguous ordering while preserving cross-view sample correspondence and the underlying clustering problem. Further details are provided in Appendix~\ref{exp_data}.

 \textbf{Baselines:} To comprehensively evaluate the performance of our GFT-based framework, we benchmark it against five state-of-the-art TMC methods. These baselines represent the current frontier of FFT-based low-rank tensor learning, including \textbf{t-SVD-MVC (IJCV 2018)} \cite{xie2018unifying}, \textbf{TLSpNM (TPAMI 2023)} \cite{guo2022logarithmic}, \textbf{ASR-ETR (ICCV 2023)} \cite{ji2023anchor}, \textbf{ESTMC (TPAMI 2025)} \cite{ji2025anchors}, and \textbf{TLRLF4MVC (TPAMI 2025)} \cite{long2025tlrlf4mvc}. 
Further details are provided in Appendix~\ref{exp_data}.

 \textbf{Evaluation Metrics:} To quantitatively assess the clustering performance, we adopt five widely recognized evaluation metrics: Accuracy (ACC), Normalized Mutual Information (NMI), Purity, F-score, and Adjusted Rand Index (ARI). In addition to performance metrics, we also record the CPU runtime (including graph construction) to evaluate the computational efficiency.

\textbf{Implementation Details:} To ensure a rigorous comparison, we integrate the GFT-based transform into representative TMC baselines by replacing their sample-mode FFT components. We evaluate our framework using three graph construction strategies: a $k$-nearest neighbor graph (SimG, $k_n=10$), an adaptive graph (GMC \cite{wang2019gmc}), and an ideal graph (oracle upper-bound reference). These variants are referred to as \textbf{GFT(SimG)}, \textbf{GFT(GMC)}, and \textbf{GFT(Ideal)}, respectively. We also employ the \textbf{AGFT} (\cref{def:A-GFT}) with anchors selected via BKHK \cite{Nie2024Fast}. The number of anchors is set to $k_a=2^{\min(c,10)}$, where $c$ denotes the number of clusters. Furthermore, we include an \textbf{FFT(view mode)} variant \cite{liu2025large} that performs FFT along the view mode to specifically benchmark against existing efforts to mitigate sample-ordering sensitivity.  All results are averaged over five independent trials with hyperparameters optimized via grid search. More details are provided in Appendix~\ref{exp_data}. The code is available at \href{https://github.com/jijintian/GFT-TMVC}{https://github.com/jijintian/GFT-TMVC}.

\begin{table*}[h]
\setlength{\abovecaptionskip}{0.1cm} 
\centering
\caption{Runtime (main execution + transform matrix construction) and performance of TLRLF4MVC with different transforms.}
\label{time_table}

  \resizebox{0.9\textwidth}{!}
  {
\begin{tabular}{ccl|cl|cl|cl}
    \toprule
    \multirow{2}[4]{*}{Methods} & \multicolumn{2}{c}{Shuffled CCV} & \multicolumn{2}{c}{Shuffled Aloi-100} & \multicolumn{2}{c}{Shuffled Caltech101-all} & \multicolumn{2}{c}{Shuffled Animal} \\
\cmidrule(lr){2-3} \cmidrule(lr){4-5}  \cmidrule(lr){6-7}   \cmidrule(lr){8-9}    & ACC   & Time (s) & ACC   & Time (s) & ACC   & Time (s) & ACC   & Time (s) \\
    \midrule
+FFT   & 26.53 & 21.55 + 0 & 72.08 & 65.52 + 0 & 26.17 & 110.59 + 0 & 19.42 & 88.09 + 0 \\
    +SimG  & 27.28 & 95.67 + 10.77 & 74    & 116.41 + 25.83 & 29.1  & 168.56 + 17.92 & 20.53 & 150.44 + 37.23  \\
    +GMC   & 28.11 & 63.75 + 157.38 & 75.06 & 142.09 + 284.13 & 28.72 & 167.91 + 207.36 & 21.57 & 131.58 + 279.72 \\
    +Ideal & \textbf{65.38} & 36.7 + 10.66  & \textbf{82.44} & 126.91 + 26.50 & \textbf{65.75} & 195.85 + 16.89 & \textbf{68.96} & 122.93 + 33.10 \\
    \rowcolor[rgb]{ .851,  .851,  .851} +AGFT  & 26.87 & \textbf{11.68 + 0.62} & 75.06 & \textbf{33.17 + 1.60} & 30.54 & \textbf{37.98 + 12.81} & 20.82 & \textbf{22.48 + 35.66 } \\
    \bottomrule
    \end{tabular}}
\end{table*}

\subsection{Results}
The results of different baselines (with different variants) are shown in Table \ref{tab:main_results}, where the {\bf best} and \underline{second-best} are highlighted in bold and underlined, respectively, excluding the GFT(Ideal). Based on these results, we make the following observations.

\textbf{General Superiority of GFT over FFT}.
The experimental results reveal that, in the vast majority of scenarios, our proposed GFT-based frameworks consistently and substantially outperform the traditional FFT-based framework. Notably, on Shuffled MSRCv1, GFT improves the ACC of the TLSpNM baseline from $47.05\%$ to $74.48\%$. The fundamental reason for this consistent superiority lies in the nature of the transform basis. While the standard FFT employs a fixed trigonometric basis that assumes a rigid, periodic structure, our GFT framework leverages a data-adaptive graph basis. By constructing a graph that encapsulates the intrinsic manifold structure, the GFT "aligns" its spectral atoms with the underlying sample clusters. Consequently, while the FFT suffers from spectral dispersion and the collapse of the low-rank prior when sample relationships are disordered, the GFT maintains a highly compact energy distribution, leading to more robust and accurate clustering.

\textbf{Limitations of FFT(view mode).} Performing FFT along the view dimension avoids sample-order sensitivity, but its fixed sinusoidal basis is not adapted to the heterogeneous dependencies among views. As shown in Table~\ref{tab:main_results}, its effectiveness varies considerably across baseline architectures. For example, on TLSpNM, FFT(view mode) underperforms +GFT(GMC) by 19.53 percentage points in ACC, whereas on ASR-ETR it yields only a 0.22-point NMI improvement over the original FFT variant. These results indicate that FFT(view mode) can be effective in some cases, but does not consistently capture cross-view dependencies as well as a data-adaptive graph-spectral transform.

\textbf{Empirical Performance Lower and Upper Bounds.} The near-ideal performance of GFT(Ideal) defines a high ceiling, proving that the efficacy of the GFT-based framework scales with graph quality. Crucially, our SimG and AGFT serve as the performance floor of this paradigm. Despite using minimalist, non-optimized graphs, these configurations outperform FFT-based methods in most scenarios. This underscores the inherent robustness of the GFT-based framework over order-dependent FFT-based pipelines, suggesting that GFT-based TMC offers a superior, scalable path toward optimal low-rank tensor learning as graph learning evolves.

\textbf{Efficiency and Scalability of AGFT.} As shown in Table \ref{time_table}, AGFT achieves substantially lower runtime than full-graph GFT variants and is competitive with FFT on large-scale datasets. By reducing the transform complexity per tube from $\mathcal{O}(n^2)$ ($\mathcal{O}(n\log n)$ in FFT) to $\mathcal{O}(nk_a)$ ($k_a \ll n$), AGFT effectively eliminates the scalability bottleneck of traditional GFT. Remarkably, AGFT even surpasses FFT in efficiency; for instance, on Caltech101-all, it reduces main execution time from 110.59s to 37.98s—a nearly 3-fold speedup—while achieving the highest non-ideal ACC of 30.54\%. These results demonstrate that AGFT not only bridges the efficiency gap between graph and Fourier-based methods but also provides a highly scalable solution for large-scale multi-view clustering.

\subsection{Ablation Study}

\textbf{Influence of Shuffle Rates on GFT-based Methods.} To evaluate the robustness of our framework against data disorder, we conduct experiments with shuffle rates ranging from 0 to 1. As illustrated in \cref{fig:shuffle_rate}, with additional results
provided in Appendix~\ref{exp_ablation}, the performance of the FFT-based framework deteriorates sharply as the sample order becomes randomized; our GFT(SimG), GFT(GMC), and AGFT variants exhibit remarkable stability. Their clustering metrics remain nearly invariant as the shuffle rate increases, maintaining a near-constant plateau. This empirically validates that the graph-adaptive basis effectively captures the intrinsic topology, making the spectral representation substantially more stable under sample permutations.
\begin{figure}[!h]
    \vspace{-0.5cm}
	\centering
	\subfloat[MSRCv1]{\includegraphics[width=0.23\textwidth]{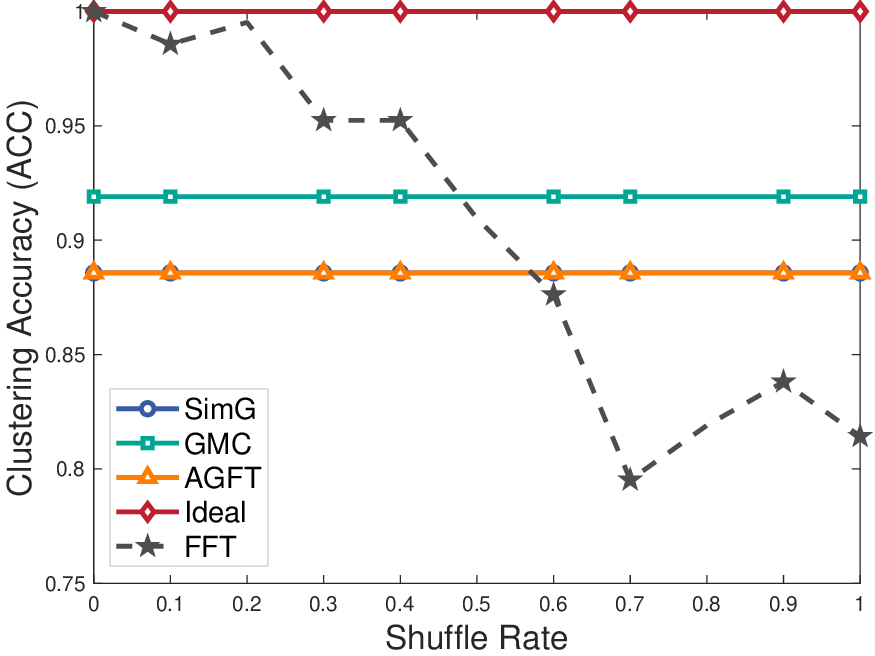}}\hspace{2pt}
	\subfloat[NGs]{\includegraphics[width=0.23\textwidth]{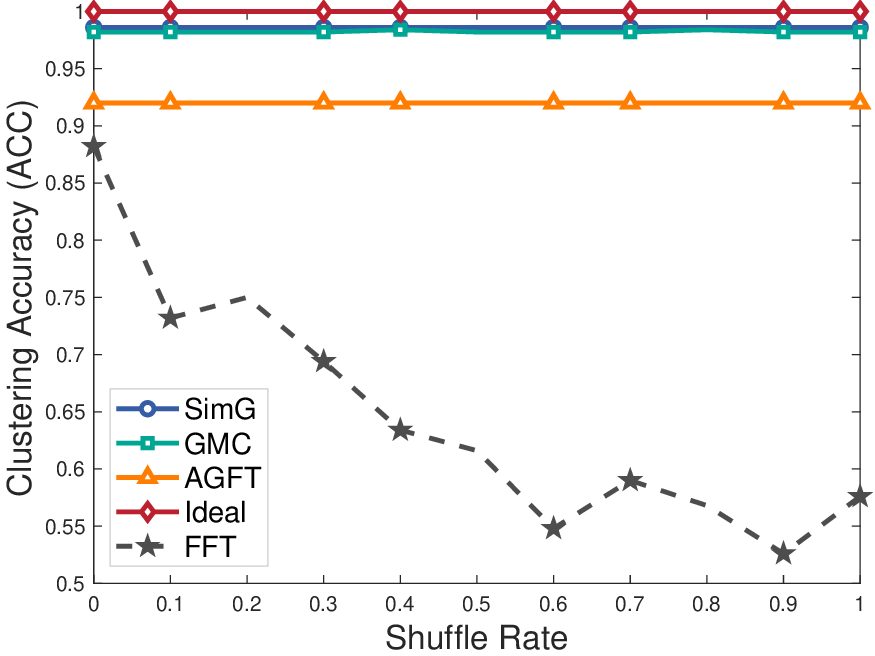}}\hspace{2pt}
    \setlength{\abovecaptionskip}{0.1cm} 
 	\caption{Influence of shuffle rates on different transform variants within the t-SVD-MVC framework.}
	\label{fig:shuffle_rate}
    \vspace{-0.2cm}
\end{figure}

\textbf{Spectral Energy Compaction Analysis.} To investigate how sample reordering affects the learned spectral representation, we evaluate different methods over different shuffle rates and analyze the energy concentration of the transformed tensor. Specifically, the Cumulative Energy Ratio (CER) of the first $K$ non-DC components is defined as
\begin{equation}
\mathrm{CER}(K)
=
\frac{\sum_{i=2}^{K+1} E_i}
{\sum_{i=2}^{n} E_i},
\quad K \in \{1,2,\ldots,n-1\},
\end{equation}
where $E_i=\|\bar{\mathcal{Z}}(:,:, i)\|_F^2$ denotes the energy of the $i$-th frequency slice, $\|\cdot\|_F$ is the Frobenius norm, and the first slice is excluded as the DC component.

As illustrated in \cref{energy_ablation}, with additional results provided in Appendix~\ref{exp_ablation}, the CER profile of FFT changes noticeably as the shuffle rate increases, indicating that its spectral energy allocation is sensitive to the sample order. In particular, sample reordering redistributes energy across a broader range of frequency components, weakening the concentration of energy in the leading components. By contrast, GFT(SimG), GFT(GMC), and GFT(Ideal) exhibit substantially more consistent CER curves across different shuffle rates. This stability suggests that graph-adaptive transforms organize spectral components according to intrinsic sample relationships rather than their sequential indices, thereby preserving a more consistent energy representation under arbitrary permutations.
\begin{figure}[!h]
  \vspace{-0.4cm}
	\centering
	\subfloat[MSRCv1]{\includegraphics[width=0.23\textwidth]{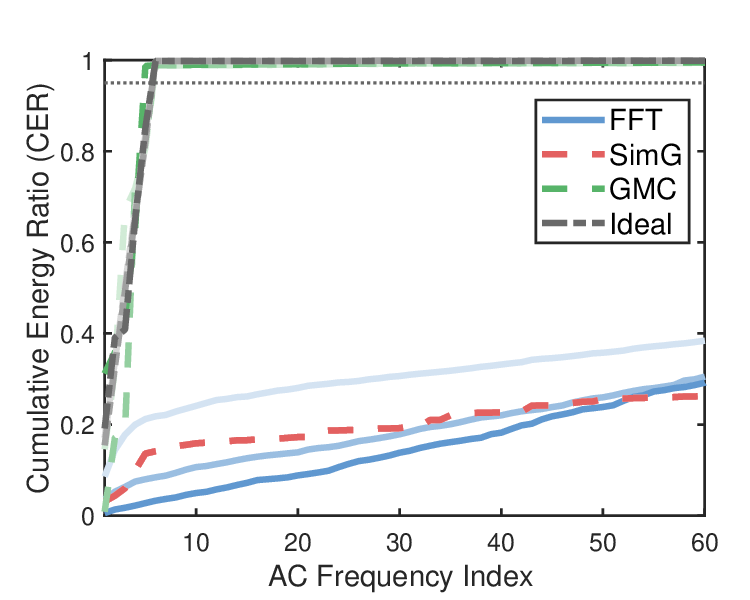}}\hspace{4pt}
	\subfloat[NGs]{\includegraphics[width=0.23\textwidth]{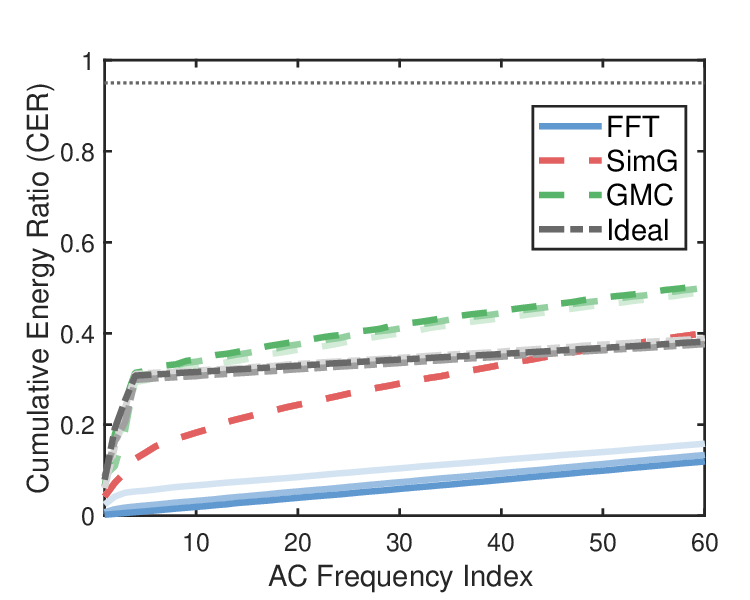}}
    \setlength{\abovecaptionskip}{0.1cm} 
 	\caption{Comparison of the Cumulative Energy Ratio (CER) at shuffle
rates of $0$, $0.5$, and $1$, where darker colors indicate higher
shuffle rates.}
	\label{energy_ablation}

\end{figure}

\textbf{Influence of Different Tensor Transform Bases.}
We visualize the basis matrices on MSRCv1 (\cref{base_ablation}) to examine how different transforms encode sample relationships. For the GFT variants, rows are ordered by ground-truth labels, and the first 30 low-frequency components are shown. The FFT basis (\cref{base_ablation}(a)) exhibits fixed, data-independent periodic patterns that are not aligned with the underlying clusters. As a result, cluster-related variations may spread across a broader spectral range rather than concentrate in the leading components. In contrast, the GFT-based matrices (\cref{base_ablation}(b--d)) display clear block-wise patterns, with several low-frequency vectors approximately aligned with the cluster structure. This alignment can concentrate structural energy within a cluster-relevant subspace. Although the Ideal GFT serves as an oracle reference, SimG and GMC (\cref{base_ablation}(b--c)) recover qualitatively similar patterns, suggesting that GFT remains structure-aware even when built from estimated graphs.
\begin{figure}[!h]

	\centering
	\subfloat[FFT]{\includegraphics[width=0.23\textwidth]{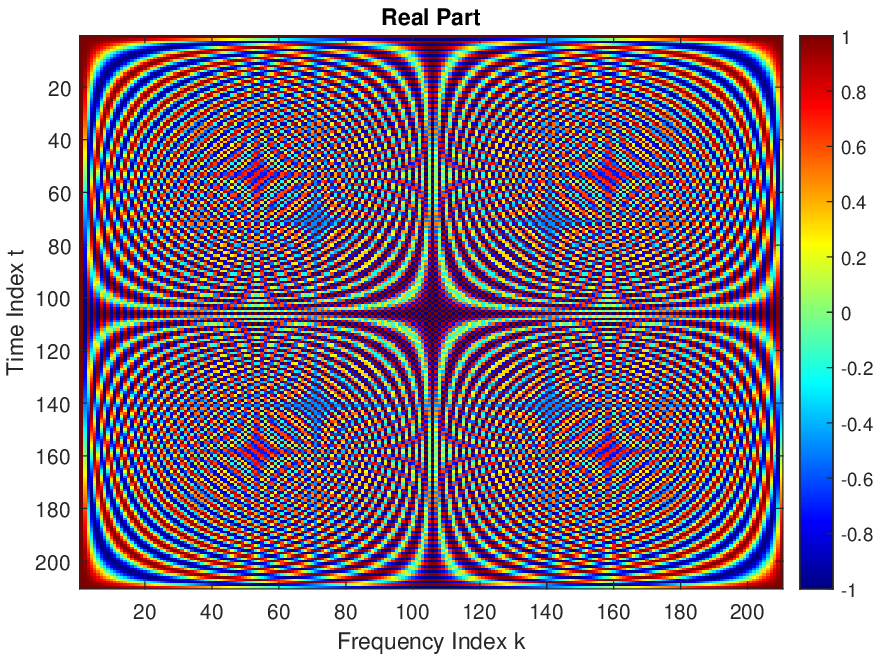}}\hspace{4pt}
	\subfloat[GFT(SimG)]{\includegraphics[width=0.23\textwidth]{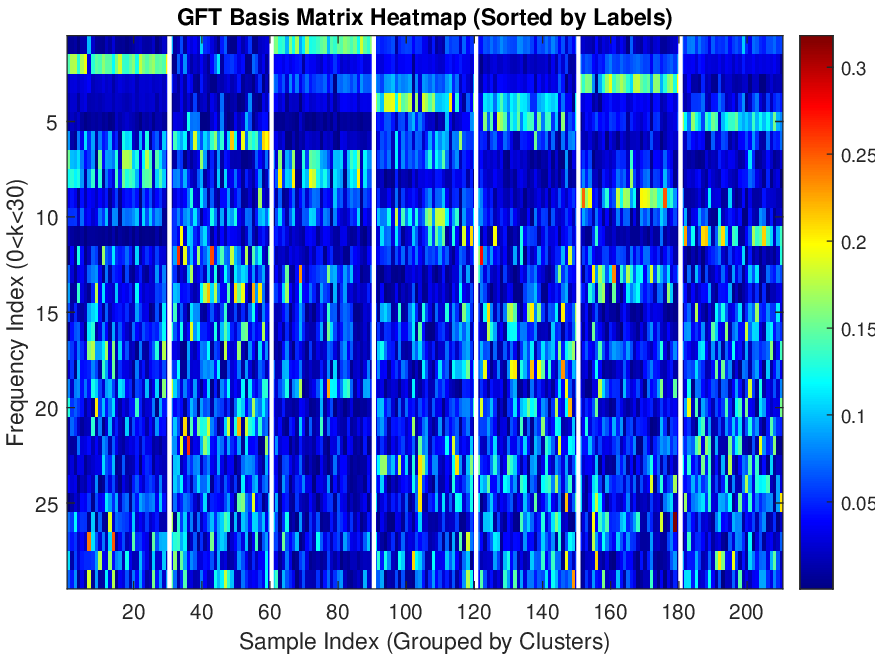}}\hspace{4pt}
    
    \subfloat[GFT(GMC)]
    {\includegraphics[width=0.23\textwidth]{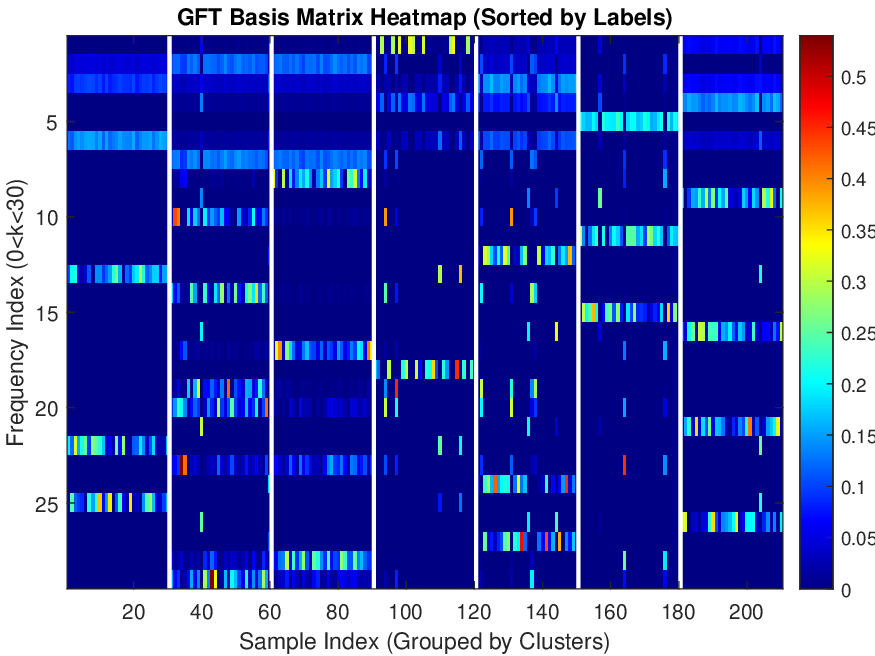}}\hspace{4pt}
    \subfloat[GFT(Ideal)]{\includegraphics[width=0.23\textwidth]{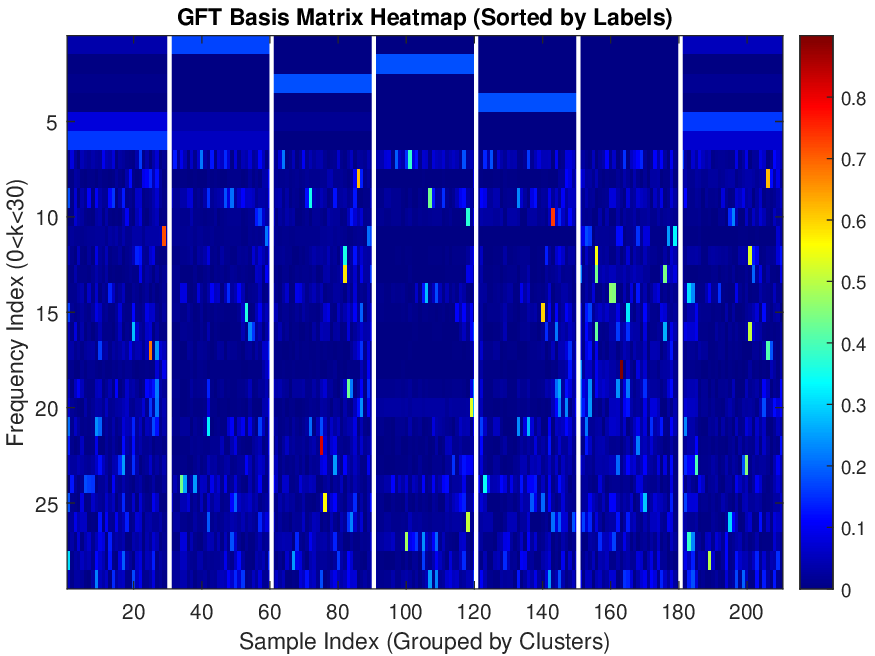}}\hspace{4pt}
    \setlength{\abovecaptionskip}{0.1cm} 
 	\caption{Visualization of the transform basis matrices on the MSRCv1 dataset.}
	\label{base_ablation}

\end{figure}

\textbf{Comparison with Other Transformations.}
To assess the contribution of GFT, we replace it with the random orthogonal basis, FFT, DCT, and wavelet transforms while keeping all other components of ESTMC unchanged. By keeping the optimization framework and all other settings unchanged, this controlled comparison attributes the differences solely to the spectral transform applied along the sample mode. As shown in \cref{tab2}, GFT consistently delivers the best performance across all four datasets. Compared with FFT, it improves ACC by 7.24\%, 10.76\%, 11.43\%, and 12.39\% on MSRCv1, NGs, BBCSport, and Caltech101-20, respectively. By contrast, DCT and wavelet generally underperform FFT, suggesting that replacing FFT with another fixed transform does not improve clustering performance in the shuffled case. The advantage of GFT can be attributed to its data-adaptive basis, which is constructed from the sample graph and thus better preserves the intrinsic relationships among samples under random permutations. These results validate GFT as a more suitable transform for sample-mode low-rank tensor learning.
\begin{table}[t]
\setlength{\abovecaptionskip}{0.1cm} 
\centering
\caption{Comparison of different transforms with the ESTMC framework on shuffled datasets (ACC\%).}
\label{tab2}
  \resizebox{\linewidth}{!}
  {
\begin{tabular}{lcccc}
\toprule
Method & MSRCv1 & NGs & BBCSport & Caltech101-20 \\
\midrule
+Random & 77.62 & 74.80 & 80.33 & 42.46 \\
+FFT & 77.81 & 85.84 & 83.42 & 41.93 \\
+GFT(SimG) & \textbf{85.05} & \textbf{96.60} & \textbf{94.85} & \textbf{54.32} \\
+DCT & 74.76 & 69.24 & 68.64 & 36.48 \\
+Wavelet & 73.33 & 74.80 & 80.33 & 42.85 \\
\bottomrule
\end{tabular}}

\end{table}

\section{Conclusion}
In this paper, we investigated the sample-order sensitivity of t-SVD-based tensorial multi-view clustering frameworks. We showed theoretically that applying a fixed Fourier basis along the sample mode does not preserve the spectral representation under generic sample permutations and can redistribute structural energy across frequency components. To address this issue, we proposed a graph-spectral low-rank tensor learning framework that replaces the fixed Fourier basis with a data-adaptive GFT basis. We further introduced an anchor-based variant that reduces the transform cost to $\mathcal{O}(nk_a)$ and improves scalability. Experiments across multiple datasets and representative TMC baselines demonstrate that the proposed graph-based transforms are substantially more robust to sample reordering while achieving competitive or improved clustering performance. These results highlight the importance of respecting permutation structure when applying transform-domain tensor regularization to unordered samples. This work bridges the gap between graph signal processing and tensor learning, offering a more robust and scalable foundation for the TMC community.

\bibliographystyle{ACM-Reference-Format}
\bibliography{main}

\appendix

\section{Proof of Theorem 3.1}\label{proof1}
\textbf{\cref{fft_ps}}\textbf{(Permutation Sensitivity of FFT-based Tensor Transform)}
Let $\mathcal{T}_{{\rm fft}}(\mathcal{X},\mathbf{F},3)$ denote the Fast Fourier Transform of a representation tensor $\mathcal{X} \in \mathbb{R}^{d \times m \times n}$ along its third dimension. For a generic non-identity permutation matrix $\mathbf{P} \in \{0, 1\}^{n \times n}$, let $\mathcal{X}' = \mathcal{X} \times_3 \mathbf{P}$ be the permuted version of $\mathcal{X}$. For generic $\mathcal{X}$, the following properties hold:

\textbf{1. Spectral Non-Equivariance:} 
\begin{equation}
\begin{aligned}
        &\mathcal{T}_{{\rm fft}}(\mathcal{X}',\mathbf{F},3) \neq\mathcal{T}_{{\rm fft}}(\mathcal{X},\mathbf{F},3),\ \ \\ &\mathcal{T}_{{\rm fft}}(\mathcal{X}',\mathbf{F},3) \neq\mathcal{T}_{{\rm fft}}(\mathcal{X},\mathbf{F},3)\times_3\mathbf{P}, \\
\end{aligned}
\end{equation}

\textbf{2. Energy Redistribution Sensitivity:} Define the non-direct-current (non-DC) energy vector as $\boldsymbol{e}(\mathcal{X}) = [\|\bar{\mathcal{X}}^{(2)}\|^2_F, \dots, \|\bar{\mathcal{X}}^{(n)}\|^2_F]^\top$. The energy distribution is non-invariant under a generic permutation:
\begin{equation}
\boldsymbol{e}(\mathcal{X} \times_3 \mathbf{P}) \neq \boldsymbol{e}(\mathcal{X}).
\end{equation}
This implies that there exists at least one frequency $k \in \{2, \dots, n\}$ such that $\|\bar{\mathcal{X}}^{\prime(k)}\|^2_F \neq \|\bar{\mathcal{X}}^{(k)}\|^2_F$.

\begin{proof} 
The Fast Fourier Transform (FFT) of a tensor $\mathcal{X} \in \mathbb{R}^{d \times m\times n}$ along the third dimension, denoted as $\mathcal{T}_{\text{\rm fft}}(\mathcal{X},\mathbf{F}, 3)$, is mathematically equivalent to the mode-3 product of the tensor with the Discrete Fourier Transform (DFT) matrix $\mathbf{F} \in \mathbb{C}^{n \times n}$,
$\mathcal{T}_{\text{\rm fft}}(\mathcal{X},\mathbf{F}, 3) = \mathcal{X} \times_3 \mathbf{F},$
where the components of $\mathbf{F}$ are defined as: $F_{ab} = \omega^{(a-1)(b-1)}, \quad \text{with } \omega = e^{-{2\pi i}/{n}}$, so the matrix $\mathbf{F}$ can be expressed as:
\begin{equation}
\label{F_matric}
    \mathbf{F} = \begin{bmatrix} 
1 & 1 & 1 & \dots & 1 \\
1 & \omega & \omega^2 & \dots & \omega^{n-1} \\
1 & \omega^2 & \omega^4 & \dots & \omega^{2(n-1)} \\
\vdots & \vdots & \vdots & \ddots & \vdots \\
1 & \omega^{n-1} & \omega^{2(n-1)} & \dots & \omega^{(n-1)(n-1)}
\end{bmatrix}
\end{equation}

\textbf{Proof of spectral non-equivariance.}
Let $\mathcal{X}' = \mathcal{X} \times_3 \mathbf{P}$ be the tensor obtained by permuting the third mode. By the definition of the fast Fourier transform, we have:
\begin{equation}
    \mathcal{T}_{\text{fft}}(\mathcal{X}',\mathbf{F}, 3) = \mathcal{X} \times_3 \mathbf{P} \times_3\mathbf{F} = \mathcal{X} \times_3(\mathbf{FP}).
\end{equation}

So if $\mathbf{P}$ is a non-trivial and generic permutation and $\mathbf{P}\ne \mathbf{I}$, it is easy to derive the following property:
\begin{equation}
    \mathcal{T}_{\text{fft}}(\mathcal{X}',\mathbf{F}, 3) = \mathcal{X} \times_3 \mathbf{P} \times_3\mathbf{F} = \mathcal{X} \times_3(\mathbf{FP})\ne \mathcal{X} \times_3\mathbf{F}=\mathcal{T}_{\rm fft}(\mathcal{X},\mathbf{F},3).
\end{equation}

Then, for the transform to be equivariant in the spectral domain (i.e., 
$\mathcal{T}_{\text{fft}}(\mathcal{X}',\mathbf{F}, 3) = \mathcal{T}_{\text{fft}}(\mathcal{X},\mathbf{F}, 3) \times_3 \mathbf{P}$), 
it requires $\mathbf{FP} = \mathbf{PF}$. Note that $\mathbf{F}$ does not commute with 
$\mathbf{P}$ even when $\mathbf{P}$ is a circular shift $\mathbf{S}$; instead, by the 
DFT shift theorem, $\mathbf{F}$ \emph{diagonalizes} $\mathbf{S}$: 
$\mathbf{F}\mathbf{S}\mathbf{F}^{-1}=\mathbf{D}=\mathrm{diag}(1,\omega,\dots,\omega^{n-1})$. 
Since $\mathbf{D}$ is diagonal with unit-modulus entries, a circular shift only 
rescales each frequency slice by a unit-modulus phase, leaving its energy 
$\|\bar{\mathcal{X}}^{(k)}\|_F$ invariant. For a generic (non-cyclic) permutation 
$\mathbf{P}$, however, $\mathbf{F}\mathbf{P}\mathbf{F}^{-1}$ is a dense unitary 
matrix rather than a diagonal one, so each permuted slice $\bar{\mathcal{X}}'^{(k)}$ 
becomes a mixture of multiple original frequency slices, and $\mathbf{FP}\neq\mathbf{PF}$. 
Consequently:
\begin{equation}\mathcal{T}_{\text{fft}}(\mathcal{X}',\mathbf{F}, 3) \neq \mathcal{T}_{\text{fft}}(\mathcal{X},\mathbf{F}, 3) \times_3\mathbf{P}.
\end{equation}

\textbf{To prove energy redistribution sensitivity.}  Let $\bar{\mathcal{X}} = \mathcal{T}_{\text{fft}}(\mathcal{X}, \mathbf{F}, 3)$ be the spectral tensor. The $k$-th frontal slice $\bar{\mathcal{X}}^{(k)}$ is the $k$-th component of the DFT along the third dimension. For any fiber $\boldsymbol{x} = \mathcal{X}_{i,j,:} \in \mathbb{R}^n$, its $k$-th spectral coefficient is given by the inner product $y_k = \langle \boldsymbol{x}, \boldsymbol{f}_k \rangle$, where $\boldsymbol{f}_k$ is the $k$-th column of the DFT matrix $\mathbf{F}$. The squared Frobenius norm of the $k$-th spectral slice is the sum of the magnitudes of these projections:
\begin{equation}\|\bar{\mathcal{X}}^{(k)}\|_F^2 = \sum_{i,j} \| \langle \mathcal{X}_{i,j,:}, \boldsymbol{f}_k \rangle \|^2.
\end{equation}
When $\mathcal{X}$ is transformed by a generic permutation $\mathbf{P}$, the new fiber becomes $\boldsymbol{x}' = \mathbf{P}\boldsymbol{x}$. The $k$-th spectral component of the permuted fiber is:
\begin{equation}y'_k = \langle \mathbf{P}\boldsymbol{x}, \boldsymbol{f}_k \rangle = \langle \boldsymbol{x}, \mathbf{P}^\top \boldsymbol{f}_k \rangle.
\end{equation}

Let $\boldsymbol{y}=\mathbf{F}\boldsymbol{x}$. After permutation, the Fourier
coefficients become
\[
\boldsymbol{y}'
=
\mathbf{F}\mathbf{P}\boldsymbol{x}
=
\left(\mathbf{F}\mathbf{P}\mathbf{F}^{-1}\right)\boldsymbol{y}.
\]
For a generic permutation matrix $\mathbf{P}$,
$\mathbf{F}\mathbf{P}\mathbf{F}^{-1}$ is not diagonal and generally mixes
multiple Fourier components. Therefore, for generic non-degenerate data, at
least one component-wise energy changes. Summing over all fibers yields
\[
\boldsymbol{e}(\mathcal{X}\times_3\mathbf{P})
\neq
\boldsymbol{e}(\mathcal{X}).
\]
\end{proof}

\section{Proof of Proposition~\ref{FEB}}\label{proof2}

\textbf{Proposition~\ref{FEB} (Frequency Energy Bound)}
Let $\bar{\mathcal{X}} = \mathcal{T}_{\text{\rm fft}}(\mathcal{X},\mathbf{F},3)$ be the Fourier-domain representation of the tensor $\mathcal{X}\in\mathbb{R}^{d \times m\times n}$. The Frobenius norm of the $k$-th alternating-current (AC) frequency component is bounded by the Circular Total Variation (CTV) defined in \cref{TV}:
\begin{equation}\|\bar{\mathcal{X}}^{(k)}\|_F \leq \frac{CTV(\mathcal{X})}{2 \left| \sin(\pi (k-1) / n) \right|}, \quad \forall k=2, \dots, n.
\end{equation}

\begin{proof} We establish the bound by analyzing the relationship between the first-order difference of the tensor and its representation in the Fourier domain. Let $\boldsymbol{x} = \mathcal{X}_{i,j,:}\in\mathbb{R}^n$ denote an arbitrary fiber of $\mathcal{X}$ along the sample mode.

Define the circular difference vector $\Delta \boldsymbol{x} \in \mathbb{R}^n$ such that $(\Delta \boldsymbol{x})_t = x_{t+1} - x_t$ for $t=1, \dots, n-1$, and $(\Delta \boldsymbol{x})_n = x_1 - x_n$. Let $\boldsymbol{y} = \text{fft}(\boldsymbol{x})$ be the Discrete Fourier Transform (DFT) of $\boldsymbol{x}$. By the linearity and the time-shift property of the DFT, the $k$-th component of the Fourier transform of $\Delta \boldsymbol{x}$ is:
\begin{equation}
\mathcal{F}(\Delta \boldsymbol{x})_k = \left( e^{\frac{2\pi i (k-1)}{n}} - 1 \right) y_k,
\end{equation}
where $y_k$ corresponds to the entry $\bar{\mathcal{X}}_{i,j,k}$. Taking the modulus of both sides and applying the identity $|e^{i\theta} - 1| = 2|\sin(\theta/2)|$, we obtain:

\begin{equation}|\mathcal{F}(\Delta \boldsymbol{x})_k| = \left| e^{\frac{2\pi i (k-1)}{n}} - 1 \right| \cdot |y_k| = 2 \left| \sin\left(\frac{\pi (k-1)}{n}\right) \right| \cdot |y_k|.\label{eq:spectral_modulus}
\end{equation}

By the definition of the DFT, the term $\mathcal{F}(\Delta \boldsymbol{x})_k$ is a weighted sum of the differences: $\mathcal{F}(\Delta \boldsymbol{x})_k = \sum_{t=1}^n (\Delta \boldsymbol{x})_t \omega^{(t-1)(k-1)}$. Applying the triangle inequality:
\begin{equation}|\mathcal{F}(\Delta \boldsymbol{x})_k| = \left| \sum_{t=1}^n (x_{t+1} - x_t) e^{-j\frac{2\pi (t-1)(k-1)}{n}} \right| \leq \sum_{t=1}^n |x_{t+1} - x_t|.
\end{equation}

Note that $\sum_{t=1}^n |x_{t+1} - x_t|$ represents the Circular Total Variation. Substituting this into \cref{eq:spectral_modulus} yields the fiber-wise bound:
\begin{equation}
|y_k| = \frac{|\mathcal{F}(\Delta \boldsymbol{x})_k|}{2 \left| \sin(\pi (k-1) / n) \right|} \leq \frac{CTV(\boldsymbol{x})}{2 \left| \sin(\pi (k-1) / n) \right|}.\label{eq:fiber_bound_final}
\end{equation}
The Frobenius norm of the $k$-th spectral slice is defined as the square root of the sum of squared moduli of all fibers at frequency $k$. Setting $\mathcal{X}^{(n+1)}=\mathcal{X}^{(1)}$, we sum the fiber-wise bounds over all spatial indices $(i, j)$:

\begin{equation}
\begin{aligned}
\|\bar{\mathcal{X}}^{(k)}\|_F &= \sqrt{\sum_{i,j} |y_{i,j,k}|^2} \\
&\leq \frac{1}{2 \left| \sin(\pi (k-1) / n) \right|} \sqrt{\sum_{i,j} \left( \sum_{t=1}^n |x_{i,j,t+1} - x_{i,j,t}| \right)^2} \\ 
&\leq \frac{1}{2 \left| \sin(\pi (k-1) / n) \right|} \sum_{t=1}^n \sqrt{\sum_{i,j} |x_{i,j,t+1} - x_{i,j,t}|^2} \\
&= \frac{\sum_{t=1}^n \|\mathcal{X}^{(t+1)} - \mathcal{X}^{(t)}\|_F}{2 \left| \sin(\pi (k-1) / n) \right|} = \frac{CTV(\mathcal{X})}{2 \left| \sin(\pi (k-1) / n) \right|}.
\end{aligned}
\end{equation}
Therefore, the bound holds for all $k \in \{2, \dots, n\}$.
\end{proof}

\section{Additional Empirical Evidence for Theorem 3.1 and Proposition 3.2} 
\label{Aforfft}

In this section, we provide extended empirical evaluations to validate the theoretical claims established in Theorem \ref{fft_ps} (Permutation Sensitivity) and Proposition~\ref{FEB} (Frequency Energy Bound). Our investigation focuses on how the clustering performance of FFT-based tensor models is affected by sample permutations along the third dimension.
 \begin{enumerate}
     \item \textbf{Correlation between CTV and Clustering Performance}. As illustrated in \cref{appendix_ratio1} and \cref{appendix_ratio2}, there is a clear and consistent inverse correlation between clustering performance (NMI/ARI) and the Total Variation ($CTV(\mathcal{X})$) of the representation tensor. In the baseline case (shuffle rate $ = 0$), where samples are ordered by category, the tensor $\mathcal{X}$ maintains high local continuity. This alignment minimizes $CTV(\mathcal{X})$, which, according to the Frequency Energy Bound (FEB), ensures that spectral energy is compactly concentrated in a few low-frequency slices. Such spectral sparsity is the fundamental prerequisite for FFT-based low-rank regularizers to succeed.
     \item \textbf{Impact of Stochastic Shuffling.} As the shuffle rate increases, the random permutation of sample indices disrupts the spatial coherence of the tensor, leading to a sharp rise in $CTV(\mathcal{X})$. This disruption violates the low-rank prior essential for FFT-based models. Specifically, the resulting spectral dispersion increases the effective spectral complexity and causes singular components to be distributed across more frequency slices.
\end{enumerate}

\begin{figure*}[!h]
	\centering
	\subfloat[MSRCv1]{\includegraphics[width=0.23\textwidth]{figure/MSRCv1_TLSpNM_ratio.eps}}\hspace{2pt}
	\subfloat[NGs]{\includegraphics[width=0.23\textwidth]{figure/NGs_TLSpNM_ratio.eps}}\hspace{2pt}
    \subfloat[BBCSport]{\includegraphics[width=0.23\textwidth]{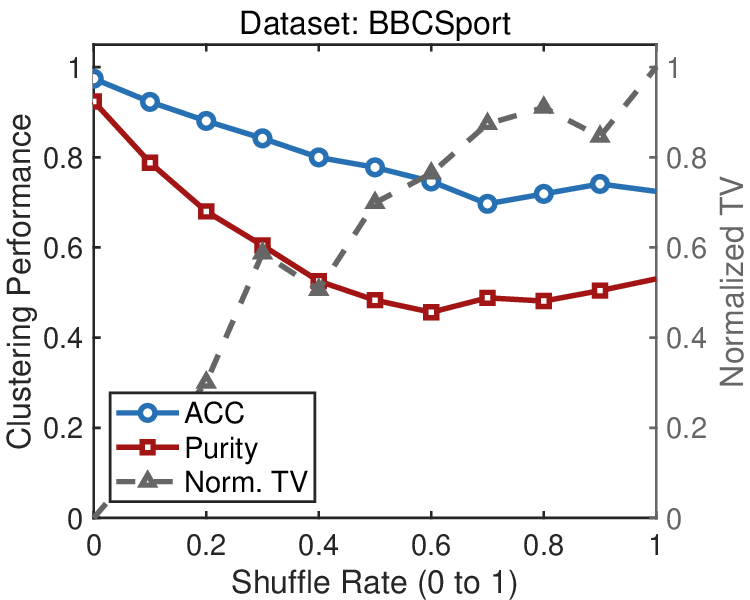}}\hspace{2pt}
	\subfloat[Caltech101-20]{\includegraphics[width=0.23\textwidth]{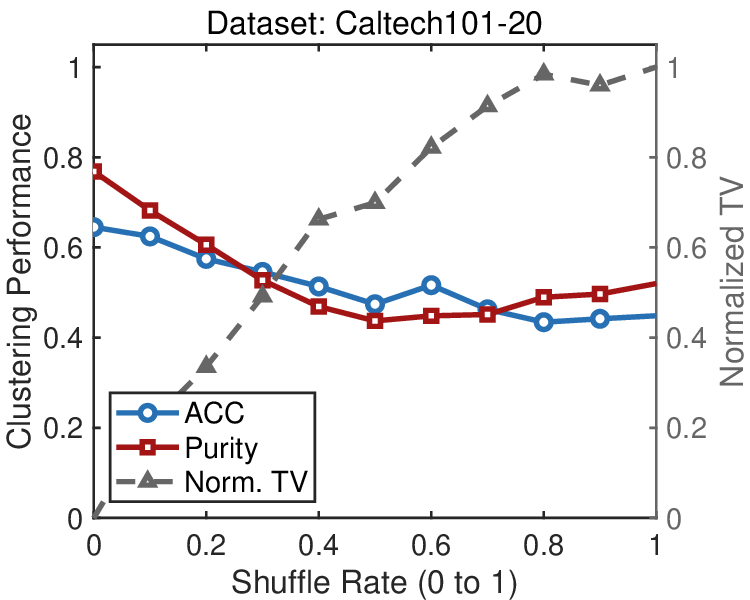}}\hspace{2pt}

 	\caption{Evaluation of the TLSpNM method on four datasets with different shuffle rates. A ratio of 0 indicates the baseline, where samples are ordered by category. }
	\label{appendix_ratio1}
\end{figure*}
\begin{figure*}[!h]
	\centering
	\subfloat[MSRCv1]{\includegraphics[width=0.23\textwidth]{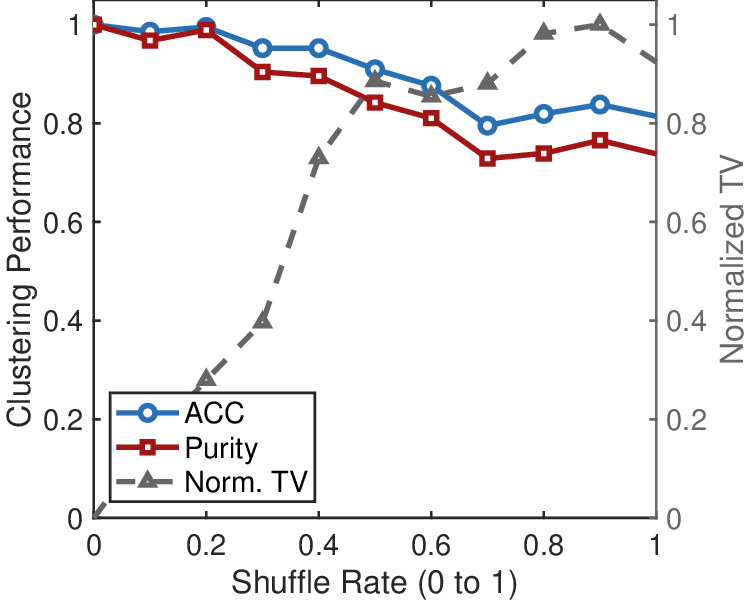}}\hspace{2pt}
	\subfloat[NGs]{\includegraphics[width=0.23\textwidth]{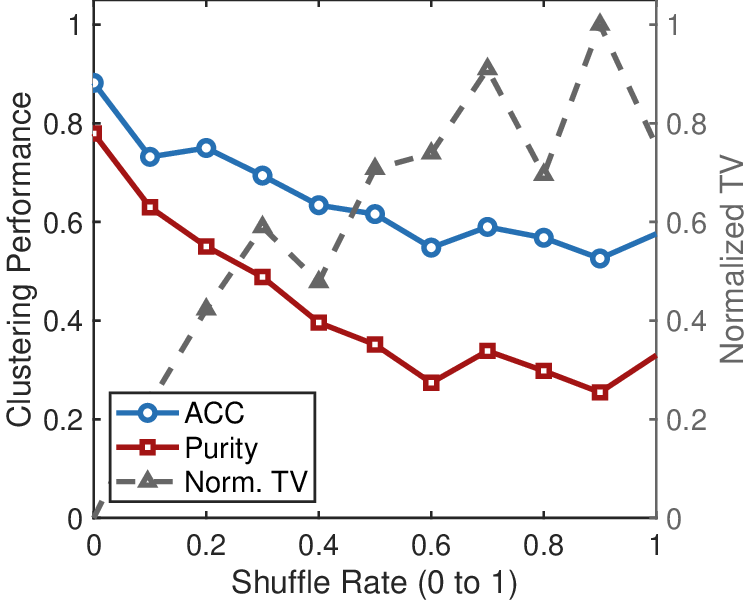}}\hspace{2pt}
    \subfloat[BBCSport]{\includegraphics[width=0.23\textwidth]{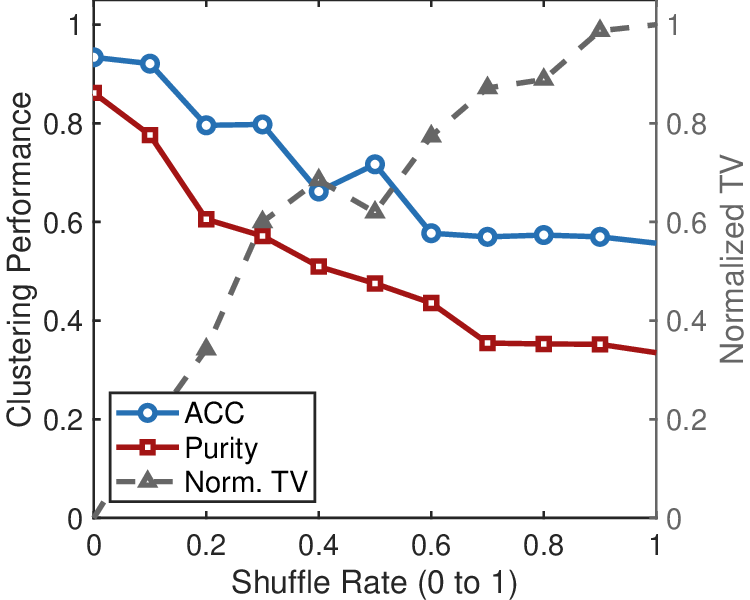}}\hspace{2pt}
	\subfloat[Caltech101-20]{\includegraphics[width=0.23\textwidth]{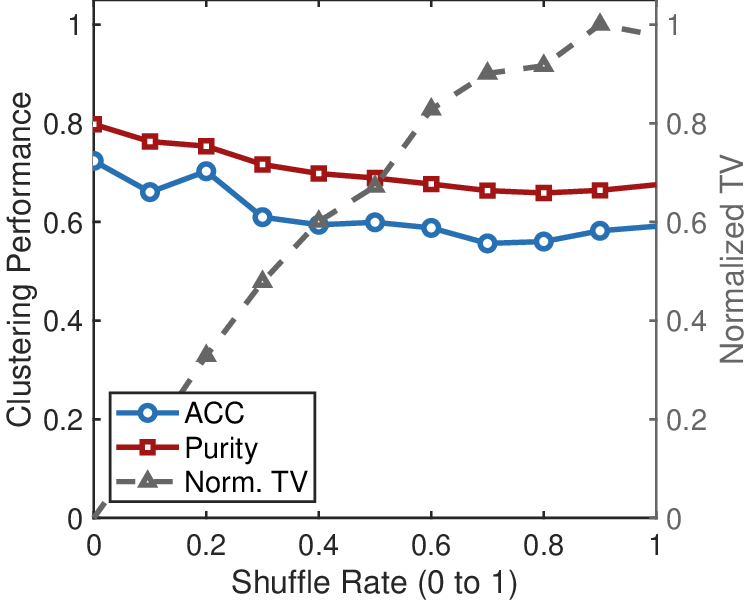}}\hspace{2pt}
 	\caption{Evaluation of the t-SVD-MVC method on four datasets with different shuffle rates. A ratio of 0 indicates the baseline, where samples are ordered by category. }
	\label{appendix_ratio2}
\end{figure*}

 \textbf{Visualizing the Frequency Energy Distribution.} To further investigate the mechanism behind performance degradation, we visualize the Frobenius norm of the spectral slices $\|\bar{\mathcal{X}}^{(k)}\|_F$ for AC frequency indices $1 < k < 50$. The results across four benchmark datasets (MSRCv1, NGs, BBCSport, Caltech101-20) are presented in \cref{fig:ratio1_appendix}. Specifically, for multi-view data $\{\mathbf{X}^{(v)}\}_{v=1}^m$ with $\mathbf{X}^{(v)} \in \mathbb{R}^{d_v \times n}$, we construct the self-representation tensor $\mathcal{X} \in \mathbb{R}^{n \times m \times n}$ via $\mathcal{X} = \Phi(\mathbf{Z}^{(1)}, \dots, \mathbf{Z}^{(m)})$, where $\mathbf{Z}^{(v)} = (\mathbf{X}^{(v)})^\top \mathbf{X}^{(v)}$ and $\Phi(\cdot)$ denotes the merging and orienting operation \cite{xie2018unifying}. From the visualization, several key insights can be drawn:
 \begin{enumerate}
     \item \textbf{Spectral Energy Compaction:} In the baseline case (shuffle rate = 0, ordered by category), the spectral energy is highly concentrated in the low-frequency bins. The magnitude $\|\bar{\mathcal{X}}^{(k)}\|_F$ exhibits a precipitous decay as $k$ increases, which is consistent with our Frequency Energy Bound (FEB) theorem. This energy compaction makes many high-frequency slices have relatively small singular values, allowing low-rank shrinkage to attenuate these components while preserving the dominant spectral structure.
     \item \textbf{Spectral Leakage and Dispersion:} As the shuffle rate increases, a significant "spectral leakage" phenomenon is observed. The energy previously concentrated in the fundamental frequencies is redistributed across the high-frequency spectrum. When the shuffle rate equals 1.0, the spectrum becomes remarkably flat, resembling white noise.
 \end{enumerate}
\begin{figure*}[!h]
	\centering
	\subfloat[MSRCv1]{\includegraphics[width=0.23\textwidth]{figure/AC_MSRCv1.eps}}\hspace{2pt}
	\subfloat[NGs]{\includegraphics[width=0.23\textwidth]{figure/AC_NGs.eps}}\hspace{2pt}
    \subfloat[BBCSport]{\includegraphics[width=0.23\textwidth]{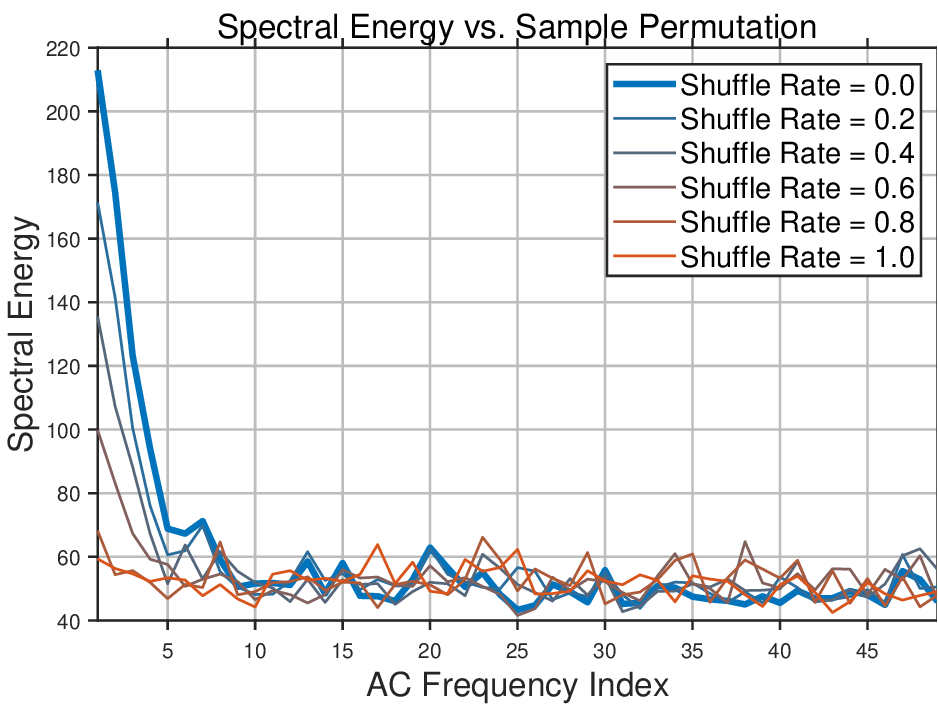}}\hspace{2pt}
	\subfloat[Caltech101-20]{\includegraphics[width=0.23\textwidth]{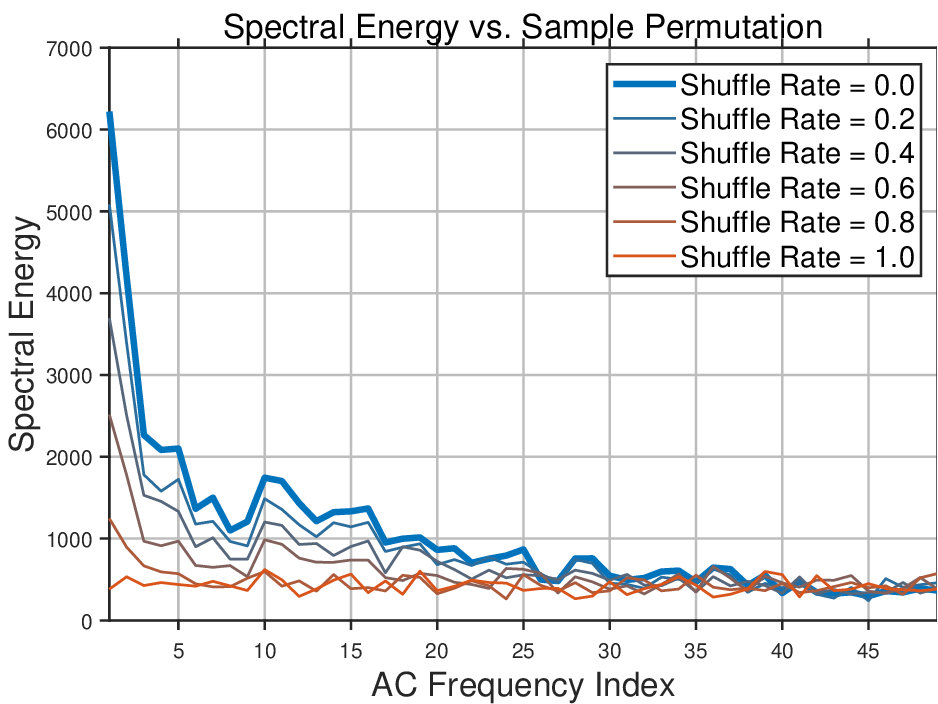}}\hspace{2pt}
 	\caption{Spectral magnitude $\|\bar{\mathcal{X}}^{(k)}\|_F$ for AC frequencies ($k>1$) on four datasets.}
	\label{fig:ratio1_appendix}

\end{figure*}

\section{Proof of \cref{theorem_gft}}\label{proof3}

\textbf{\cref{theorem_gft} (Permutation Equivariance of GFT-based Tensor Transform)}
Let $\mathcal{X} \in \mathbb{R}^{d \times m \times n}$ be a third-order tensor and $\mathbf{L} = \mathbf{U}\mathbf{\Lambda}\mathbf{U}^\top \in \mathbb{R}^{n \times n}$ be the graph Laplacian characterizing the relationships along the third mode. For any permutation matrix $\mathbf{P} \in \{0,1\}^{n \times n}$, let $\mathcal{X}' = \mathcal{X} \times_3 \mathbf{P}$ and $\mathbf{L}' = \mathbf{P}\mathbf{L}\mathbf{P}^\top = \mathbf{U}'\mathbf{\Lambda}\mathbf{U}'^\top$ be the permuted tensor and Laplacian, respectively, where the eigenvalues are ordered consistently. The following properties hold:

\textbf{1. Spectral Subspace Equivariance:} There exists a block-diagonal orthogonal matrix $\mathbf{Q}$, whose blocks correspond to the eigenspaces associated with repeated eigenvalues, such that $\mathbf{U}' = \mathbf{P}\mathbf{U}\mathbf{Q}$ and
\begin{equation}
\mathcal{T}_{\rm gft}(\mathcal{X}',\mathbf{U}'^\top,3)
=
\mathcal{T}_{\rm gft}(\mathcal{X},\mathbf{U}^\top,3)
\times_3 \mathbf{Q}^\top.
\end{equation}

\textbf{2. GTR Rank Stability:} If $\mathbf{L}$ has a simple spectrum, the following equality holds:
\begin{equation}
\left\|\mathcal{X}'\right\|_{\rm gft\text{-}GTR}
=
\left\|\mathcal{X}\right\|_{\rm gft\text{-}GTR}.
\end{equation}

\begin{proof}
\textbf{Proof of Property 1:}
According to Definition~\ref{gft_de}, the GFT-based transform is computed through the mode-$3$ product:
\begin{equation}
\mathcal{T}_{\rm gft}(\mathcal{X},\mathbf{U}^\top,3)
=
\mathcal{X}\times_3\mathbf{U}^\top,
\end{equation}
where $\mathbf{U}\in\mathbb{R}^{n\times n}$ is an orthogonal eigenvector matrix of
$\mathbf{L}=\mathbf{U}\mathbf{\Lambda}\mathbf{U}^\top$.

For the permuted Laplacian, we have
\begin{equation}
\mathbf{L}'
=
\mathbf{P}\mathbf{L}\mathbf{P}^\top
=
(\mathbf{P}\mathbf{U})\mathbf{\Lambda}(\mathbf{P}\mathbf{U})^\top.
\end{equation}
Thus, $\mathbf{P}\mathbf{U}$ is an orthogonal eigenvector matrix of $\mathbf{L}'$. When repeated eigenvalues are present, the orthonormal basis within each corresponding eigenspace is not unique. Since $\mathbf{U}'$ and $\mathbf{P}\mathbf{U}$ are orthonormal eigenvector matrices associated with the same consistently ordered eigenvalues, there exists a block-diagonal orthogonal matrix $\mathbf{Q}$, with one block for each eigenspace, such that
\begin{equation}
\mathbf{U}'
=
\mathbf{P}\mathbf{U}\mathbf{Q}.
\end{equation}
Each block of $\mathbf{Q}$ acts only within an eigenspace associated with the same eigenvalue; hence, $\mathbf{Q}\mathbf{\Lambda}\mathbf{Q}^\top=\mathbf{\Lambda}$.

Using $\mathcal{X}'=\mathcal{X}\times_3\mathbf{P}$ and the mode-$3$ product identity
\begin{equation}
(\mathcal{A}\times_3\mathbf{B})\times_3\mathbf{C}
=
\mathcal{A}\times_3(\mathbf{C}\mathbf{B}),
\end{equation}
we obtain
\begin{align}
\mathcal{T}_{\rm gft}(\mathcal{X}',\mathbf{U}'^\top,3)
&=
\mathcal{X}'\times_3\mathbf{U}'^\top \\
&=
(\mathcal{X}\times_3\mathbf{P})
\times_3
(\mathbf{P}\mathbf{U}\mathbf{Q})^\top \\
&=
\mathcal{X}\times_3
\left(
\mathbf{Q}^\top\mathbf{U}^\top\mathbf{P}^\top\mathbf{P}
\right).
\end{align}
Because $\mathbf{P}$ is orthogonal, $\mathbf{P}^\top\mathbf{P}=\mathbf{I}$, and therefore
\begin{align}
\mathcal{T}_{\rm gft}(\mathcal{X}',\mathbf{U}'^\top,3)
&=
\mathcal{X}\times_3
\left(
\mathbf{Q}^\top\mathbf{U}^\top
\right) \\
&=
\left(
\mathcal{X}\times_3\mathbf{U}^\top
\right)
\times_3\mathbf{Q}^\top \\
&=
\mathcal{T}_{\rm gft}(\mathcal{X},\mathbf{U}^\top,3)
\times_3\mathbf{Q}^\top.
\end{align}
Hence, a mode-$3$ permutation preserves the GFT representation up to an orthogonal change of basis within each repeated-eigenvalue eigenspace.

By contrast, consider a fixed transform basis $\mathbf{V}$ that does not adapt to the sample permutation, with
\begin{equation}
\mathcal{T}_{\rm fix}(\mathcal{X})
=
\mathcal{X}\times_3\mathbf{V}^\top.
\end{equation}
After permuting the third mode, its transform becomes
\begin{equation}
\mathcal{T}_{\rm fix}(\mathcal{X}\times_3\mathbf{P})
=
\mathcal{X}\times_3
\left(
\mathbf{V}^\top\mathbf{P}
\right).
\end{equation}
Since $\mathbf{V}$ does not adapt to $\mathbf{P}$, this expression generally differs from
$\mathcal{X}\times_3\mathbf{V}^\top$. Therefore, a fixed basis such as the DFT basis generally produces order-dependent spectral coefficients and spectral energy distributions.

\textbf{Proof of Property 2:}
According to Definition~\ref{GT-rank}, the Generalized Tensor Rank is defined as the average of the matrix ranks of the frontal slices in the GFT domain:
\begin{equation}
\left\|\mathcal{X}\right\|_{\rm gft\text{-}GTR}
=
\frac{1}{n}
\sum_{k=1}^{n}
\operatorname{rank}
\left(
\mathcal{T}_{\rm gft}(\mathcal{X},\mathbf{U}^\top,3)^{(k)}
\right).
\end{equation}

If $\mathbf{L}$ has a simple spectrum, every eigenspace is one-dimensional. Consequently, the block-diagonal orthogonal matrix $\mathbf{Q}$ in Property~1 reduces to a diagonal sign matrix:
\begin{equation}
\mathbf{Q}
=
\operatorname{diag}(s_1,\ldots,s_n),
\qquad
s_k\in\{-1,1\}.
\end{equation}

Let
\begin{equation}
\bar{\mathcal{X}}
=
\mathcal{T}_{\rm gft}(\mathcal{X},\mathbf{U}^\top,3),
\qquad
\bar{\mathcal{X}}'
=
\mathcal{T}_{\rm gft}(\mathcal{X}',\mathbf{U}'^\top,3).
\end{equation}
From Property~1,
\begin{equation}
\bar{\mathcal{X}}'
=
\bar{\mathcal{X}}\times_3\mathbf{Q}^\top.
\end{equation}
Since $\mathbf{Q}^\top=\mathbf{Q}$ is diagonal, the $k$-th frontal slice satisfies
\begin{equation}
\bar{\mathcal{X}}'^{(k)}
=
s_k\bar{\mathcal{X}}^{(k)},
\qquad
k=1,\ldots,n.
\end{equation}
Because multiplication by a nonzero scalar does not change matrix rank,
\begin{equation}
\operatorname{rank}
\left(
\bar{\mathcal{X}}'^{(k)}
\right)
=
\operatorname{rank}
\left(
\bar{\mathcal{X}}^{(k)}
\right),
\qquad
k=1,\ldots,n.
\end{equation}
Summing over all frontal slices and dividing by $n$ gives
\begin{equation}
\frac{1}{n}
\sum_{k=1}^{n}
\operatorname{rank}
\left(
\bar{\mathcal{X}}'^{(k)}
\right)
=
\frac{1}{n}
\sum_{k=1}^{n}
\operatorname{rank}
\left(
\bar{\mathcal{X}}^{(k)}
\right).
\end{equation}
Therefore,
\begin{equation}
\left\|\mathcal{X}'\right\|_{\rm gft\text{-}GTR}
=
\left\|\mathcal{X}\right\|_{\rm gft\text{-}GTR}.
\end{equation}
Thus, the GTR is invariant to mode-$3$ permutations under the simple spectrum condition. The same argument applies to any slice-wise rank surrogate that depends only on the singular values of each spectral slice, since multiplying a slice by
$-1$ does not change its singular values.
\end{proof}

\section{Bounded Equivariance Error under a Non-simple Spectrum}
\label{app:bounded_equivariance_error}
We further analyze the permutation behavior of the GFT-based low-rank learning procedure in \cref{alg:gft_low_rank} when the graph Laplacian has repeated eigenvalues. Let
$\mathcal{D}_{\lambda,\psi}$ denote the matrix spectral shrinkage operator
\begin{equation}
\mathcal{D}_{\lambda,\psi}(\mathbf{A})
=
\mathbf{R}
\operatorname{Prox}_{\psi,\lambda}({\Sigma})
\mathbf{V}^{\top},
\end{equation}
where
$\mathbf{A}=\mathbf{R}{\Sigma}\mathbf{V}^{\top}$
is an SVD of $\mathbf{A}$. Let
$\mathscr{D}_{\lambda,\psi}$ denote its tensor extension that applies
$\mathcal{D}_{\lambda,\psi}$ independently to every frontal slice. Given a
graph Fourier basis $\mathbf{U}$, define the GFT-based low-rank mapping as
\begin{equation}
\mathcal{F}_{\lambda,\psi}^{\mathbf{U}}(\mathcal{Y})
=
\mathscr{D}_{\lambda,\psi}
\left(
\mathcal{Y}\times_3\mathbf{U}^{\top}
\right)
\times_3\mathbf{U}.
\label{eq:gft_low_rank_mapping}
\end{equation}
Accordingly, Algorithm~\ref{alg:gft_low_rank} can be written as
$\mathcal{Z}=\mathcal{F}_{\lambda,\psi}^{\mathbf{U}}(\mathcal{Y})$.

For a permutation matrix $\mathbf{P}$, let
$\mathcal{Y}'=\mathcal{Y}\times_3\mathbf{P}$ and
$\mathbf{U}'=\mathbf{P}\mathbf{U}\mathbf{Q}$, where $\mathbf{Q}$ is the
block-diagonal orthogonal matrix induced by the non-unique choice of
orthonormal bases within repeated-eigenvalue eigenspaces. We define the
permutation-equivariance error as
\begin{equation}
\mathcal{E}_{\rm equiv}(\mathcal{Y},\mathbf{P})
=
\left\|
\mathcal{F}_{\lambda,\psi}^{\mathbf{U}'}
\left(
\mathcal{Y}\times_3\mathbf{P}
\right)
-
\mathcal{F}_{\lambda,\psi}^{\mathbf{U}}(\mathcal{Y})
\times_3\mathbf{P}
\right\|_F.
\label{eq:gft_equivariance_error}
\end{equation}

\begin{theorem}
\textbf{(Repeated-Eigenspace Energy Bound for the Equivariance Error)}
\label{theorem_gft_bounded_equivariance}

Let $\mathcal{Y}\in\mathbb{R}^{d\times m\times n}$ and let
$\mathcal{Y}'=\mathcal{Y}\times_3\mathbf{P}$. Suppose that the corresponding
graph Fourier bases satisfy
\begin{equation}
\mathbf{U}'
=
\mathbf{P}\mathbf{U}\mathbf{Q},
\end{equation}
where
\begin{equation}
\mathbf{Q}
=
\operatorname{blkdiag}
\left(
\mathbf{Q}_1,\ldots,\mathbf{Q}_s
\right)
\end{equation}
and each orthogonal block
$\mathbf{Q}_{\ell}\in\mathbb{R}^{r_{\ell}\times r_{\ell}}$
acts within the eigenspace associated with the distinct eigenvalue
$\lambda_{\ell}$. Let $\mathcal{I}_{\ell}$ denote the corresponding spectral
index set, with
$r_{\ell}=|\mathcal{I}_{\ell}|$, and define
\begin{equation}
\mathcal{R}
=
\left\{
\ell:
r_{\ell}>1
\right\}
\end{equation}
as the set of repeated-eigenvalue eigenspaces.

Let
\begin{equation}
\bar{\mathcal{Y}}
=
\mathcal{Y}\times_3\mathbf{U}^{\top}
\end{equation}
and define the spectral energy contained in the repeated-eigenvalue
eigenspaces as
\begin{equation}
E_{\mathcal{R}}(\mathcal{Y})
=
\sum_{\ell\in\mathcal{R}}
\left\|
\bar{\mathcal{Y}}(:,:, \mathcal{I}_{\ell})
\right\|_F^2
=
\sum_{\ell\in\mathcal{R}}
\sum_{k\in\mathcal{I}_{\ell}}
\left\|
\bar{\mathcal{Y}}^{(k)}
\right\|_F^2.
\label{eq:repeated_eigenspace_energy}
\end{equation}

Assume that
$\mathcal{D}_{\lambda,\psi}(\mathbf{0})=\mathbf{0}$,
that $\mathcal{D}_{\lambda,\psi}$ is $L_{\psi}$-Lipschitz continuous with
respect to the Frobenius norm, and that it is equivariant to sign changes:
\begin{equation}
\mathcal{D}_{\lambda,\psi}(s\mathbf{A})
=
s\mathcal{D}_{\lambda,\psi}(\mathbf{A}),
\qquad
s\in\{-1,1\}.
\end{equation}
Then the permutation-equivariance error defined in
\cref{eq:gft_equivariance_error} satisfies
\begin{equation}
\begin{aligned}
\mathcal{E}_{\rm equiv}(\mathcal{Y},\mathbf{P})
&\leq
2L_{\psi}
\left[
\sum_{\ell\in\mathcal{R}}
\left\|
\bar{\mathcal{Y}}(:,:, \mathcal{I}_{\ell})
\right\|_F^2
\min
\left\{
1,
\left\|
\mathbf{Q}_{\ell}
-
\mathbf{I}_{r_{\ell}}
\right\|_2^2
\right\}
\right]^{1/2}
\\
&\leq
2L_{\psi}
\sqrt{
E_{\mathcal{R}}(\mathcal{Y})
}.
\end{aligned}
\label{eq:repeated_eigenspace_equivariance_bound}
\end{equation}
For standard singular value soft-thresholding, $L_{\psi}=1$. In particular,
if the graph Laplacian has a simple spectrum, then
$\mathcal{R}=\varnothing$ and
\begin{equation}
\mathcal{E}_{\rm equiv}(\mathcal{Y},\mathbf{P})
=
0.
\end{equation}
\end{theorem}

\begin{proof}
Let
\begin{equation}
\bar{\mathcal{Y}}
=
\mathcal{Y}\times_3\mathbf{U}^{\top}
\end{equation}
and
\begin{equation}
\bar{\mathcal{Y}}'
=
\left(
\mathcal{Y}\times_3\mathbf{P}
\right)
\times_3\mathbf{U}'^{\top}.
\end{equation}
Using
$\mathbf{U}'=\mathbf{P}\mathbf{U}\mathbf{Q}$,
$\mathbf{P}^{\top}\mathbf{P}=\mathbf{I}$, and the mode-$3$ product identity,
we obtain
\begin{align}
\bar{\mathcal{Y}}'
&=
\left(
\mathcal{Y}\times_3\mathbf{P}
\right)
\times_3
\left(
\mathbf{P}\mathbf{U}\mathbf{Q}
\right)^{\top}
\\
&=
\mathcal{Y}
\times_3
\left(
\mathbf{Q}^{\top}
\mathbf{U}^{\top}
\mathbf{P}^{\top}
\mathbf{P}
\right)
\\
&=
\bar{\mathcal{Y}}
\times_3\mathbf{Q}^{\top}.
\end{align}

Using the definition in \cref{eq:gft_equivariance_error}, together with
$\bar{\mathcal{Y}}'
=
\bar{\mathcal{Y}}\times_3\mathbf{Q}^{\top}$
and the orthogonality of $\mathbf{P}\mathbf{U}$, the equivariance error can be
written in the graph-spectral domain as
\begin{equation}
\mathcal{E}_{\rm equiv}(\mathcal{Y},\mathbf{P})
=
\left\|
\mathscr{D}_{\lambda,\psi}
\left(
\bar{\mathcal{Y}}\times_3\mathbf{Q}^{\top}
\right)
\times_3\mathbf{Q}
-
\mathscr{D}_{\lambda,\psi}
\left(
\bar{\mathcal{Y}}
\right)
\right\|_F.
\label{eq:repeated_error_spectral_form}
\end{equation}

For each distinct eigenvalue $\lambda_{\ell}$, define
\begin{equation}
\bar{\mathcal{Y}}_{\ell}
=
\bar{\mathcal{Y}}(:,:, \mathcal{I}_{\ell}).
\end{equation}
Because $\mathbf{Q}$ is block diagonal and
$\mathscr{D}_{\lambda,\psi}$ acts independently on frontal slices, the error
in \cref{eq:repeated_error_spectral_form} decomposes orthogonally across the
eigenspaces:
\begin{equation}
\mathcal{E}_{\rm equiv}^2(\mathcal{Y},\mathbf{P})
=
\sum_{\ell=1}^{s}
\varepsilon_{\ell}^2,
\label{eq:eigenspace_error_decomposition}
\end{equation}
where
\begin{equation}
\varepsilon_{\ell}
=
\left\|
\mathscr{D}_{\lambda,\psi}
\left(
\bar{\mathcal{Y}}_{\ell}
\times_3
\mathbf{Q}_{\ell}^{\top}
\right)
\times_3
\mathbf{Q}_{\ell}
-
\mathscr{D}_{\lambda,\psi}
\left(
\bar{\mathcal{Y}}_{\ell}
\right)
\right\|_F.
\label{eq:block_equivariance_error}
\end{equation}

If $r_{\ell}=1$, then
$\mathbf{Q}_{\ell}=[s_{\ell}]$ with
$s_{\ell}\in\{-1,1\}$. By sign equivariance,
\begin{equation}
\mathscr{D}_{\lambda,\psi}
\left(
s_{\ell}\bar{\mathcal{Y}}_{\ell}
\right)
=
s_{\ell}
\mathscr{D}_{\lambda,\psi}
\left(
\bar{\mathcal{Y}}_{\ell}
\right).
\end{equation}
Since $s_{\ell}^2=1$, it follows that
\begin{equation}
\varepsilon_{\ell}
=
0,
\qquad
r_{\ell}=1.
\label{eq:simple_eigenspace_zero_error}
\end{equation}
Therefore, only eigenspaces associated with repeated eigenvalues can
contribute to the equivariance error.

We next bound $\varepsilon_{\ell}$ for
$\ell\in\mathcal{R}$. Adding and subtracting
$\mathscr{D}_{\lambda,\psi}
(\bar{\mathcal{Y}}_{\ell})
\times_3\mathbf{Q}_{\ell}$
in \cref{eq:block_equivariance_error} and applying the triangle inequality
give
\begin{align}
\varepsilon_{\ell}
&\leq
\left\|
\left[
\mathscr{D}_{\lambda,\psi}
\left(
\bar{\mathcal{Y}}_{\ell}
\times_3
\mathbf{Q}_{\ell}^{\top}
\right)
-
\mathscr{D}_{\lambda,\psi}
\left(
\bar{\mathcal{Y}}_{\ell}
\right)
\right]
\times_3
\mathbf{Q}_{\ell}
\right\|_F
\\
&\quad+
\left\|
\mathscr{D}_{\lambda,\psi}
\left(
\bar{\mathcal{Y}}_{\ell}
\right)
\times_3
\left(
\mathbf{Q}_{\ell}
-
\mathbf{I}_{r_{\ell}}
\right)
\right\|_F.
\end{align}
Since $\mathbf{Q}_{\ell}$ is orthogonal and
$\mathcal{D}_{\lambda,\psi}$ is $L_{\psi}$-Lipschitz,
\begin{align}
&
\left\|
\left[
\mathscr{D}_{\lambda,\psi}
\left(
\bar{\mathcal{Y}}_{\ell}
\times_3
\mathbf{Q}_{\ell}^{\top}
\right)
-
\mathscr{D}_{\lambda,\psi}
\left(
\bar{\mathcal{Y}}_{\ell}
\right)
\right]
\times_3
\mathbf{Q}_{\ell}
\right\|_F
\\
&=
\left\|
\mathscr{D}_{\lambda,\psi}
\left(
\bar{\mathcal{Y}}_{\ell}
\times_3
\mathbf{Q}_{\ell}^{\top}
\right)
-
\mathscr{D}_{\lambda,\psi}
\left(
\bar{\mathcal{Y}}_{\ell}
\right)
\right\|_F
\\
&\leq
L_{\psi}
\left\|
\bar{\mathcal{Y}}_{\ell}
\times_3
\left(
\mathbf{Q}_{\ell}^{\top}
-
\mathbf{I}_{r_{\ell}}
\right)
\right\|_F
\\
&\leq
L_{\psi}
\left\|
\bar{\mathcal{Y}}_{\ell}
\right\|_F
\left\|
\mathbf{Q}_{\ell}
-
\mathbf{I}_{r_{\ell}}
\right\|_2.
\end{align}
Moreover, because
$\mathcal{D}_{\lambda,\psi}(\mathbf{0})=\mathbf{0}$,
\begin{equation}
\left\|
\mathscr{D}_{\lambda,\psi}
\left(
\bar{\mathcal{Y}}_{\ell}
\right)
\right\|_F
\leq
L_{\psi}
\left\|
\bar{\mathcal{Y}}_{\ell}
\right\|_F.
\end{equation}
Hence,
\begin{align}
&
\left\|
\mathscr{D}_{\lambda,\psi}
\left(
\bar{\mathcal{Y}}_{\ell}
\right)
\times_3
\left(
\mathbf{Q}_{\ell}
-
\mathbf{I}_{r_{\ell}}
\right)
\right\|_F
\\
&\leq
L_{\psi}
\left\|
\bar{\mathcal{Y}}_{\ell}
\right\|_F
\left\|
\mathbf{Q}_{\ell}
-
\mathbf{I}_{r_{\ell}}
\right\|_2.
\end{align}
Combining these inequalities yields
\begin{equation}
\varepsilon_{\ell}
\leq
2L_{\psi}
\left\|
\bar{\mathcal{Y}}_{\ell}
\right\|_F
\left\|
\mathbf{Q}_{\ell}
-
\mathbf{I}_{r_{\ell}}
\right\|_2.
\label{eq:block_q_dependent_bound}
\end{equation}

A uniform bound follows directly from
\cref{eq:block_equivariance_error}:
\begin{align}
\varepsilon_{\ell}
&\leq
\left\|
\mathscr{D}_{\lambda,\psi}
\left(
\bar{\mathcal{Y}}_{\ell}
\times_3
\mathbf{Q}_{\ell}^{\top}
\right)
\right\|_F
+
\left\|
\mathscr{D}_{\lambda,\psi}
\left(
\bar{\mathcal{Y}}_{\ell}
\right)
\right\|_F
\\
&\leq
2L_{\psi}
\left\|
\bar{\mathcal{Y}}_{\ell}
\right\|_F,
\end{align}
where the orthogonality of $\mathbf{Q}_{\ell}$ has been used. Combining this
uniform bound with \cref{eq:block_q_dependent_bound} gives
\begin{equation}
\varepsilon_{\ell}
\leq
2L_{\psi}
\left\|
\bar{\mathcal{Y}}_{\ell}
\right\|_F
\min
\left\{
1,
\left\|
\mathbf{Q}_{\ell}
-
\mathbf{I}_{r_{\ell}}
\right\|_2
\right\}.
\label{eq:block_combined_bound}
\end{equation}

Finally, substituting
\cref{eq:simple_eigenspace_zero_error,eq:block_combined_bound}
into \cref{eq:eigenspace_error_decomposition} yields
\begin{align}
\mathcal{E}_{\rm equiv}(\mathcal{Y},\mathbf{P})
&\leq
2L_{\psi}
\left[
\sum_{\ell\in\mathcal{R}}
\left\|
\bar{\mathcal{Y}}_{\ell}
\right\|_F^2
\min
\left\{
1,
\left\|
\mathbf{Q}_{\ell}
-
\mathbf{I}_{r_{\ell}}
\right\|_2^2
\right\}
\right]^{1/2}
\\
&\leq
2L_{\psi}
\left[
\sum_{\ell\in\mathcal{R}}
\left\|
\bar{\mathcal{Y}}_{\ell}
\right\|_F^2
\right]^{1/2}
\\
&=
2L_{\psi}
\sqrt{
E_{\mathcal{R}}(\mathcal{Y})
}.
\end{align}
This completes the proof.
\end{proof}

\section{Detailed Description of Experimental Setup}\label{exp_data}
To empirically validate the efficacy of our GFT-based low-rank tensor framework, we present a comparative analysis with state-of-the-art (SOTA) clustering methods across various benchmark datasets. All algorithms are implemented in MATLAB R2020b and executed on a workstation equipped with an Intel Core i7-11700 CPU (2.50GHz) and 64GB of RAM. The code is available at \\ \href{https://github.com/jijintian/GFT-TMVC}{https://github.com/jijintian/GFT-TMVC}.

\textbf{Datasets:}
To verify the effectiveness of our proposed methods, eight challenging datasets with diverse properties are used and summarized in Table~\ref{table_appendix_dataset}. More detailed information about these datasets is provided below. To construct each shuffled dataset, we generate a random permutation of the sample indices and apply the same permutation consistently to all views and the corresponding ground-truth labels. In this way, the cross-view correspondence of each sample and its label information are strictly preserved, while only the global sample ordering is changed.
 \begin{itemize} 
\item[$\bullet$]{\bf MSRCv1}\footnote{\href{http://research.microsoft.com/en~us/projects/objectclassrecognition}{http://research.microsoft.com/en~us/projects/objectclassrecognition}}. It consists of 210 scene recognition images and is divided into 7 categories. Each image is represented by five different feature sets.

  \begin{table}[!htb]
\centering
		\caption{
			Summary of the benchmark datasets.
		}
		\label{table_appendix_dataset}
          \resizebox{\linewidth}{!}
  {
		\begin{tabular}{ccc}
            \toprule
			{\bf Dataset} & {\bf Sam./Clu.} & {\bf Dimensions}\\
            \midrule
			\noalign{\smallskip}
            {\bf MSRCv1}  & 210 / 7  & 24/576/512/256/254\\
            {\bf NGs}  & 500 / 5  & 2000/2000/2000\\
		      {\bf BBCSport}  & 544 / 5 &3183/3203 \\
            {\bf Caltech101-20}  & 2386 /20 & 48/40/254/1984/512/928\\
            {\bf CCV}  & 6773 / 20  & 20/20/20\\
            {\bf Caltech101-all}  & 9144 / 102  & 48/40/254/1984/512/928\\
            {\bf Aloi-100}  & 11025 / 100  & 77/13/64/125\\
            {\bf Animal}  & 11673 / 20  & 2689/2000/2001/2000\\
			\bottomrule
		\end{tabular}}
\end{table}
 \item[$\bullet$]{\bf NGs}\footnote{\href{https://lig-membres.imag.fr/grimal/data.html}{https://lig-membres.imag.fr/grimal/data.html}}. It is a subset of the 20 Newsgroups dataset, which consists of 500 newsgroup documents. Each raw document is pre-processed with three different methods (three views) and is assigned to one of five categories.
   \item[$\bullet$]{\bf BBCSport}\footnote{\href{http://mlg.ucd.ie/datasets/segment.html}{http://mlg.ucd.ie/datasets/segment.html}}. It consists of 544 sports news articles in five topical areas (athletics, cricket, football, rugby, and tennis). Each sample is described by two views, and their dimensions are 3183 and 3203, respectively.

  \item[$\bullet$]{\bf Caltech101-20}. It is a subset of the challenging object database Caltechall \cite{fei2004learning} and contains 2386 samples and 20 classes. Each sample has six types of features, i.e., 48-dim Gabor, 40-dim wavelet moments, 254-dim CENTRIST, 1984-dim HOG, 512-dim GIST, and 928-dim LBP.
  
\item[$\bullet$]{\bf CCV}\footnote{\href{http://www.ee.columbia.edu/ln/dvmm/CCV/}{http://www.ee.columbia.edu/ln/dvmm/CCV/}}. It consists of 6773 samples belonging to 20 semantic categories of YouTube videos. Each sample has three views, and the dimensions are all 20.

\item[$\bullet$]{\bf Caltech101-all}. It is a subset of the challenging object database Caltechall \cite{fei2004learning} and contains 9144 images belonging to 102 categories. Each sample has six types of features, i.e., 48-dim Gabor, 40-dim wavelet moments, 254-dim CENTRIST, 1984-dim HOG, 512-dim GIST, and 928-dim LBP.

\item[$\bullet$]{\bf Aloi-100}\footnote{\href{http://elki.dbs.ifi.lmu.de/wiki/DataSets/MultiView}{http://elki.dbs.ifi.lmu.de/wiki/DataSets/MultiView}}. It is a widely used small-object dataset, which consists of 11025 images of 100 small objects. Each sample is represented by four types of features: RGB, HSV, Color similarity, and Haralick features.
 
 \item[$\bullet$]{\bf Animal}\footnote{\href{https://github.com/wangsiwei2010/large_scale_multi-view_clustering_datasets}{https://github.com/wangsiwei2010/}}. It consists of 11673 animal images from 20 categories, with each sample represented by four views.
 
\end{itemize}

\textbf{Baselines:}
We select five representative state-of-the-art FFT-based low-rank tensor learning methods:
\begin{itemize} 
\item[$\bullet$]\textbf{t-SVD-MVC (IJCV 2018)} \cite{xie2018unifying} is the first to apply FFT-based low-rank tensor learning to multi-view clustering. Applying FFT along the sample mode significantly improved clustering performance.
\item[$\bullet$]\textbf{TLSpNM (TPAMI 2023)} \cite{guo2022logarithmic} proposes a novel low-rank constraint for FFT-based low-rank tensor learning, demonstrating that nonconvex rank approximation functions outperform traditional Tensor Nuclear Norm (TNN).
\item[$\bullet$]\textbf{ASR-ETR (ICCV 2023)} \cite{ji2023anchor} proposes a fast version of the t-SVD-based low-rank tensor learning framework, which greatly improves the efficiency of traditional FFT-based frameworks when handling large-scale data.
\item[$\bullet$]\textbf{ESTMC (TPAMI 2025)} \cite{ji2025anchors} integrates various non-convex rank approximations into a fast t-SVD-based low-rank tensor learning framework.
\item[$\bullet$]\textbf{TLRLF4MVC (TPAMI 2025)} \cite{long2025tlrlf4mvc} introduces a novel tensor low-frequency component to smooth the representation across samples, which improves the performance of the t-SVD-based low-rank tensor learning framework.
\end{itemize}
 To ensure a rigorous and fair comparison, all experiments are conducted using the official implementations provided by the respective authors. We strictly follow the parameter tuning protocols suggested in the original publications to report the best achievable results for each baseline.

\textbf{Implementation Details:} To ensure a rigorous and fair comparison, we seamlessly integrate the proposed GFT-based framework into the baseline architectures by substituting the standard sample-mode FFT components with our GFT-based spectral transforms.
 For the GFT-based variants, we evaluate three distinct graph construction methodologies: (1) our proposed SimG strategy (10 neighbors); (2) the GMC algorithm \cite{wang2019gmc}; and (3) an ideal graph (using ground-truth labels) to establish the performance upper bound. Additionally, we implement the proposed AGFT (\cref{def:A-GFT}) with Balanced $k$-means-based Hierarchical $k$-means (BKHK) \cite{Nie2024Fast} to select anchors, and the graph construction method follows this work \cite{long2025tlrlf4mvc}. The number of anchors is set to $k_a=2^{\min(c,10)}$ on all test datasets, where $c$ denotes the number of clusters. Specifically, $c$ determines the depth of the recursive binary partition when $c<10$, while the maximum partition depth is capped at 10 to control the computational cost. Consequently, the number of anchors increases with the complexity of the underlying clustering structure but is bounded above by $2^{10}=1024$. All datasets considered in our experiments satisfy $n\ge k_a$.
 To specifically benchmark against existing methods that attempt to mitigate the permutation sensitivity of standard FFT, we include a variant denoted as FFT(view mode). Following the CVPR 2025 method in \cite{liu2025large}, this variant performs the FFT along the view mode instead of the sample mode to circumvent sample-ordering issues.
 All experiments are conducted over five independent trials. We report the mean performance along with the standard deviation across all metrics. Hyperparameters for each method are optimized via grid search according to their respective original publications to ensure peak performance.

\textbf{Parameter Setting of $k_n$:}
For SimG, we set the number of nearest neighbors to $k_n=10$ by default, following common practice in graph-based methods. As shown in Table~\ref{tab:kn_sensitivity}, the performance remains stable when $k_n$ varies within $\{5,10,15,20\}$, indicating that the proposed framework is not sensitive to this parameter.

\begin{table}[H]
\centering
\setlength{\abovecaptionskip}{0.1cm} 
\caption{Clustering accuracy (ACC\%) of ESTMC with GFT under different values of $k_n$.}
\label{tab:kn_sensitivity}
\begin{tabular}{lcccc}
\toprule
\textbf{Dataset} & $k_n=5$ & $k_n=10$ & $k_n=15$ & $k_n=20$ \\
\midrule
MSRCv1   & 86.19 & 85.05 & 84.29 & 84.38 \\
NGs      & 96.40 & 96.60 & 96.80 & 96.00 \\
BBCSport & 95.04 & 94.85 & 94.85 & 94.30 \\
\bottomrule
\end{tabular}
\end{table}

\textbf{Parameter Setting of $k_a$:}
Since BKHK \cite{Nie2024Fast} recursively performs binary partitioning, the number of anchors is restricted to powers of two. We therefore set $k_a=2^{\min(c,10)}$, where $c$ is the number of clusters, yielding $k_a=128$ for MSRCv1 and $k_a=32$ for BBCSport and NGs. For datasets with more than $10$ clusters, $k_a$ is capped at $2^{10}=1024$. As shown in Table~\ref{tab_ka}, the default setting achieves competitive or superior performance across the three datasets.

\begin{table}[H]
\centering
\setlength{\abovecaptionskip}{0.1cm} 
\caption{Clustering accuracy (ACC\%) of ESTMC with AGFT under different values of $k_a$.}
\label{tab_ka}
\begin{tabular}{lccccc}
\toprule
\textbf{Dataset} & $k_a=8$ & $k_a=16$ & $k_a=32$ & $k_a=64$ & $k_a=128$ \\
\midrule
MSRCv1   & 74.95 & 75.33 & 83.81 & 81.43 & 87.81 \\
BBCSport & 71.51 & 88.42 & 88.05 & 87.87 & 86.40 \\
NGs      & 69.40 & 70.60 & 87.80 & 86.80 & 86.80 \\
\bottomrule
\end{tabular}
\end{table}

\begin{table*}[t]
\centering
\caption{Clustering performance (mean(std)) of five baseline architectures across different low-rank tensor learning frameworks on the BBCSport and Caltech101-20 datasets. +FFT denotes the original framework; +FFT (view) refers to FFT performed along the view dimension; +GFT(SimG), +GFT(GMC), and +GFT(Ideal) represent our proposed GFT framework using SimG, GMC, and ground-truth graphs, respectively; +AGFT denotes the proposed AGFT framework.}
\label{tab:appendix_results}
\small
  \resizebox{\textwidth}{!}
  {
\begin{tabular}{l l ccccc ccccc}
\toprule
\multirow{2}{*}{\textbf{Baselines}} & \multirow{2}{*}{\textbf{Variants}} & \multicolumn{5}{c}{\textbf{Shuffled BBCSport}} & \multicolumn{5}{c}{\textbf{Shuffled Caltech101-20}} \\
\cmidrule(lr){3-7} \cmidrule(lr){8-12}
& & ACC& NMI & Purity& F-score &ARI & ACC& NMI & Purity&  F-score &ARI \\
\midrule

  & +FFT   & 55.29(5.46) & 35.95(5.42) & 69.30(4.33) & 46.25(4.56) & 31.24(5.82) & \underline{57.28(1.39)} & 66.45(0.78) & 79.72(1.12) & \underline{51.73(0.91)} & \underline{46.99(1.01)} \\
          & +FFT(view mode) & 86.21(0) & 75.98(0) & 86.21(0) & 83.32(0) & 77.97(0) & 56.25(0.89) & \underline{67.11(0.49)} & \underline{80.18(0.74)} & 50.59(0.85) & 45.09(0.89) \\
          \rowcolor{gray!15} \cellcolor{white} & +GFT(SimG) & \textbf{93.20(0)} & \textbf{81.05(0.19)} & \textbf{93.20(0)} & \textbf{86.67(0.18)} & \textbf{82.51(0.10)} & 53.33(0.49) & 65.10(0.62) & 80.17(0.79) & 44.27(0.63) & 38.50(0.65) \\
          \rowcolor{gray!15} \cellcolor{white} & +GFT(GMC) & 88.60(0.68) & \underline{78.33(1.55)} & 88.60(0.69) & \underline{84.60(1.32)} & \underline{79.82(1.61)} & \textbf{59.38(1.37)} & \textbf{68.86(0.35)} & \textbf{81.60(0.86)} & \textbf{53.60(0.86)} & \textbf{48.08(0.89)} \\
          \rowcolor{gray!15} \cellcolor{white} & +AGFT & \underline{90.29(0.08)} & 73.90(0.43) & \underline{90.29(0.08)} & 82.20(0.01) & 76.73(0.24) & 56.00(0.32) & 66.33(0.20) & 79.59(0.17) & 48.27(0.55) & 43.14(0.61) \\
          \rowcolor{gray!30} \multirow{-6}{*}{\cellcolor{white} t-SVD-MVC} & +GFT(Ideal) & 100(0) & 100(0) & 100(0) & 100(0) & 100(0) & 78.90(1.68) & 88.98(1.31) & 95.41(0.78) & 69.50(1.84) & 65.51(1.98) \\
\midrule
\addlinespace[0.5em] 

  & +FFT   & 71.86(3.21) & 52.10(1.68) & 77.00(1.11) & 60.99(1.68) & 49.95(2.11) & 44.17(1.00) & 51.13(0.90) & 74.95(0.81) & 35.80(1.22) & 29.92(1.26) \\
          & +FFT(view mode) & 76.10(0) & 61.94(0.05) & 81.25(0.03) & 67.23(0.04) & 58.01(0.05) & 37.31(1.18) & 48.76(0.54) & 72.51(0.70) & 32.82(1.06) & 27.01(1.10) \\
          \rowcolor{gray!15} \cellcolor{white} & +GFT(SimG) & \underline{83.31(0.08)} & \underline{71.06(0.29)} & \underline{84.96(0.08)} & \underline{74.36(0.14)} & \underline{67.02(0.18)} & 47.39(2.60) & 55.47(1.08) & 77.66(0.48) & \underline{41.01(2.11)} & \underline{35.07(2.11)} \\
          \rowcolor{gray!15} \cellcolor{white} & +GFT(GMC) & \textbf{89.34(2.26)} & \textbf{76.33(1.40)} & \textbf{89.34(2.26)} & \textbf{84.22(1.74)} & \textbf{79.33(2.17)} & \textbf{47.60(1.07)} & \underline{56.58(0.91)} & \underline{78.07(0.99)} & 39.81(0.72) & 33.96(0.71) \\
          \rowcolor{gray!15} \cellcolor{white} & +AGFT & 78.97(0.10) & 61.23(0.08) & 80.15(0.02) & 66.42(0.04) & 56.72(0.04) & \underline{47.50(1.52)} & \textbf{56.89(0.77)} & \textbf{78.09(0.81)} & \textbf{42.02(1.12)} & \textbf{36.35(1.17)} \\
          \rowcolor{gray!30} \multirow{-6}{*}{\cellcolor{white} TLSpNM} & +GFT(Ideal) & 99.60(0.08) & 98.55(0.23) & 99.60(0.08) & 99.21(0.16) & 98.96(0.21) & 65.52(4.36) & 77.05(1.77) & 92.36(0.99) & 57.43(3.55) & 52.76(3.46) \\
\midrule
\addlinespace[0.5em] 

  & +FFT   & 83.48(0.64) & 67.29(1.41) & 83.46(0.64) & 73.58(3.16) & 65.56(3.83) & 51.89(0.43) & 56.34(0.58) & 76.34(1.00) & 46.69(1.03) & 40.76(1.11) \\
          & +FFT(view mode) & 85.11(0) & 68.72(0) & 85.11(0) & 75.52(0) & 68.25(0) & 51.85(0.55) & 56.74(0.80) & \underline{77.41(0.55)} & 45.05(1.45) & 38.94(1.51) \\
          \rowcolor{gray!15} \cellcolor{white} & +GFT(SimG) & \underline{94.63(0.49)} & \underline{84.15(1.40)} & \underline{94.63(0.49)} & \underline{89.19(0.99)} & \underline{85.73(1.33)} & 54.12(2.18) & 57.73(0.81) & 75.14(0.91) & 48.05(2.26) & 41.26(2.70) \\
          \rowcolor{gray!15} \cellcolor{white} & +GFT(GMC) & \textbf{94.85(0)} & \textbf{85.09(0.01)} & \textbf{94.85(0.04)} & \textbf{89.72(0.08)} & \textbf{86.50(0.05)} & \textbf{59.72(2.36)} & \textbf{63.28(1.63)} & 77.21(1.99) & \textbf{54.78(2.39)} & \textbf{48.22(2.63)} \\
          \rowcolor{gray!15} \cellcolor{white} & +AGFT & 88.53(0.10) & 70.64(0.35) & 88.35(0.10) & 77.64(0.16) & 70.25(0.18) & \underline{56.89(0.91)} & \underline{60.20(1.13)} & \textbf{77.66(0.98)} & \underline{51.18(3.10)} & \underline{45.14(3.47)} \\
          \rowcolor{gray!30} \multirow{-6}{*}{\cellcolor{white} ASR-ETR} & +GFT(Ideal) & 100(0) & 100(0) & 100(0) & 100(0) & 100(0) & 97.52(1.50) & 98.37(0.65) & 98.12(1.27) & 99.22(0.54) & 99.07(0.64) \\
\midrule
\addlinespace[0.5em] 
   & +FFT   & 83.42(0.85) & 64.72(1.62) & 83.42(0.85) & 71.75(1.38) & 63.59(1.75) & 41.93(1.17) & 51.72(0.74) & 75.02(0.73) & 35.85(1.39) & 30.11(1.43) \\
          & +FFT(view mode) & 83.82(0) & 64.08(0) & 83.82(0) & 71.72(0) & 63.56(0) & 42.08(1.94) & 51.55(1.58) & \underline{75.24(1.45)} & 35.95(1.42) & 30.24(1.49) \\
          \rowcolor{gray!15} \cellcolor{white} & +GFT(SimG) & \textbf{94.85(0)} & \textbf{85.66(0)} & \textbf{94.85(0)} & \textbf{89.32(0)} & \textbf{86.01(0)} & 54.32(1.08) & 54.95(0.52) & 74.73(1.17) & 49.85(1.09) & 43.64(1.02) \\
          \rowcolor{gray!15} \cellcolor{white} & +GFT(GMC) & \underline{90.96(3.18)} & \underline{80.34(3.21)} & \underline{90.96(3.18)} & \underline{87.13(1.73)} & \underline{83.04(2.32)} & \textbf{58.30(2.59)} & \textbf{59.18(2.12)} & 75.05(2.96) & \underline{52.24(3.22)} & \underline{45.76(3.22)} \\
          \rowcolor{gray!15} \cellcolor{white} & +AGFT & 88.05(0) & 71.03(0) & 88.05(0) & 80.13(0) & 74.16(0) & \underline{55.55(1.30)} & \underline{58.33(0.55)} & \textbf{76.96(1.11)} & \textbf{52.83(1.29)} & \textbf{47.19(1.42)} \\
          \rowcolor{gray!30} \multirow{-6}{*}{\cellcolor{white} ESTMC} & +GFT(Ideal) & 100(0) & 100(0) & 100(0) & 100(0) & 100(0) & 96.68(4.41) & 98.57(1.81) & 98.04(2.35) & 99.27(1.11) & 99.13(1.32) \\
\midrule
\addlinespace[0.5em] 
  & +FFT   & 75.55(0) & 51.53(0) & 75.55(0) & 63.70(0) & 52.87(0) & 57.21(1.35) & 60.54(2.46) & 74.48(2.80) & 54.05(2.30) & 47.74(2.42) \\
          & +FFT(view mode) & 73.35(1.23) & 52.35(0) & 75.51(0.33) & 60.37(0.53) & 47.88(1.01) & 56.50(0.66) & 63.47(0.72) & \textbf{77.75(0.47)} & 55.23(0.82) & 49.54(0.92) \\
          \rowcolor{gray!15} \cellcolor{white} & +GFT(SimG) & \underline{81.88(2.71)} & \underline{75.52(0.48)} & \underline{84.04(0.49)} & \underline{76.27(0.93)} & \underline{67.74(1.24)} & 57.57(0.90) & 59.88(2.37) & 73.70(2.66) & 54.80(1.23) & 48.34(1.69) \\
          \rowcolor{gray!15} \cellcolor{white} & +GFT(GMC) & \textbf{84.74(0)} & \textbf{76.24(0)} & \textbf{84.74(0)} & \textbf{78.28(0)} & \textbf{70.22(0)} & \textbf{61.94(0)} & \textbf{65.80(0)} & 74.10(0) & \textbf{61.34(0)} & \textbf{54.72(0)} \\
          \rowcolor{gray!15} \cellcolor{white} & +AGFT & 80.33(0) & 65.54(0) & 80.88(0) & 72.10(0) & 61.98(0) & \underline{60.44(0)} & \underline{64.51(0)} & \underline{76.74(0)} & \underline{57.84(0)} & \underline{52.11(0)} \\
          \rowcolor{gray!30} \multirow{-6}{*}{\cellcolor{white} TLRLF4MVC} & +GFT(Ideal) & 88.79(0) & 94.54(0) & 88.79(0) & 94.91(0) & 93.20(0) & 88.60(0) & 95.24(0) & 90.61(0) & 97.34(0) & 96.82(0) \\
\bottomrule
\end{tabular}}
\end{table*}

\section{Additional Experimental Results}\label{exp_result}
To further validate the generalization of the proposed frameworks, we provide extended experimental evaluations on the BBCSport and Caltech101-20 datasets. The extended results in Table~\ref{tab:appendix_results} further demonstrate the broad effectiveness and robustness of the proposed GFT-based variants across different baseline architectures.

\textbf{Efficiency and Scalability of AGFT.} To evaluate the scalability of the proposed AGFT, we compare its main execution time + transform matrix construction time and clustering accuracy (ACC) against two fast TMC methods (ESTMC and TLRLF4MVC) across four large-scale datasets: CCV, Caltech101-all, Aloi-100, and Animal. As shown in Table~\ref{time_table} and \ref{time_table_appendix1}, AGFT significantly outperforms full-graph spectral methods such as GMC and GFT. This efficiency gain is theoretically grounded in the reduction in transform complexity from $\mathcal{O}(n^2)$ to $\mathcal{O}(nk_a)$, where $k_a$ represents the number of anchors ($k_a \ll n$). While the Fast Fourier Transform (FFT) remains efficient at $\mathcal{O}(n \log n)$ by exploiting structural symmetry, AGFT provides a highly scalable alternative for non-structured graph signals. Notably, AGFT may occasionally be slower than FFT. This is mainly due to two factors: (i) the additional overhead of transform matrix construction, and (ii) slower convergence of the iterative optimization after spectral truncation. For example, on the Animal dataset, FFT converges within 30 iterations, whereas AGFT requires 45 iterations. Despite these cases, AGFT remains efficient overall and effectively narrows the efficiency gap, making it a scalable solution for large-scale tasks.

\begin{table*}[h]
\centering
\caption{Runtime (main execution + transform matrix construction) vs. Accuracy of ESTMC with different transforms.}
\label{time_table_appendix1}
\small
  \resizebox{\textwidth}{!}
  {
\begin{tabular}{ccl|cl|cl|cl}
    \toprule
    \multirow{2}[4]{*}{Methods} & \multicolumn{2}{c}{Shuffled CCV} & \multicolumn{2}{c}{Shuffled Aloi-100} & \multicolumn{2}{c}{Shuffled Caltech101-all} & \multicolumn{2}{c}{Shuffled Animal} \\
\cmidrule(lr){2-3} \cmidrule(lr){4-5}  \cmidrule(lr){6-7}   \cmidrule(lr){8-9}    & ACC   & Time (s) & ACC   & Time (s) & ACC   & Time (s) & ACC   & Time (s) \\
    \midrule
+FFT   & 19.23 & 7.13 + 0  & 62.72 & 50.96 + 0 & 26.82 & 100.93 + 0 & 14.82 & \textbf{166.85 + 0} \\
    +SimG  & 19.75 & 14.61 + 10.77 & 65.13 & 149.89 + 25.83 & 24.14 & 183.03 + 17.92 & 18.33 & 218.05 + 37.23  \\
    +GMC   & 26.96 & 14.08 + 157.38 & 63.89 & 104.32 + 284.13  & 25.63 & 209.99 + 207.36 & 19.12 & 223.90 + 279.72\\
    +Ideal & \textbf{96.37} & 14.9 + 10.66  & \textbf{88.58} & 126.53 + 26.50 & \textbf{88.01} & 204.11 + 16.89 & 100   & 216.29 + 33.10 \\
    \rowcolor[rgb]{ .851,  .851,  .851} +AGFT  & 19.98 &\textbf{4.47 + 0.62} & 64.89 & \textbf{12.92 + 1.60} & 26.19 & \textbf{46.98  + 12.81} & 17.21 & 196.55 + 35.66 \\
    \bottomrule
    \end{tabular}}
\end{table*}

\section{Additional Ablation Study}\label{exp_ablation}

\textbf{Influence of Different Transform Bases.}
To demonstrate the generalizability of our observations, we provide additional visualizations of the transform bases for NGs, BBCSport, and Caltech101-20 in \cref{base_ablation1}. Across these diverse data types (ranging from text to images), the results consistently confirm that the GFT basis matrices exhibit clearer block-diagonal structures than the FFT basis matrix. This further validates that the manifold-adaptive nature of GFT is robust to varying data distributions, consistently ensuring class-discriminative energy concentration in the spectral domain.

\begin{figure*}[!h]
	\centering
	\subfloat[FFT]{\includegraphics[width=0.23\textwidth]{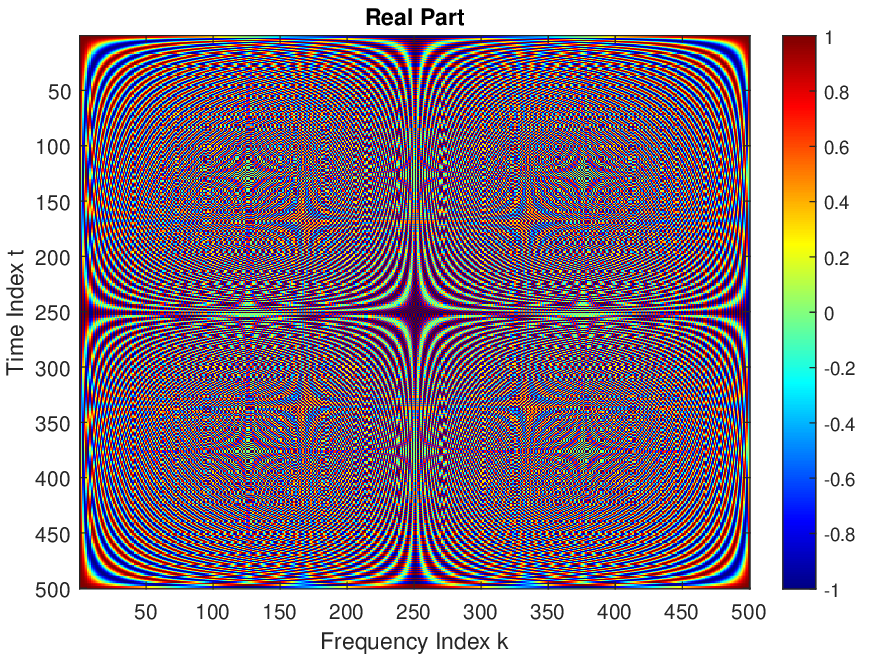}}\hspace{2pt}
	\subfloat[GFT(SimG)]{\includegraphics[width=0.23\textwidth]{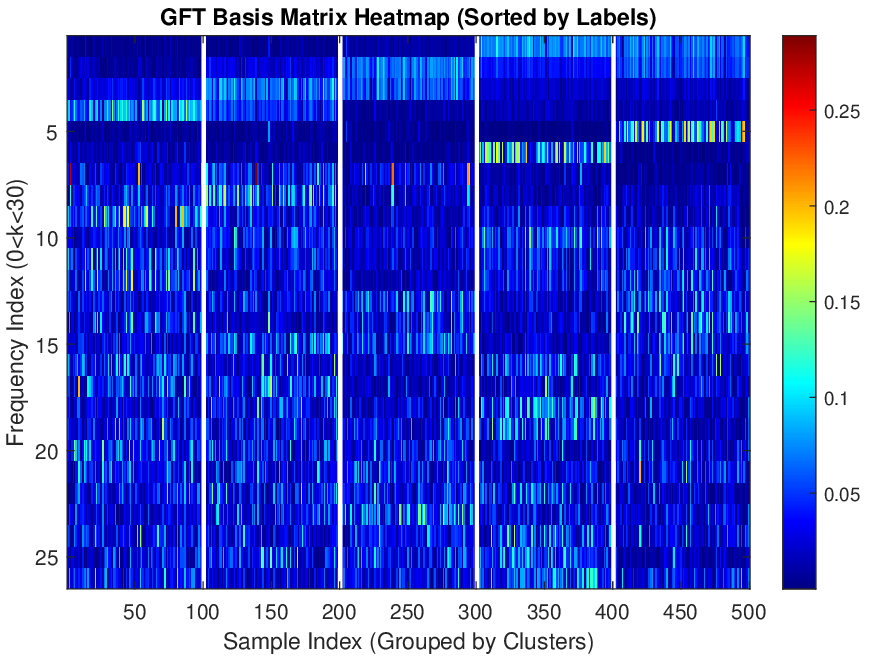}}\hspace{2pt}
    	\subfloat[GFT(GMC)]{\includegraphics[width=0.23\textwidth]{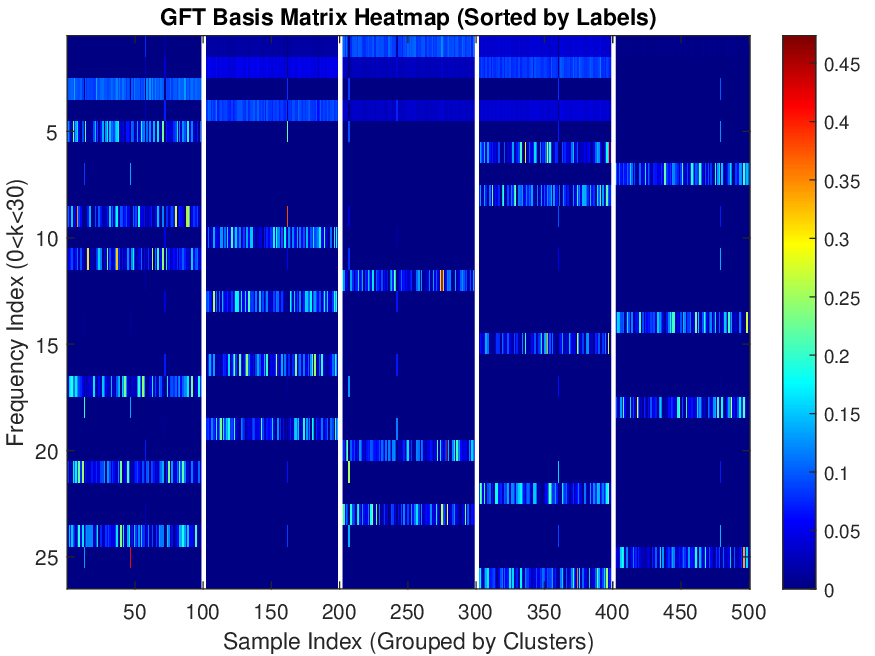}}\hspace{2pt}
        	\subfloat[GFT(Ideal)]{\includegraphics[width=0.23\textwidth]{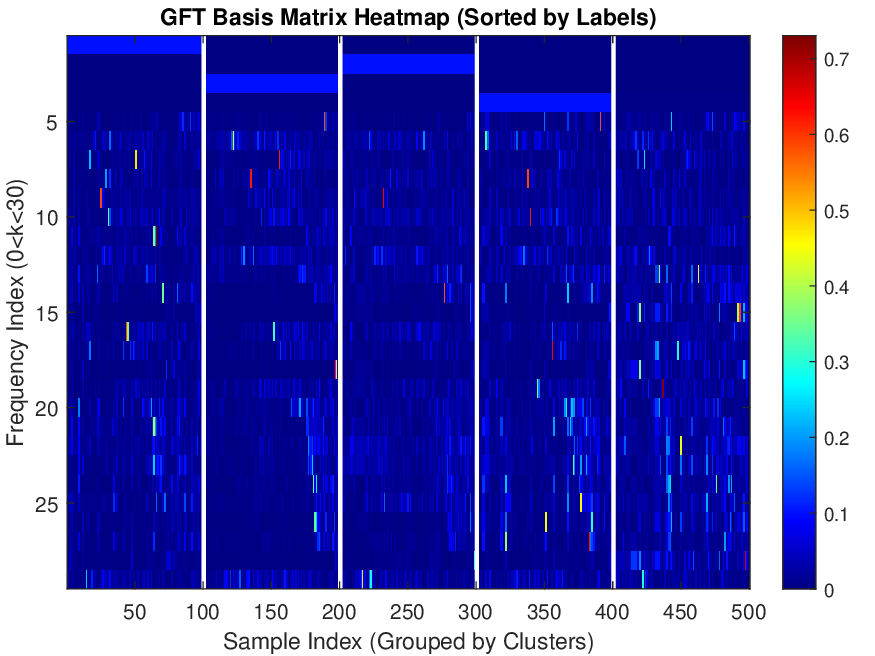}}\hspace{2pt}

	\subfloat[FFT]{\includegraphics[width=0.23\textwidth]{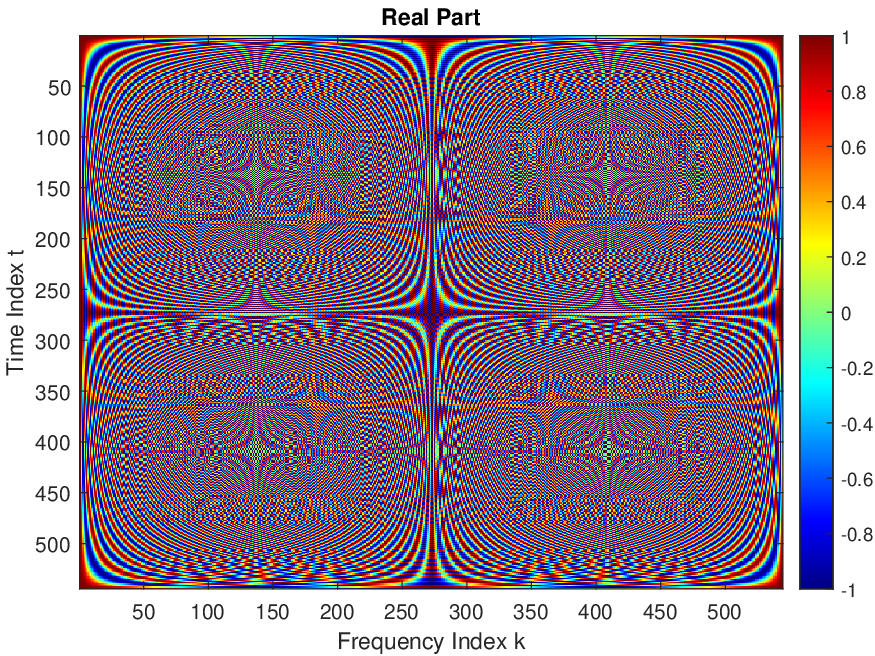}}\hspace{2pt}
	\subfloat[GFT(SimG)]{\includegraphics[width=0.23\textwidth]{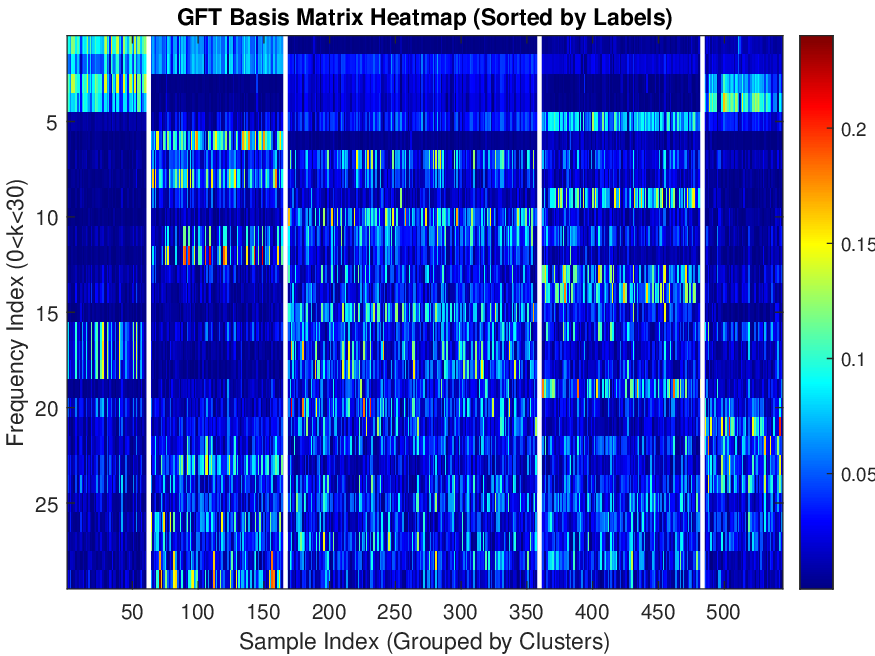}}\hspace{2pt}
    	\subfloat[GFT(GMC)]{\includegraphics[width=0.23\textwidth]{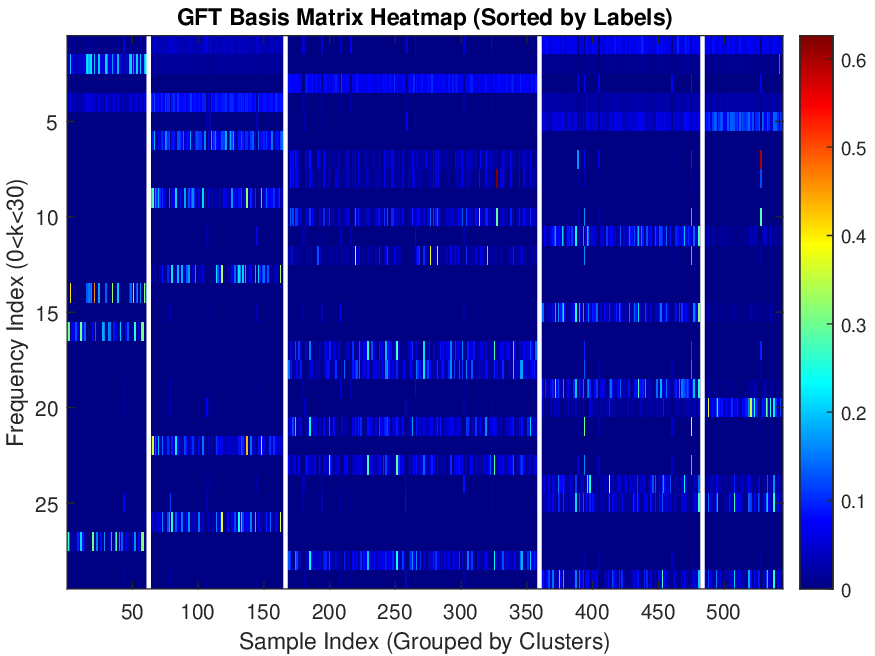}}\hspace{2pt}
        	\subfloat[GFT(Ideal)]{\includegraphics[width=0.23\textwidth]{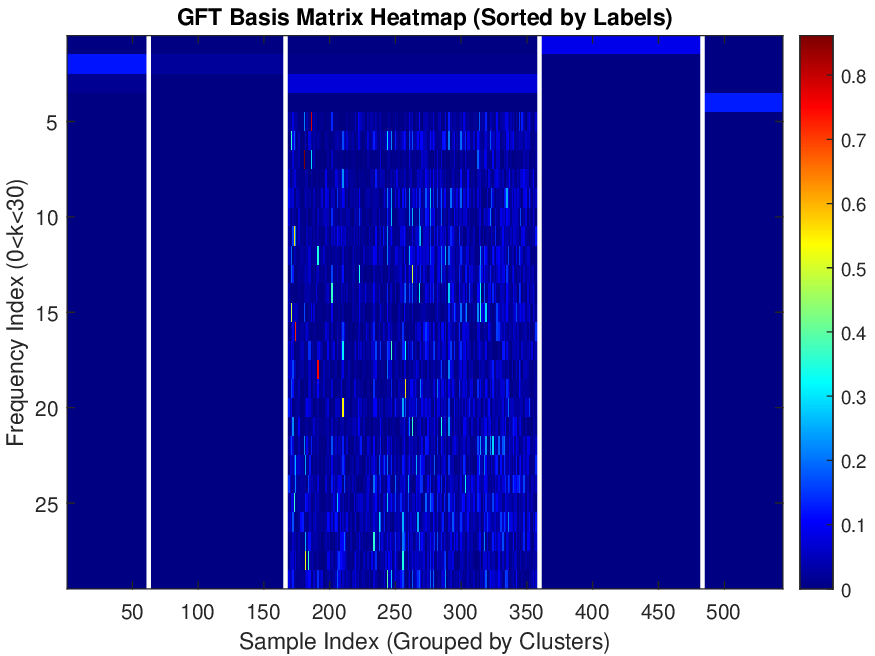}}\hspace{2pt}

	\subfloat[FFT]{\includegraphics[width=0.23\textwidth]{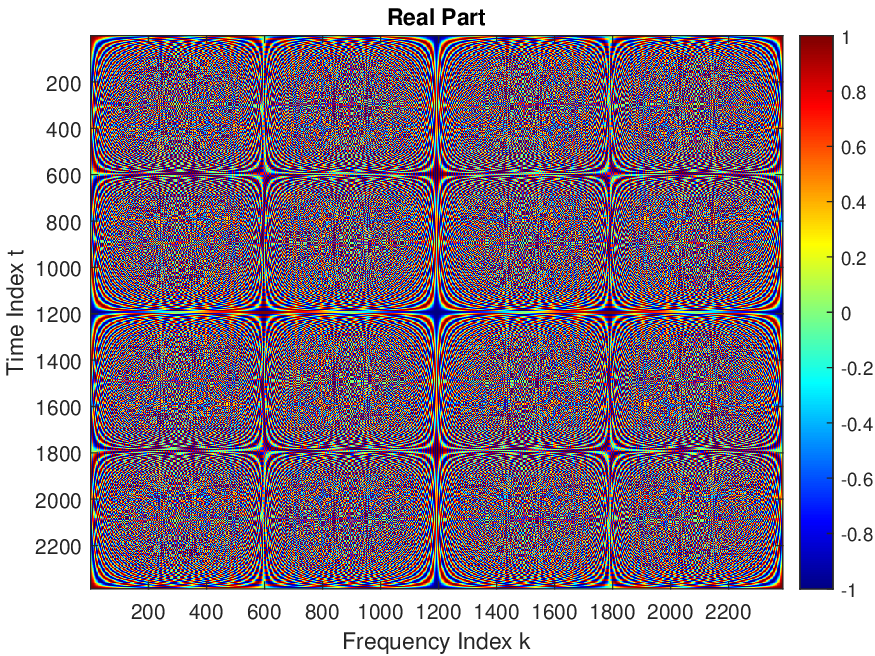}}\hspace{2pt}
	\subfloat[GFT(SimG)]{\includegraphics[width=0.23\textwidth]{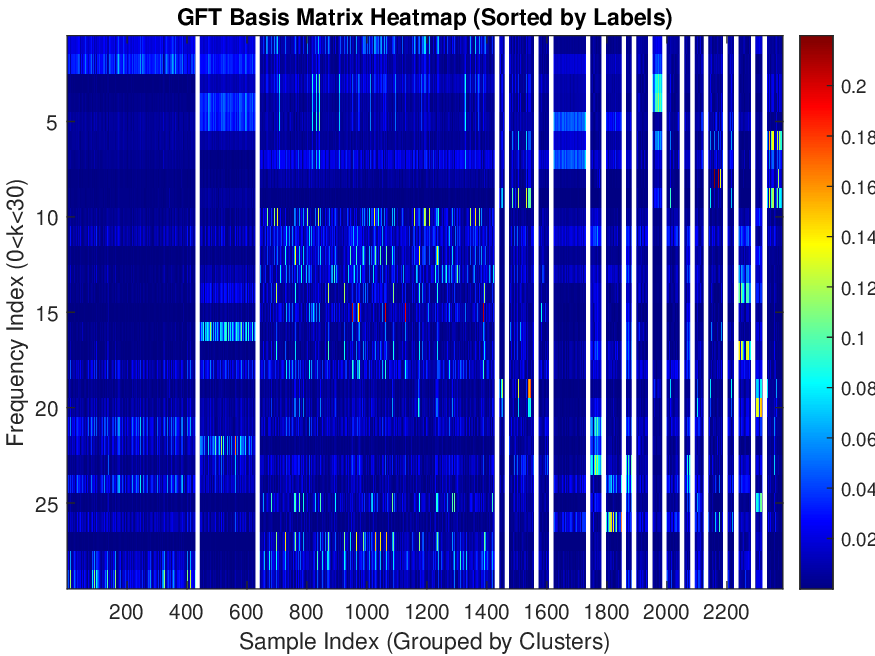}}\hspace{2pt}
    	\subfloat[GFT(GMC)]{\includegraphics[width=0.23\textwidth]{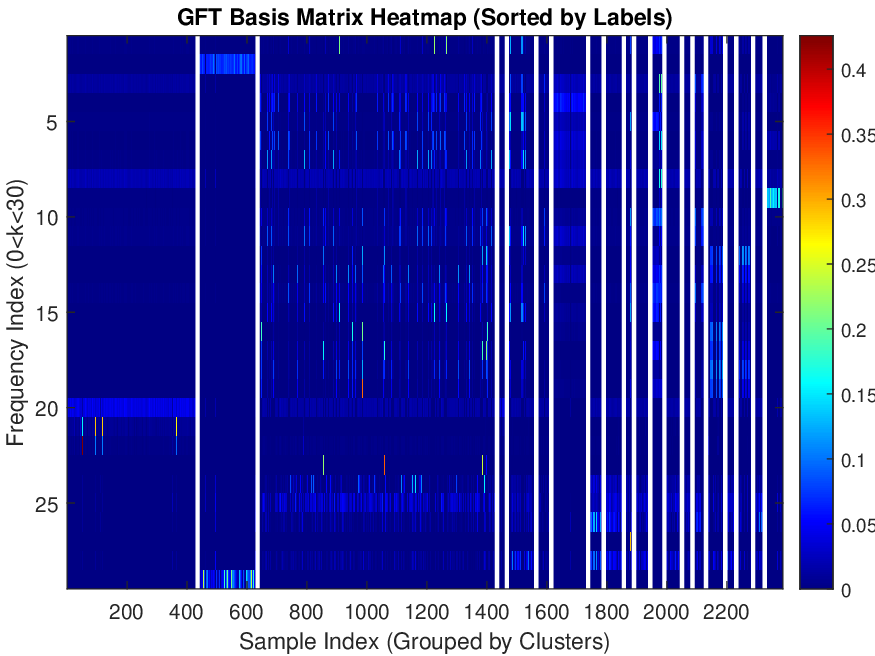}}\hspace{2pt}
        	\subfloat[GFT(Ideal)]{\includegraphics[width=0.23\textwidth]{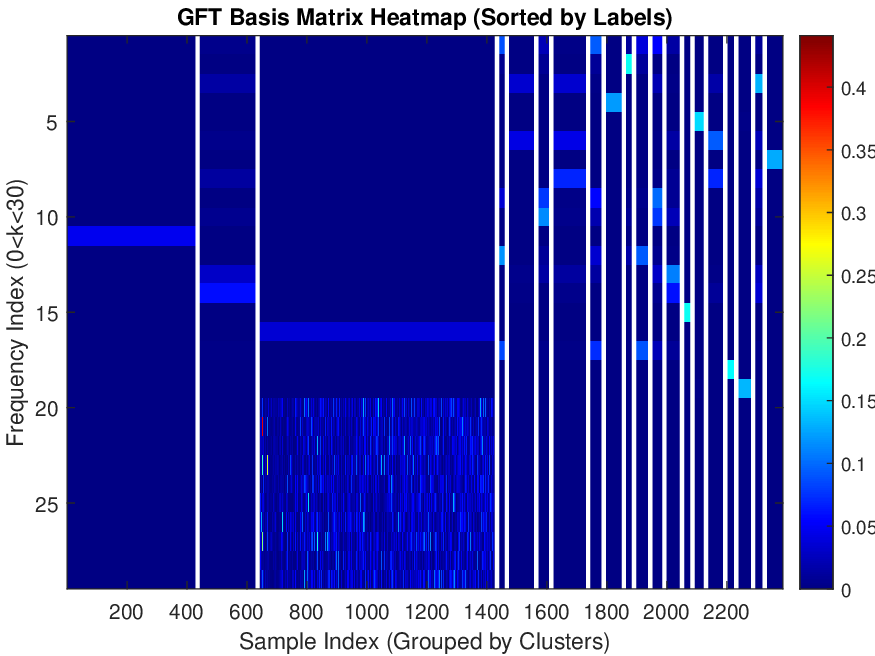}}\hspace{2pt}
 	\caption{Visualization of the transform basis matrices on the NGs (a--d), BBCSport (e--h) and Caltech101-20 (i--l) datasets.}
	\label{base_ablation1}
\end{figure*}

\begin{figure*}[!h]
	\centering
	\subfloat[MSRCv1]{\includegraphics[width=0.23\textwidth]{figure/energy_con_MSRCv1.eps}}\hspace{1pt}
	\subfloat[NGs]{\includegraphics[width=0.23\textwidth]{figure/energy_con_NGs.eps}}\hspace{1pt}
    \subfloat[BBCSport]{\includegraphics[width=0.23\textwidth]{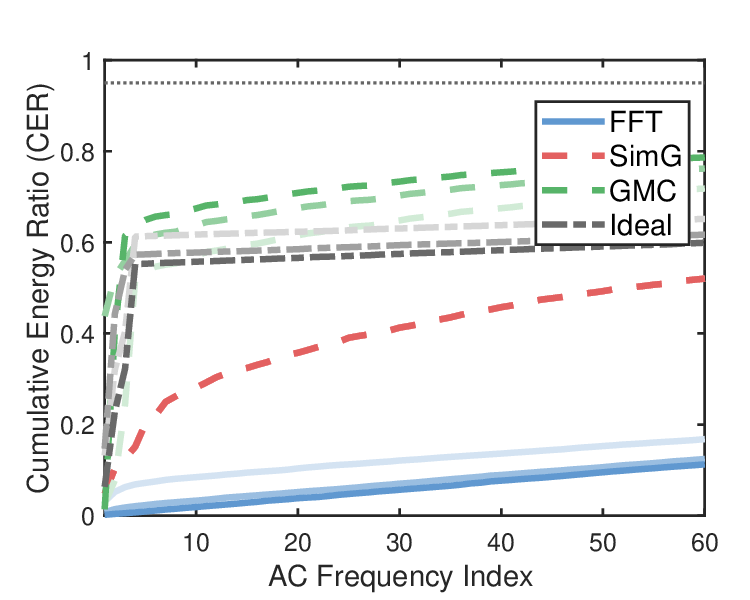}}\hspace{1pt}
	\subfloat[Caltech101-20]{\includegraphics[width=0.23\textwidth]{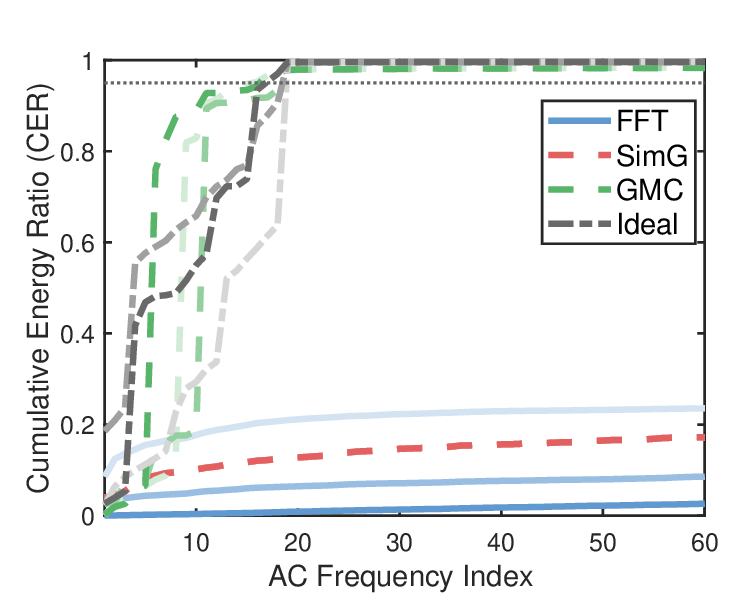}}\hspace{1pt}
 	\caption{Comparison of Cumulative Energy Ratio (CER) across different shuffle rates ($\{0, 0.5, 1\}$, where darker colors indicate higher shuffle rates) on the four datasets.}
	\label{energy_ablation_APPENDIX}
\end{figure*}

\textbf{Spectral Energy Compaction Analysis.}
To evaluate the energy-compaction capability of different spectral transforms, we organize the multi-view representations into a third-order tensor $\mathcal{Z}=\Phi(\mathbf{Z}^{(1)},\ldots,\mathbf{Z}^{(m)})\in\mathbb{R}^{d\times m\times n}$, where $n$ and $m$ denote the numbers of samples and views, respectively. Applying FFT or GFT along the third, i.e., the sample mode, yields $\bar{\mathcal{Z}}=\mathcal{T}(\mathcal{Z})$. The energy of the $i$-th spectral component is defined as:
\begin{equation}
    E_i = \| \bar{\mathcal{Z}}(:,:, i) \|_F^2, \quad i \in \{ 1, \dots, n\},
\end{equation}
where $\| \cdot \|_F$ denotes the Frobenius norm. To focus on the structural encoding capability, we analyze the AC frequencies (excluding the DC component at $i=1$). The Cumulative Energy Ratio (CER) for the first $K$ non-DC components is defined as:
\begin{equation}
    \text{CER}(K) = \frac{\sum_{i=2}^{K+1} E_i}{\sum_{i=2}^{n} E_i}, \quad K \in \{1, 2, \dots, n-1\}.
\end{equation}

A steeper CER curve indicates that the transform more effectively concentrates the structural information in a smaller number of spectral components. We provide the CER curves for all evaluated datasets in this section. The experiments are conducted under varying shuffle rates $\{0, 0.5, 1\}$, where 1 represents a fully randomized sample order. The results in \cref{energy_ablation_APPENDIX} lead to the following observations:

1. The CER curves of the GFT-based methods remain nearly unchanged across the tested shuffle rates. In contrast, the FFT baseline shows significant energy dispersion as the shuffle rate increases.

2. On most datasets, SimG and GMC exhibit a higher CER in the low-frequency AC range compared to FFT (especially in disordered cases). This confirms that graph-adaptive bases successfully align the intrinsic multi-view correlations into a compact spectral subspace.

3. In some specific cases where the shuffle rate $=0$ (perfectly ordered by category), FFT may show high compaction due to the artificial periodicity. However, our proposed GFT-based framework (SimG, GMC, and Ideal) provides a more robust and reliable representation for real-world scenarios where such an order is unknown.

\textbf{Influence of Shuffle Rates on GFT-based Methods.} To evaluate the robustness of our framework against data disorder, we conduct experiments with shuffle rates ranging from 0 to 1. As illustrated in \cref{fig:shuffle_rate_appendix}, while the performance of the FFT-based framework deteriorates sharply as the sample order becomes randomized, our GFT(SimG), GFT(GMC), and AGFT variants exhibit remarkable stability. Their clustering metrics remain nearly invariant as the shuffle rate increases, maintaining a near-constant plateau. This empirically validates that the graph-adaptive basis effectively captures the intrinsic topology, making the spectral representation substantially more stable under sample permutations.
\begin{figure*}[!h]
	\centering
	\subfloat[MSRCv1]{\includegraphics[width=0.30\textwidth]{figure/MSRCv1_tsvd_TV_gft.eps}}\hspace{4pt}
	\subfloat[NGs]{\includegraphics[width=0.30\textwidth]{figure/NGs_tsvd_TV_gft.eps}}\hspace{4pt}
    \subfloat[BBCSport]{\includegraphics[width=0.30\textwidth]{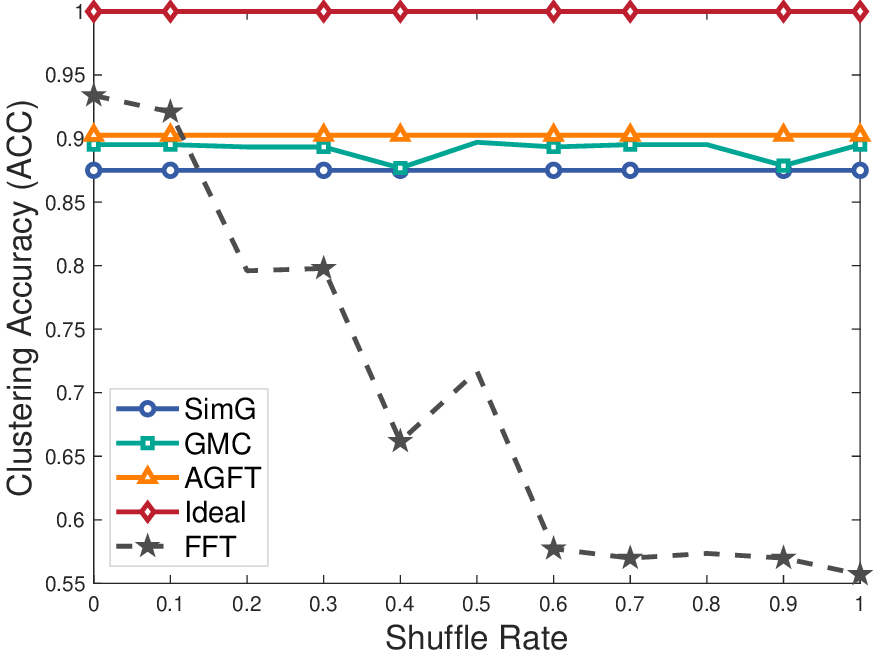}}\hspace{4pt}

    \subfloat[MSRCv1]{\includegraphics[width=0.30\textwidth]{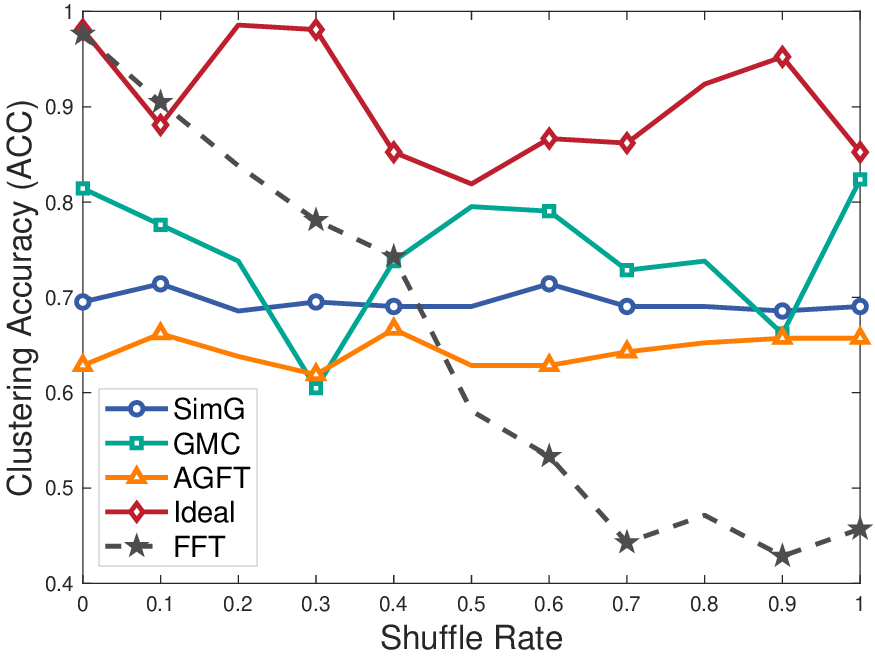}}\hspace{4pt}
	\subfloat[NGs]{\includegraphics[width=0.30\textwidth]{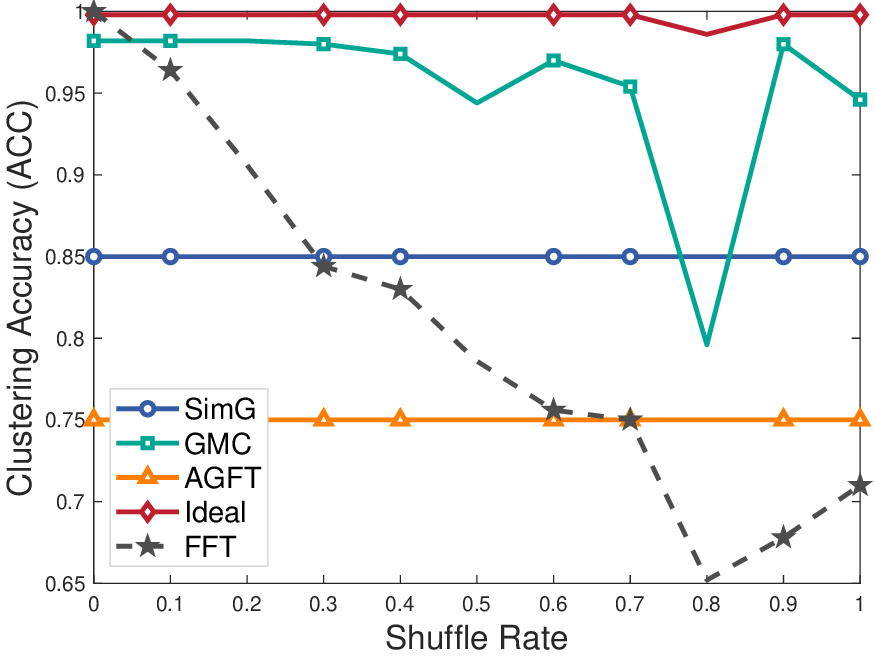}}\hspace{4pt}
    \subfloat[BBCSport]{\includegraphics[width=0.30\textwidth]{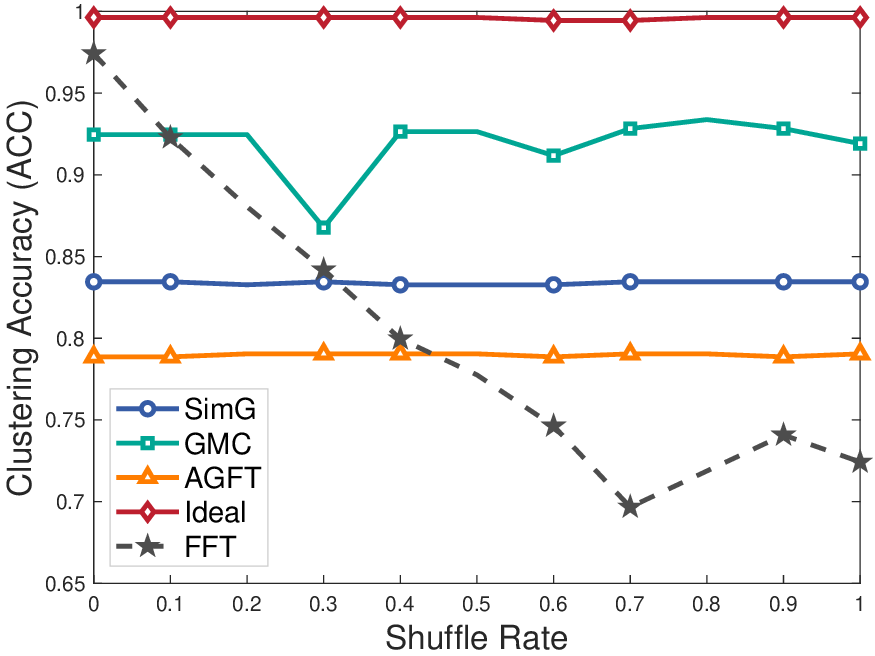}}\hspace{4pt}

    \subfloat[MSRCv1]{\includegraphics[width=0.30\textwidth]{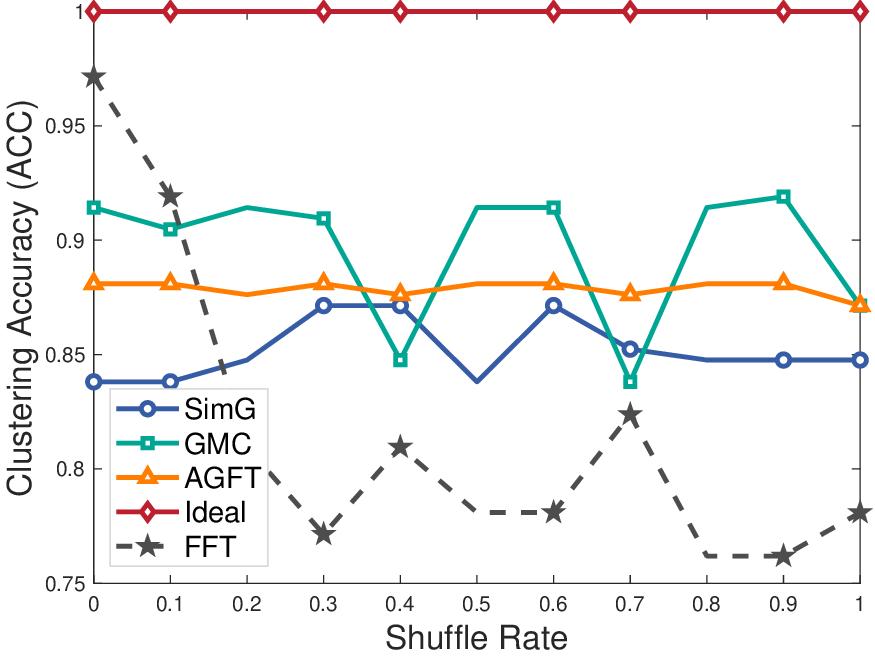}}\hspace{4pt}
	\subfloat[NGs]{\includegraphics[width=0.30\textwidth]{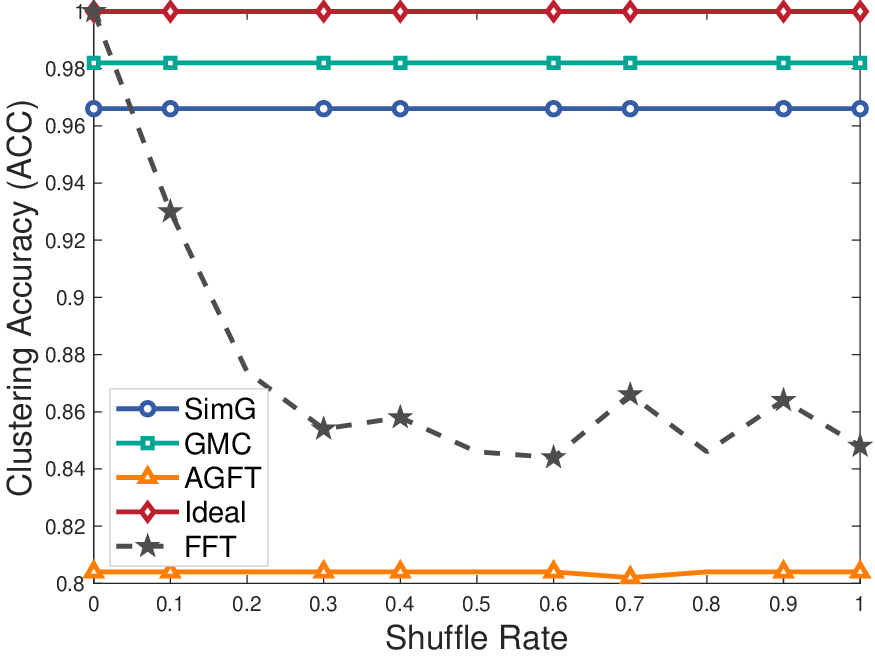}}\hspace{4pt}
    \subfloat[BBCSport]{\includegraphics[width=0.30\textwidth]{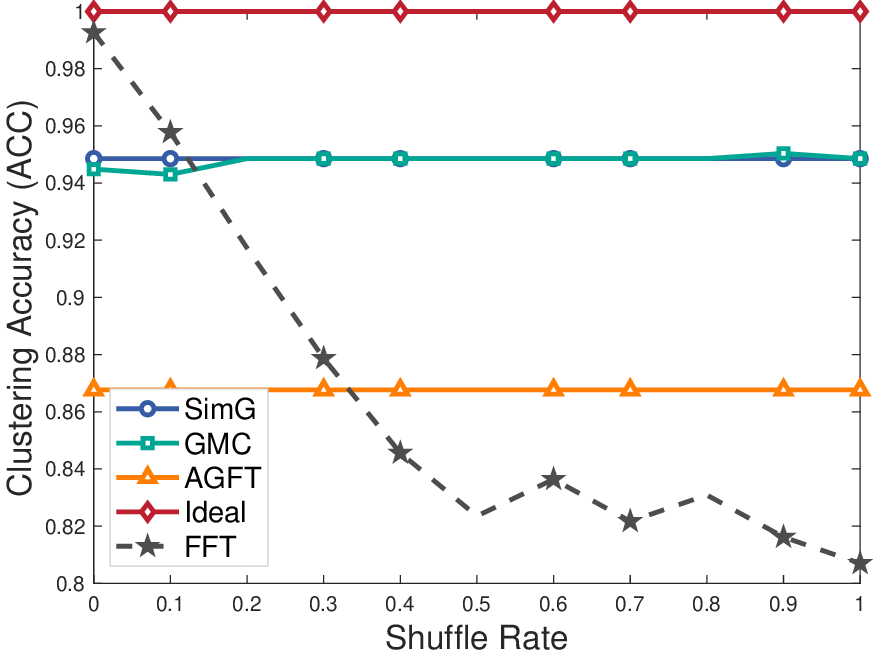}}\hspace{4pt}

    \subfloat[MSRCv1]{\includegraphics[width=0.30\textwidth]{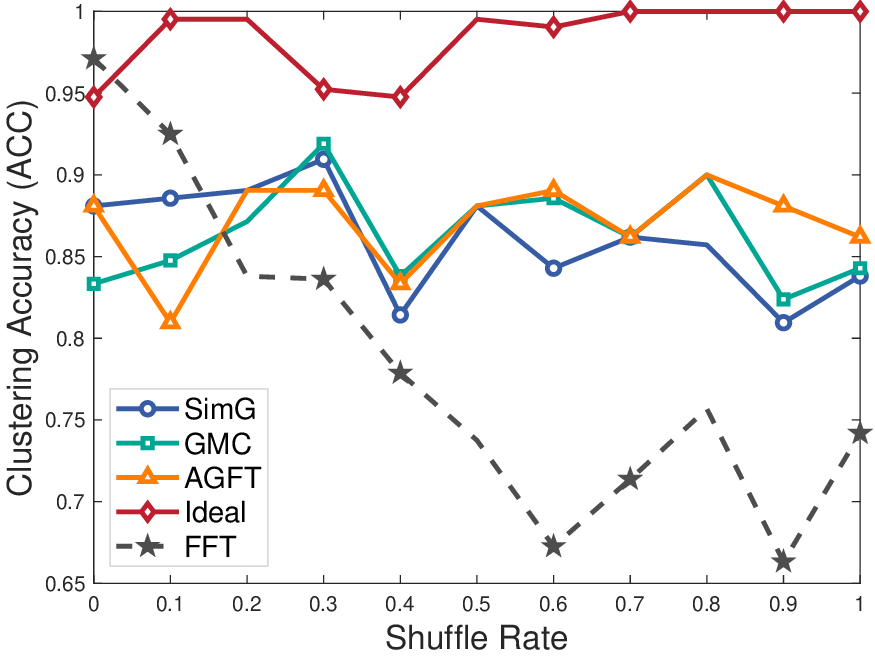}}\hspace{4pt}
	\subfloat[NGs]{\includegraphics[width=0.30\textwidth]{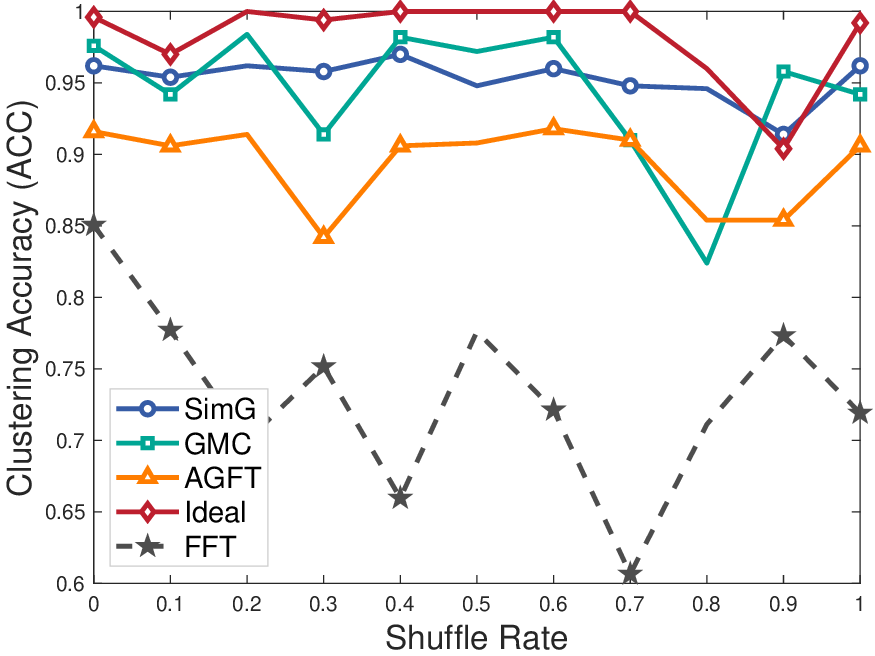}}\hspace{4pt}
    \subfloat[BBCSport]{\includegraphics[width=0.30\textwidth]{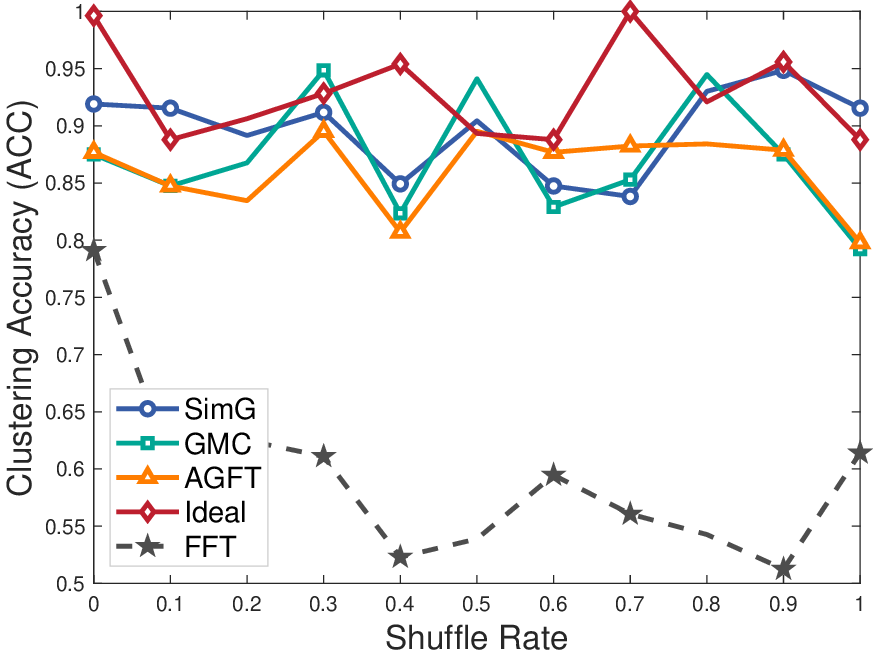}}\hspace{4pt}
 	\caption{Influence of shuffle rates on different GFT methods. Baselines: t-SVD-MVC (a-c), TLSpNM (d-f), ESTMC (g-i), TLRLF4MVC (j-l).}
	\label{fig:shuffle_rate_appendix}

\end{figure*}

\end{document}